\documentclass[11pt]{article}
\usepackage{fullpage}
\usepackage{amsmath,amsfonts,amsthm,amssymb}
\usepackage{url}
\usepackage{color}
\usepackage[usenames,dvipsnames,svgnames,table]{xcolor}
\usepackage[colorlinks=true, linkcolor=red, urlcolor=blue, citecolor=gray]{hyperref}
\usepackage{graphicx}
\usepackage{caption} \usepackage{subcaption, enumitem}
\usepackage{hhline}
\usepackage{algorithm}
\usepackage{algpseudocode}
\usepackage{float}
\usepackage{threeparttable}
\makeatletter

\DeclareMathOperator*{\E}{\mathbb{E}}
\let\Pr\relax
\DeclareMathOperator*{\Pr}{\mathbb{P}}
\DeclareMathOperator{\tr}{tr}

\newcommand{\R}{\mathbb{R}}

\DeclareMathOperator*{\argmin}{arg\,min}
\DeclareMathOperator*{\argmax}{arg\,max}

\newcommand{\eqdef}{\mathbin{\stackrel{\rm def}{=}}}
\newcommand{\norm}[1]{\|#1\|}

\newcommand{\bs}[1]{\boldsymbol{#1}}
\newcommand{\bv}[1]{\mathbf{#1}}

\usepackage{color}

\newcommand{\algoname}[1]{\textnormal{\textsc{#1}}}
\makeatletter

\newcommand{\specialcell}[2][c]{%
  \begin{tabular}[#1]{@{}c@{}}#2\end{tabular}}

\newtheorem*{rep@theorem}{\rep@title}
\newcommand{\newreptheorem}[2]{%
\newenvironment{rep#1}[1]{%
 \def\rep@title{#2 \ref{##1}}%
 \begin{rep@theorem}}%
 {\end{rep@theorem}}}
\makeatother
\newtheorem{theorem}{Theorem}
\newreptheorem{theorem}{Theorem}

\newtheorem{corollary}[theorem]{Corollary}
\newtheorem{lemma}[theorem]{Lemma}
\newtheorem*{lemma*}{Lemma}
\newreptheorem{lemma}{Lemma}
\newtheorem{fact}[theorem]{Fact}

\newtheorem{definition}[theorem]{Definition}

  \usepackage{nth}
  \usepackage{intcalc}

\title{Recursive Sampling for the Nystr{\"o}m Method}
\author{
Cameron Musco\\MIT\\ \texttt{cnmusco@mit.edu}
\and
Christopher Musco\\MIT\\ \texttt{cpmusco@mit.edu}
}
\begin{document}
\maketitle

\begin{abstract}
We give the first algorithm for kernel Nystr{\"o}m approximation that runs in \emph{linear time in
the number of training points} and is provably accurate for all kernel matrices, without dependence on regularity or incoherence conditions. The algorithm projects the kernel onto a set of $s$ landmark points sampled by their \emph{ridge leverage scores}, requiring just $O(ns)$ kernel evaluations and $O(ns^2)$ additional runtime.
While leverage score sampling has long been known to give strong theoretical guarantees for Nystr{\"o}m approximation, by employing a fast recursive sampling scheme, our algorithm is the first to make the approach scalable. Empirically we show that it finds more accurate, lower rank kernel approximations in less time than popular techniques such as uniformly sampled Nystr{\"o}m approximation and the random Fourier features method.

\end{abstract}

\thispagestyle{empty}
\clearpage
\setcounter{page}{1}

\section{Introduction}
The kernel method is a powerful tool for applying linear learning algorithms (SVMs, linear regression, etc.) to nonlinear problems. 
The key idea is to map data to a higher dimensional \emph{kernel feature space}, where linear relationships correspond to nonlinear relationships in the original data.

Typically this mapping is implicit. 
A \emph{kernel function} is used to compute inner products in the high-dimensional kernel space, without ever actually mapping original data points to the space. Given $n$ data points $\bv{x}_1,\ldots,\bv{x}_n$, the $n \times n$ kernel matrix $\bv{K}$ is formed where $\bv{K}_{i,j}$ contains the high-dimensional inner product between $\bv{x}_i$ and $\bv{x}_j$, as computed by the kernel function.
All computations required by a linear learning method are performed using the inner product information in $\bv{K}$.

Unfortunately, the transition from linear to nonlinear comes at a high cost. Just generating the entries of $\bv{K}$ requires $\Theta(n^2)$ time, which is prohibitive for large datasets.

\subsection{Kernel approximation}

A large body of work seeks to accelerate kernel methods by finding a compressed, often low-rank, approximation $\bv{\tilde K}$ to the true kernel matrix $\bv{K}$. 
Techniques include random sampling and embedding \cite{ams01,Balcan:2006,avron2014subspace}, \emph{random Fourier feature} methods for shift invariant kernels \cite{rahimi2007random,rahimi2009weighted,le2013fastfood}, and incomplete Cholesky factorization \cite{fine2002efficient,bach2002kernel}. 

One of the most popular techniques is the \emph{Nystr{\" o}m method}, which constructs $\bv{\tilde K}$ using a subset of ``landmark'' data points \cite{williams2001using}. 
Once $s$ data points are selected, $\bv{\tilde K}$ (in factored form) takes just $O(ns)$ kernel evaluations and $O(s^3)$ additional time to compute, requires $O(ns)$ space to store, and can be manipulated quickly in downstream applications. E.g., inverting $\bv{\tilde K}$ takes $O(ns^2)$ time.

The Nystr{\" o}m method performs well in practice \cite{NIPS2012_4588,gittens2013revisiting,stephentu}, is widely implemented \cite{hall2009weka,scikit-learn,libskylark}, and is used in many applications under different names such as ``landmark isomap'' \cite{de2003global} and ``landmark MDS'' \cite{platt2005fastmap}.
 In the classic variant, landmark points are selected uniformly at random. However, significant research seeks to improve performance via data-dependent sampling that selects landmarks which more closely approximate the full kernel matrix than uniformly sampled ones 
\cite{smola2000sparse,drineas2005nystrom,Zhang:2008,belabassNystrom1,kumar2012sampling,wang2013improving,gittens2013revisiting,chengtaodpp}. 

Theoretical work has converged on \emph{leverage score} based approaches, as they give the strongest provable guarantees for both kernel approximation \cite{drineas2008relative,gittens2013revisiting} and statistical performance in downstream applications \cite{alaoui2015fast,rudi2015less,wang2016column}. Leverage scores
capture how important an individual data point is in composing the span of the kernel matrix.

Unfortunately, these scores are prohibitively expensive to compute. All known approximation schemes require $\Omega(n^2)$ time or assume strong conditions on $\bv{K}$ -- e.g. good conditioning or  data ``incoherence'' \cite{drineas2012fast,gittens2013revisiting,alaoui2015fast,Calandriello:2016:ANM:3020948.3020956}. Hence, leverage score-based approaches remain largely in the domain of theory, with limited practical impact \cite{kumar2012sampling,li2015large,yang2015randomized}.


\subsection{Our contributions}
In this work, we close the gap between strong approximation bounds and efficiency: we present a new Nystr{\"o}m algorithm based on \emph{recursive leverage score sampling} which
achieves the ``best of both worlds'':
it produces kernel approximations provably matching the accuracy of leverage score methods while only requiring $O(ns)$ kernel evaluations and $O(n s^2)$ runtime for $s$ landmark points.

Theoretically, this runtime is surprising. In the typical case when $s \ll n$, the algorithm evaluates just a small subset of $\bv{K}$, ignoring most of the kernel space inner products. Yet its performance guarantees hold for general kernels, requiring \emph{no assumptions on coherence or regularity}.

Empirically, the runtime's linear dependence on $n$ means that our method is the first leverage score algorithm that can compete with the most commonly implemented techniques, including the classic uniform sampling Nystr\"{o}m method and random Fourier features sampling \cite{rahimi2007random}.
Since our algorithm obtains higher quality samples, we show experimentally that it outperforms these methods on benchmark datasets -- it can obtain as accurate a kernel approximation in significantly less time. As a bonus, our approximations have lower rank, so they can be stored in less space and processed more quickly in downstream learning tasks.

\subsection{Paper outline}
Our recursive sampling algorithm is built on top of a Nystr{\"o}m scheme of Alaoui and Mahoney that samples landmark points based on their  \emph{ridge leverage scores} \cite{alaoui2015fast}. After reviewing preliminaries in Section \ref{sec:2}, in Section \ref{sec:leverage_sampling} we analyze this scheme, which we refer to as \emph{RLS-Nystr{\" o}m}. 
To simplify prior work, which studies the statistical performance of RLS-Nystr{\" o}m for specific kernel learning tasks \cite{alaoui2015fast,rudi2015less,wang2016column}, we prove a strong, application independent approximation guarantee:
 for any $\lambda$, if $\bv{\tilde K}$ is constructed with $s = \Theta(d_{\text{eff}}^{\lambda} \log d_{\text{eff}}^{\lambda})$ samples\footnote{$O(d_{\text{eff}}^{\lambda} \log d_{\text{eff}}^{\lambda})$ samples is within a log factor of the best possible for any low-rank approximation with error $\lambda$.}, where $d_{\text{eff}}^{\lambda} =  \tr(\bv{K}(\bv{K} + \lambda\bv{I})^{-1})$ is the so-called ``$\lambda$-effective dimensionality'' of $\bv{K}$, then with high probability,
$ \|\bv{K} - \bv{\tilde K}\|_2 \leq \lambda.$

 In Appendix \ref{sec:apps}, we show that this guarantee implies prior results on the statistical performance of {RLS-Nystr{\" o}m for kernel ridge regression and canonical correlation analysis. We also use it to prove new results on the performance of RLS-Nystr{\" o}m for kernel rank-$k$ PCA and $k$-means clustering -- in both cases just $O(k\log k)$ samples are required to give a solution with good accuracy.
  
 After affirming the favorable theoretical properties of RLS-Nystr{\" o}m, in Section \ref{sec:algos} we show that its runtime can be significantly improved using a recursive sampling approach.  Intuitively our algorithm is simple. We show how to approximate the kernel ridge leverage scores using a \emph{uniform} sample of $\frac{1}{2}$ of our input points. While the subsampled kernel matrix still has a prohibitive $n^2/4$ entries, we can \emph{recursively approximate} it, using our same sampling algorithm. If our final Nystr{\"o}m approximation will use $s$ landmarks, the recursive approximation only needs rank $O(s)$, which lets us estimate the ridge leverage scores of the original kernel matrix in just $O(ns^2)$ time. Since $n$ is cut in half at each level of recursion, our total runtime is $O \left (ns^2 + \frac{ns^2}{2} + \frac{ns^2}{4} + ...\right) = O(ns^2)$, significantly improving upon the method of \cite{alaoui2015fast}, which takes $\Theta(n^3)$ time in the worst case.
 
Our approach builds on recent work on iterative sampling methods for approximate linear algebra \cite{cohen2015uniform,cohen2015ridge}. 
 %
While the analysis in the kernel setting is technical, our final algorithm is simple and easy to implement. 
%
We present and test a parameter-free variation of Recursive RLS-Nystr{\"o}m in Section \ref{sec:experiments}, confirming superior performance compared to existing methods. 

\section{Preliminaries}
\label{sec:2}
Consider an input space $\mathcal{X}$ and a positive semidefinite kernel function $K: \mathcal{X} \times \mathcal{X} \rightarrow \mathbb{R}$. Let $\mathcal{F}$ be an associated reproducing kernel Hilbert space and $\phi: \mathcal{X} \rightarrow \mathcal{F}$ be a (typically nonlinear) feature map such that for any $\bv{x}, \bv{y} \in \mathcal{X}$, $K(\bv{x},\bv{y}) = \langle \phi(\bv{x}), \phi(\bv{y}) \rangle_\mathcal{F}$.
Given a set of $n$ input points $\bv{x}_1,\ldots,\bv{x}_n \in \mathcal{X}$, define the kernel matrix $\bv{K} \in \mathbb{R}^{n \times n}$ by $\bv{K}_{i,j} = K(\bv{x}_i,\bv{x}_j).$ 

It will often be natural to consider the kernelized data matrix that generates $\bv{K}$. Informally, let $\bv{\Phi} \in \mathbb{R}^{n \times d'}$ be the matrix containing $\phi(\bv{x}_1),...,\phi(\bv{x}_n)$ as its rows (note that $d'$ may be infinite). $\bv{K} = \bv{\Phi}\bv{\Phi}^T$. While we use $\bv{\Phi}$ for intuition, in our formal proofs we replace it with any matrix $\bv{B} \in \mathbb{R}^{n \times n}$ satisfying $\bv{B}\bv{B}^T = \bv{K}$ (e.g. a Cholesky factor).

We repeatedly use the singular value decomposition, which allows us to write any rank $r$ matrix $\bv{M} \in \mathbb{R}^{n \times d}$ as $\bv{M} = \bv{U}\bv{\Sigma}\bv{V^T}$,  where $\bv{U} \in \mathbb{R}^{n \times r}$ and $\bv{V} \in \mathbb{R}^{d \times r}$ have orthogonal columns (the left and right singular vectors of $\bv{M}$), and $\bv{\Sigma} \in \mathbb{R}^{r \times r}$ is a positive diagonal matrix containing the singular values: $\sigma_1(\bv{M}) \ge \sigma_2(\bv{M}) \ge \ldots \ge \sigma_r(\bv{M}) > 0$. $\bv{M}$'s pseudoinverse  is given by $\bv{M}^+ = \bv{V} \bv{\Sigma}^{-1} \bv{U}^T$.

\subsection{Nystr{\"o}m approximation}\label{nystromPrelim}

The Nystr{\" o}m method selects a subset of ``landmark'' points and uses them to construct a low-rank approximation to $\bv{K}$. Given a matrix $\bv{S} \in \mathbb{R}^{n \times s}$ that has a single entry in each column equal to $1$ so that $\bv{K}\bv{S}$ is a subset of $s$ columns from $\bv{K}$, the associated Nystr{\" o}m approximation is:
\begin{align}\label{nystrom}
\bv{\tilde K} = \bv{K S} (\bv{S}^T\bv{KS})^+ \bv{S}^T\bv{ K}.
\end{align}
$\bv{\tilde K}$ can be stored in $O(n s)$ space by separately storing $\bv{KS} \in \mathbb{R}^{n \times s}$ and $(\bv{S}^T\bv{KS})^+ \in \mathbb{R}^{s \times s}$. Furthermore, the factors can be computed using just $O(n s)$  evaluations of the kernel inner product to form $\bv{K}\bv{S}$ and $O(s^3)$ time to compute $(\bv{S}^T\bv{KS})^+$. Typically $s \ll n$ so these costs are significantly lower than the cost to form and store the full kernel matrix $\bv{K}$.

We view Nystr{\"o}m approximation as a low-rank approximation to the dataset in feature space. Recalling that $\bv{K} = \bv{\Phi}\bv{\Phi}^T$, $\bv{S}$ selects $s$ kernelized data points $\bv{S}^T\bv{\Phi}$ and we approximate $\bv{\Phi}$ using its projection onto these points.
Informally, let $\bv{P}_\bv{S} \in \mathbb{R}^{d' \times d'}$ be the orthogonal projection onto  the row span of $\bv{S}^T\bv{\Phi}$. We approximate $\bv{\Phi}$ by $\bv{\tilde \Phi} \eqdef \bv{\Phi}\bv{P}_\bv{S}$.
 We can write $\bv{P}_\bv{S} = \bv{\Phi}^T\bv{S} (\bv{S}^T\bv{\Phi} \bv{\Phi}^T \bv{S})^+ \bv{S}^T\bv{\Phi}$. 
 Since it is an orthogonal projection, $\bv{P}_\bv{S}\bv{P}_\bv{S}^T = \bv{P}_\bv{S}^2 = \bv{P}_\bv{S}$, and so we can write:
\begin{align*}
\bv{\tilde K} = \bv{\tilde \Phi} \bv{\tilde \Phi}^T = \bv{\Phi} \bv{P}_\bv{S}^2 \bv{\Phi}^T &= \bv{\Phi} \left (\bv{\Phi}^T\bv{S} (\bv{S}^T\bv{\Phi} \bv{\Phi}^T \bv{S})^+ \bv{S}^T\bv{\Phi}\right) \bv{\Phi}^T = \bv{K} \bv{S} (\bv{S}^T\bv{K}\bv{S})^+ \bv{S}^T \bv{K}.
\end{align*}

This recovers the standard Nystr{\"o}m approximation \eqref{nystrom}. 
Note that we present the above for intuition and do not rigorously handle possibly infinite dimensional feature spaces. To formalize the argument, replace $\bv{\Phi}$ with any $\bv{B} \in \mathbb{R}^{n \times n}$ satisfying $\bv{B}\bv{B}^T = \bv{K}$.

\section{The RLS-Nystr{\"o}m method}
\label{sec:leverage_sampling}

We now introduce the RLS-Nystr{\" o}m method, which uses ridge leverage score sampling to select landmark data points, 
and discuss its strong approximation guarantees for any kernel matrix $\bv{K}$. 

\subsection{Ridge leverage scores}
In classical Nystr{\"o}m approximation \eqref{nystrom}, $\bv{S}$ is formed by sampling data points uniformly at random. Uniform sampling can work in practice, but it only gives theoretical guarantees under strong regularity or incoherence assumptions on $\bv{K}$ \cite{gittens2011spectral}. It will fail for many natural kernel matrices where the relative ``importance'' of points is not uniform across the dataset 

For example, imagine a dataset where points fall into several clusters, but one of the clusters is much larger than the rest. Uniform sampling will tend to oversample landmarks from the large cluster while undersampling or possibly missing smaller but still important clusters. Approximation of $\bv{K}$ and learning performance (e.g. classification accuracy) will decline as a result.

\begin{figure}[h]
\centering
\begin{subfigure}{0.42\textwidth}
\centering
\includegraphics[width=.87\textwidth]{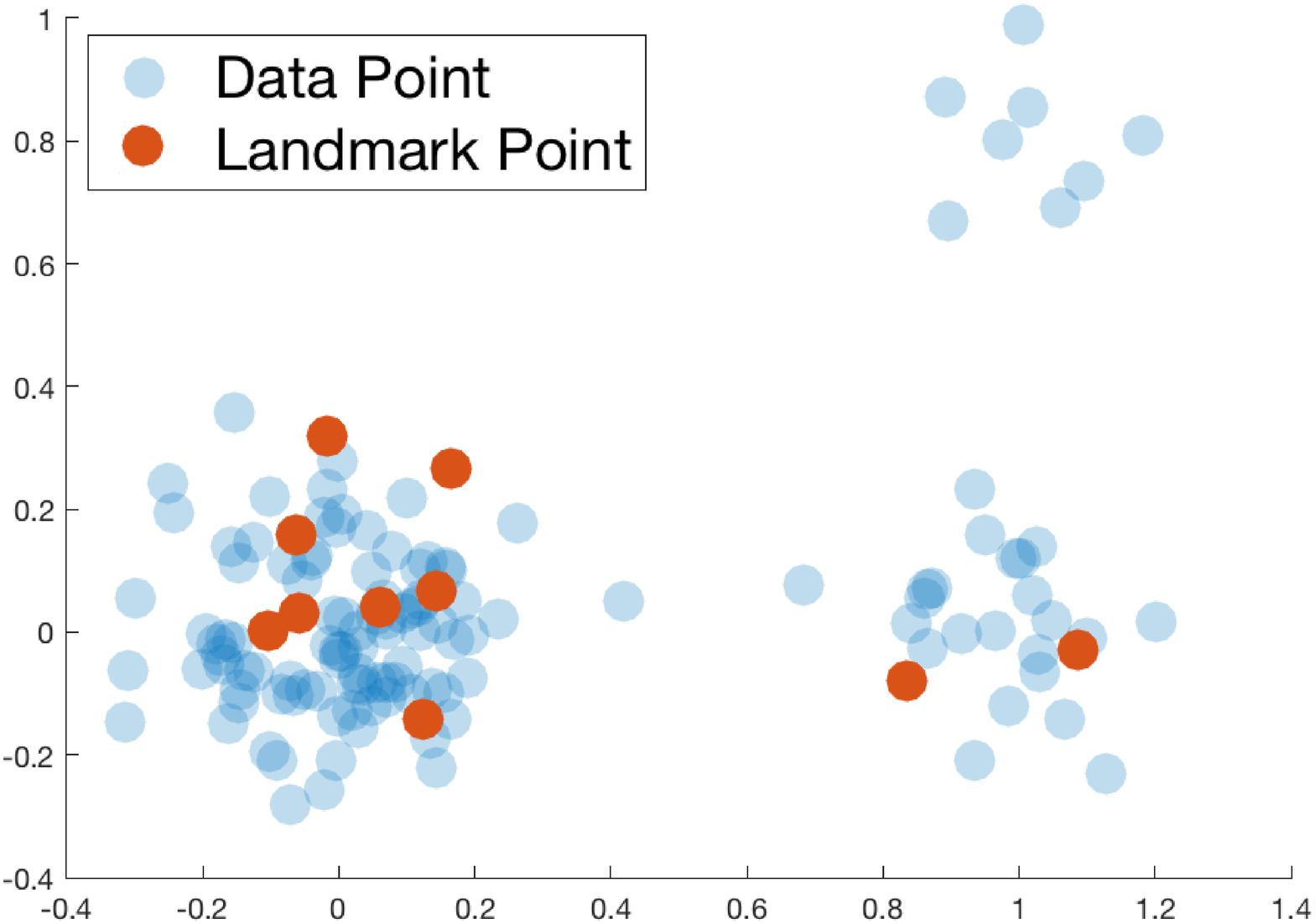} 
\caption{Uniform landmark sampling.}
\label{fig1}
\end{subfigure}%
\begin{subfigure}{0.42\textwidth}
\centering
\includegraphics[width=.87\textwidth]{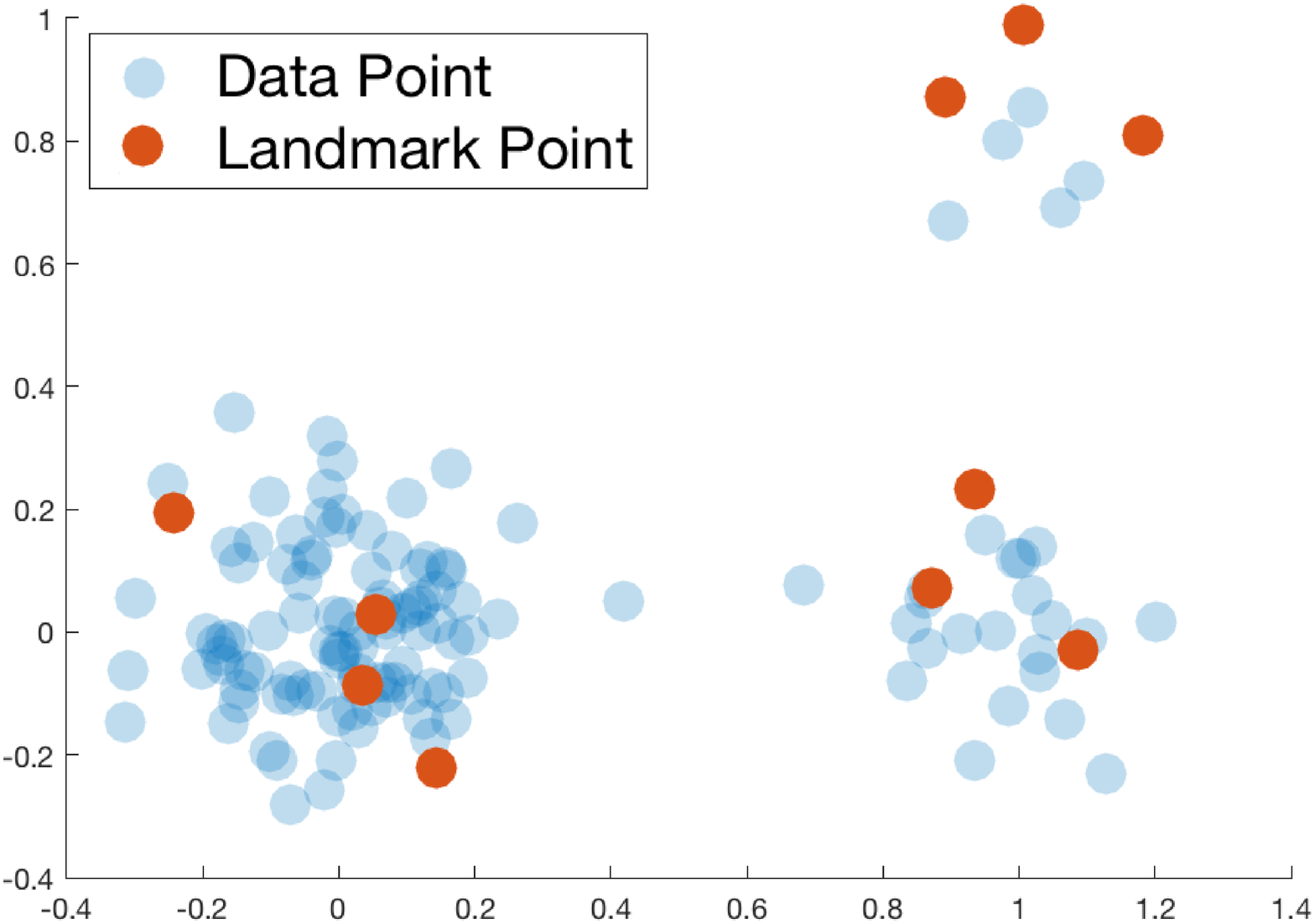}
\caption{Improved landmark sampling.}
\label{fig2}
\end{subfigure}
\caption{Uniform sampling for Nystr{\"o}m approximation can oversample from denser parts of the dataset. A better Nystr{\"o}m scheme will select points that more equally cover the relevant data.}
\label{examplefig}
\end{figure}

To combat this issue, alternative methods compute a measure of point importance that is used to select landmarks. For example, one heuristic applies $k$-means clustering to the input and takes the cluster centers as landmarks \cite{Zhang:2008}. A large body of theoretical work measures importance using variations on the \emph{statistical leverage scores}. 
One natural variation is the \emph{ridge leverage score}:
\begin{definition}[Ridge leverage scores \cite{alaoui2015fast}]\label{leverageDef}
For any $\lambda > 0$, the $\lambda$-ridge leverage score of data point $\bv{x}_i$ with respect to the kernel matrix $\bv{K}$ is defined as
\begin{align}
\label{plain_lev_def}
l_i^\lambda(\bv{K}) \eqdef \left (\bv{K}(\bv{K} + \lambda \bv{I})^{-1}\right )_{i,i},
\end{align}
where $\bv{I}$ is the $n\times n$ identity  matrix. For any $\bv{B} \in \mathbb{R}^{n \times n}$ satisfying $\bv{B}\bv{B}^T = \bv{K}$, we can also write
\begin{align}
\label{b_lev_def}
l_i^\lambda(\bv{K}) = \bv{b}_i^T (\bv{B}^T\bv{B} + \lambda \bv{I})^{-1} \bv{b}_i,
\end{align}
where $\bv{b}_i^T \in \mathbb{R}^{1 \times n}$ is the $i^{th}$ row of $\bv{B}$.
\end{definition}
For conciseness we write $l_i^\lambda(\bv{K})$ as $l_i^\lambda$ and include the argument only when referring to the ridge leverage scores of a kernel matrix other than $\bv{K}$. To check that \eqref{plain_lev_def} and \eqref{b_lev_def} are equivalent note that $ \bv{b}_i^T (\bv{B}^T\bv{B} + \lambda \bv{I})^{-1} \bv{b}_i =  \left(\bv{B} (\bv{B}^T\bv{B} + \lambda \bv{I})^{-1} \bv{B}^T\right)_{i,i}$. Using the SVD to write $\bv{B} = \bv{U}\bs{\Sigma}\bv{V}^T$ and accordingly $\bv{K} = \bv{U}\bs{\Sigma}^2\bv{U}^T$ confirms that $\bv{K}(\bv{K} + \lambda \bv{I})^{-1} = \bv{B} (\bv{B}^T\bv{B} + \lambda \bv{I})^{-1} \bv{B}^T  = \bv{U}\bs{\Sigma}^2\left(\bs{\Sigma}^2 + \lambda \bv{I}\right)^{-1}\bv{U}^T$.

It's not hard to check (see \cite{cohen2015uniform}) that the ridge scores can be defined alternatively as:
\begin{align}
\label{yintat_view}
l_i^\lambda &= \min_{\bv{y} \in \R^n} \frac{1}{\lambda}\|\bv{b}_i^T - \bv{y}^T\bv{B}\|_2^2 + \|\bv{y}\|_2^2.
\end{align}
This formulation provides better insight into the meaning of these scores.
Since $\bv{B}\bv{B}^T = \bv{K}$, any kernel learning algorithm effectively performs linear learning with $\bv{B}$'s rows as data points. So the ridge scores should reflect the relative importance or uniqueness of these rows. From \eqref{yintat_view} it's clear that $l_i^\lambda \leq 1$ since we can set $\bv{y}$ to the $i^\text{th}$ standard basis vector. A row $\bv{b}_i^T$ will have ridge score $\ll1$ (i.e. is less important) when it's possible to find a more ``spread out'' $\bv{y}$ that uses other rows in $\bv{B}$ to approximately reconstruct $\bv{b}_i^T$ -- in other words when the row is less unique.

\subsection{Sum of ridge leverage scores}
As is standard in leverage score methods, we don't directly select landmarks to be the points with the highest scores. Instead, we sample each point with probability proportional to $l_i^\lambda$. I.e. if a point has the highest possible ridge leverage score of 1, we will select it with probability 1 to be a landmark. If a point has leverage score $1/100$, we select it with probability $1/100$.\footnote{To ensure concentration in our sampling algorithm, we will actually take points with probability $ql_i^\lambda$ where $q$ is a small oversampling parameter.}

Accordingly, the number of landmarks selected, which controls $\bv{\tilde K}$'s rank, is a random variable with expectation equal to the \emph{sum of the $\lambda$-ridge leverage scores}. To ensure compact kernel approximations, we want this sum to be small. Immediately from Definition \ref{leverageDef}, we have:
\begin{fact}[Ridge leverage scores sum to the effective dimension]
\label{sum_equal_deff}
\begin{align}
\sum_{i=1}^n l_i^\lambda(\bv{K}) = \tr(\bv{K}(\bv{K} + \lambda\bv{I})^{-1}).
\end{align}
\end{fact}
\noindent$\tr(\bv{K}(\bv{K} + \lambda\bv{I})^{-1})$ is a natural quantity, called the ``effective dimension'' or ``degrees of freedom'' for a ridge regression problem on $\bv{K}$ with regularization $\lambda$ \cite{elements,zhang2006learning}. We use the notation:
\begin{align}
d_\text{eff}^\lambda \eqdef \tr(\bv{K}(\bv{K} + \lambda\bv{I})^{-1}).
\end{align}
$d_\text{eff}^\lambda$ increases monotonically as $\lambda$ decreases. For any fixed $\lambda$ it is essentially the smallest possible rank achievable for $\bv{\tilde K}$ satisfying the approximation guarantee given by RLS-Nystr{\"o}m: $\norm{\bv{K}-\bv{\tilde K}}_2 < \lambda$.

\subsection{The basic sampling algorithm}
We can now introduce the RLS-Nystr{\"o}m method of Alaoui and Mahoney as Algorithm \ref{alg:highlevel}. Our pseudocode allows sampling each point by \emph{any probability greater than $l_i^\lambda$}. This is useful later when we compute ridge leverage scores approximately. Naturally, oversampling landmarks can only improve $\bv{\tilde K}$'s accuracy. It could cause us to take more samples, but we will always ensure that the sum of our approximate ridge leverage scores is not much higher than that of the exact scores.
\begin{algorithm}[H]
\caption{\algoname{RLS-Nystr{\" o}m Sampling}}
{\bf input}: $\bv{x}_1,\ldots,\bv{x}_n \in \mathcal{X}$, kernel matrix $\bv{K}$, ridge parameter $\lambda >0$, failure probability $\delta \in (0,1/8)$\\
{\bf output}: kernel approximation $\bv{\tilde K}$
\begin{algorithmic}[1]
\State{Compute an over-approximation, $\tilde l_i^{\lambda} > l_i^{\lambda}$ for the $\lambda$-ridge leverage score of each $\bv{x}_1, \ldots, \bv{x}_n$}\label{step1}
\State{Set $p_i := \min \left \{1, \tilde l_i^{\lambda} \cdot 16\log(\sum \tilde l_i^{\lambda}/\delta) \right \}$}.
\State{Construct $\bv{S} \in \mathbb{R}^{n \times s}$ by sampling $\bv{x}_1,\ldots,\bv{x}_n$ each independently with probability $p_i$. In other words, for each $i$ add a column to $\bv{S}$ with a 1 in position $i$ with probability $p_i$}.\\
\Return{the Nystr{\" o}m factors $\bv{K}\bv{S} \in \R^{n\times s}$ and $(\bv{S}^T\bv{K}\bv{S})^+ \in \R^{s \times s}$}\label{step4}.
\end{algorithmic}
\label{alg:highlevel}
\end{algorithm}
\noindent Note that this implementation of RLS-Nystr{\" o}m Sampling does not form $\bv{\tilde K}$ explicitly in Step \ref{step4}, as this would take space and time quadratic in $n$. It simply returns the factors $\bv{K}\bv{S} \in \R^{n\times s}$ and $(\bv{S}^T\bv{K}\bv{S})^+ \in \R^{s \times s}$. Any kernel learning method can then access $\bv{\tilde K}$ implicitly. 
For example, the kernel method can be implemented as a linear method run on the $n \times s$ matrix $\bv{K}\bv{S}(\bv{S}^T\bv{K}\bv{S})^{+/2}$ whose rows serves as a compression of the data points in kernel space

\subsection{Accuracy bounds}
Like other leverage scores methods, RLS-Nystr{\"o}m sampling is appealing because it provably approximates any kernel matrix. In particular, we show that the algorithm produces a $\bv{\tilde K}$ which spectrally approximates $\bv{K}$ up to a small additive error. This is the strongest type of approximation offered by any known Nystr{\"o}m method \cite{gittens2013revisiting} and, importantly, it guarantees that $\bv{\tilde K}$ will provide provable accuracy when used in place of $\bv{K}$ in many downstream machine learning applications. 
\begin{theorem}[Spectral error approximation]
\label{additiveErrorThm} For any $\lambda > 0$ and $\delta \in (0,1/8)$, Algorithm \ref{alg:highlevel} returns an $\bv{S} \in \R^{n\times s}$ such that with probability $1-\delta$, $s \leq 2\sum_i p_i$ and $\bv{\tilde K} =  \bv{KS} (\bv{S}^T\bv{K}\bv{S})^+ \bv{S}^T\bv{K}$ satisfies:
\begin{align}
\label{eq:adderrror}
\bv{\tilde K} \preceq \bv{K} \preceq \bv{\tilde K} + \lambda \bv{I}.
\end{align}
When ridge scores are computed exactly, $\sum_i p_i = O \left(d_\text{eff}^\lambda\log (d_\text{eff}^\lambda/\delta) \right )$.
\end{theorem}
$\preceq$ denotes the standard Loewner matrix ordering on positive semi-definite matrices\footnote{$\bv{M} \preceq \bv{N}$ means that $\bv{N} - \bv{M}$ is positive semidefinite.}.
Note that \eqref{eq:adderrror} immediately implies the well studied (see e.g \cite{gittens2013revisiting}) spectral norm guarantee, $\|\bv{K} - \bv{\tilde K}\|_2 \leq \lambda$. 

Intuitively, Theorem \ref{additiveErrorThm} guarantees that the $\bv{\tilde K}$ produced by RLS-Nystr{\"o}m well approximates the top of $\bv{K}$'s spectrum (i.e. any eigenvalues $> \lambda$) while allowing it to lose information about smaller eigenvalues, which are less important for many learning tasks. 

\begin{proof}
It is clear from the view of Nystr{\"o}m approximation as a low-rank projection of the kernelized data (see Section \ref{nystromPrelim}) that $\bv{\tilde K} \preceq \bv{K}.$ Formally, for any $\bv{B} \in \mathbb{R}^{n \times n}$ with $\bv{BB}^T = \bv{K}$:
\begin{align*}
\bv{\tilde K} = \bv{KS}(\bv{S}^T\bv{K}\bv{S})^+ \bv{S}^T \bv{K} = \bv{B} \bv{P}_\bv{S} \bv{B}^T,
\end{align*}
where $\bv{P}_\bv{S} = \bv{B}^T \bv{S}(\bv{S}^T \bv{B}\bv{B}^T \bv{S})^+ \bv{S}^T \bv{B}$ is the orthogonal projection onto the row span of $\bv{S}^T \bv{B}$. Since $\bv{P}_\bv{S}$ is a projection $\norm{\bv{P}_\bv{S}}_2 \le 1$. So, for
any $\bv{x} \in \mathbb{R}^n$:
\begin{align*}
\bv{x}^T\bv{\tilde K} \bv{x} = \bv{x}^T \bv{B} \bv{P}_\bv{S} \bv{B} \bv{x} = \norm{\bv{P}_\bv{S} \bv{B}\bv{x}}_2^2 \le \norm{\bv{B} \bv{x}}_2^2 = \bv{x}^T \bv{K} \bv{x},
\end{align*}
which is equivalent to $\bv{\tilde K} \preceq \bv{K}$. It remains to show that $\bv{K} \preceq \bv{\tilde K} + \lambda \bv{I}$.

In Lemma \ref{bernstein}, Appendix \ref{sampling_proofs}, we apply a matrix Bernstein bound \cite{tropp2015introduction} to prove that, when $\bv{S}$'s columns are reweighted by the inverse of their sampling probabilities, with probability $1-\delta/2$: 
\begin{align*}
\frac{1}{2}\left(\bv{B}^T\bv{B} + \lambda \bv{I}\right) \preceq \bv{B}^T\bv{SS}^T\bv{B} + \lambda \bv{I}  \preceq \frac{3}{2}\left(\bv{B}^T\bv{B} + \lambda \bv{I}\right).
\end{align*}
It is not hard to show (Corollary \ref{bernstein2}, Appendix \ref{sampling_proofs}) that even if $\bv{S}$ is unweighted, as in Algorithm \ref{alg:highlevel}, this bound implies the existence of some finite scaling factor $C > 0$ such that:
\begin{align}\label{elambda_bound}
\bv{B}^T\bv{B} \preceq C \cdot \bv{B}^T\bv{S}\bv{S}^T\bv{B} + \lambda \bv{I}.
\end{align}
Let $\bv{\bar P}_\bv{S} = \bv{I} - \bv{P}_\bv{S}$ be the projection onto the complement of the row span of $\bv{S}^T \bv{B}$. By \eqref{elambda_bound}:
\begin{align}\label{elambda_bound2}
\bv{\bar P}_\bv{ S} \bv{B}^T\bv{B} \bv{\bar  P}_\bv{S} \preceq C \cdot \bv{\bar  P}_\bv{S} \bv{B}^T\bv{S}\bv{S}^T\bv{B}\bv{\bar  P}_\bv{S} + \lambda \bv{\bar P}_\bv{S} \bv{I} \bv{\bar  P}_\bv{S}.
\end{align}
Since $\bv{\bar P}_\bv{ S}$ projects to the complement of the row span of $\bv{S}^T\bv{B}$, $\bv{S}^T\bv{B} \bv{\bar P}_\bv{S} = \bv{0}$. So \eqref{elambda_bound2} gives:
\begin{align*}
\bv{\bar P}_\bv{S} \bv{B}^T\bv{B} \bv{\bar P}_\bv{S} &\preceq \bv{0} + \lambda \bv{\bar P}_\bv{S}\bv{I} \bv{\bar P}_\bv{ S} \preceq  \lambda \bv{I}.
\end{align*}
In other notation, $\norm{\bv{\bar P}_\bv{S} \bv{B}^T\bv{B} \bv{\bar P}_\bv{S}}_2 \le \lambda$. This in turn implies $\norm{\bv{B} \bv{\bar P}_\bv{S} \bv{B}^T}_2 \le \lambda$ and hence:
\begin{align*}
\bv{B}\bv{\bar P}_\bv{S}\bv{B}^T = \bv{B}(\bv{I} - \bv{P}_\bv{S})\bv{B}^T \preceq \lambda \bv{I}.
\end{align*}
Rearranging and using $\bv{K} = \bv{B}\bv{B}^T$ and $\bv{\tilde K} = \bv{B}\bv{P}_\bv{S}\bv{B}^T$ gives the result. A Chernoff bound (see Lemma \ref{bernstein}, Appendix \ref{sampling_proofs}), gives that with probability $1-\delta/2$, $s \le  2\sum_i p_i$, completing the theorem.
\end{proof}

Often a regularization parameter $\lambda$ is specified for a learning task, and for near optimal performance on this task, we set the approximation factor in Theorem \ref{additiveErrorThm} to $\epsilon \lambda$. In this case we have:
\begin{corollary}[Tighter spectral error approximation]
\label{main_cor}
For any $\lambda > 0$ and $\delta \in (0,1/8)$, Algorithm \ref{alg:highlevel} run with ridge parameter $\epsilon \lambda$ returns $\bv{S} \in \R^{n\times s}$ such that with probability  $1-\delta$, $s = O\left(\frac{d^\lambda_\text{\normalfont eff}}{\epsilon}\log\frac{d^\lambda_{\text{\normalfont eff}}}{\delta\epsilon}\right)$ and $\bv{\tilde K} =  \bv{KS} (\bv{S}^T\bv{K}\bv{S})^+ \bv{S}^T\bv{K}$ satisfies
$\bv{\tilde K} \preceq \bv{K} \preceq \bv{\tilde K} + \epsilon\lambda \bv{I}.
$
\end{corollary}
\begin{proof}
This follows from Theorem \ref{additiveErrorThm} by noting $d_\text{eff}^{\epsilon\lambda} \leq d_\text{eff}^\lambda/\epsilon$ since $(\bv{K} + \epsilon\lambda I)^{-1} \preceq \frac{1}{\epsilon}(\bv{K} + \lambda I)^{-1}$. 
\end{proof}

Corollary \ref{main_cor} is sufficient to prove that $\bv{\tilde K}$ can be used in place of $\bv{K}$ without sacrificing performance on kernel ridge regression and canonical correlation tasks (see \cite{alaoui2015fast} and \cite{wang2016column}). 
We also use it to prove a \emph{projection-cost preservation} guarantee (Theorem \ref{pcpTheorem}, Appendix \ref{sec:pcp}). Specifically, we show that if $O((k\log k)/\epsilon)$ landmarks are sampled with an appropriately chosen ridge parameter $\lambda$, then for any rank-$k$ projection matrix $\bv{X}$, $\bv{\tilde K}$ will satisfy, for some fixed $c > 0$:
\begin{align}
\label{eq:pcp_short}
\tr(\bv{K} - \bv{X}\bv{K}\bv{X})\le \tr(\bv{\tilde K} - \bv{X}\bv{\tilde K}\bv{X}) + c \le (1+\epsilon)\tr(\bv{K} - \bv{X}\bv{K}\bv{X}).
\end{align}
\eqref{eq:pcp_short} allows us to prove approximation guarantees for kernel PCA and $k$-means clustering. 
Projection-cost preservation has proven a powerful concept in the matrix sketching literature \cite{feldman2013turning,cohen2015dimensionality,cohen2015ridge,boutsidis2016optimal,ClarksonW17}. We hope that an explicit guarantee for kernels will lead to applications of RLS-Nystr{\" o}m beyond those considered in this work.

Our results on downstream learning bounds that can be derived from Theorem \ref{additiveErrorThm} are summarized in Table \ref{tab:applications}. Details can be found in Appendices \ref{sec:pcp} and \ref{sec:apps}. 

\begin{table}[h]
\small
\centering 
\begin{threeparttable}[b]
\begin{tabular}{||c||c|c|c|c||}
\hhline{|t:=:t:====:t|}
Application  & \specialcell{Downstream \\Guarantee} &\specialcell{Relevant \\Theorem} & \specialcell{Space to \\store $\bv{\tilde K}$} & \specialcell{Time to\\ compute $\bv{\tilde K}$} \\
\hhline{|:=::====:|}
\specialcell{Kernel Ridge \\Regression w/ \\Parameter $\lambda$} 
& \specialcell{$(1+\epsilon)$ relative error \\ risk bound}  
& Thm \ref{thm:app_ridge}
&$\tilde O(\frac{n d_{\text{eff}}^\lambda}{\epsilon})$ 
&  \specialcell{$\tilde O(\frac{n (d_{\text{eff}}^\lambda)^2}{\epsilon^2}) +$\\ $\tilde O(\frac{n d_{\text{eff}}^\lambda}{\epsilon})$ kernel evals.}
\\\hhline{||-||-|-|-|-||}
\specialcell{Kernel \\ $k$-means \\ Clustering} 
& \specialcell{$(1+\epsilon)$ relative error}  
& Thm \ref{thm:app_kmeans}
&$\tilde O(\frac{nk}{\epsilon})$ 
& $\tilde O(\frac{nk^2}{\epsilon^2}) + \tilde O(\frac{nk}{\epsilon})$ kernel evals.
\\\hhline{||-||-|-|-|-||}
\specialcell{Rank $k$ \\ Kernel PCA} 
& \specialcell{$(1+\epsilon)$ relative \\ Frobenius norm error}  
& Thm \ref{thm:app_pca}
&$\tilde O(\frac{nk}{\epsilon})$ 
& $\tilde O(\frac{nk^2}{\epsilon^2}) + \tilde O(\frac{nk}{\epsilon})$ kernel evals.
\\\hhline{||-||-|-|-|-||}
\specialcell{Kernel CCA w/ \\ Regularization \\Params $\lambda_x$, $\lambda_y$} 
& \specialcell{$\epsilon$ additive error to \\ canonical correlation}  
& Thm \ref{thm:app_cca}
&$\tilde O(\frac{nd_{\text{eff}}^{\lambda_x} + nd_{\text{eff}}^{\lambda_y}}{\epsilon})$ 
& \specialcell{$\tilde O(\frac{n(d_{\text{eff}}^{\lambda_x})^2+ n(d_{\text{eff}}^{\lambda_y})^2}{\epsilon^2}) +$\\ $\tilde O(\frac{nd_{\text{eff}}^{\lambda_x} + n d_{\text{eff}}^{\lambda_y}}{\epsilon})$ kernel evals.}
\\
\hhline{|b:=:b:====:b|}
\end{tabular}

\begin{tablenotes}
\item[$*$] For conciseness, $\tilde O(\cdot )$ hides log factors in the failure probability, $d_{\text{eff}}$, and $k$.
\end{tablenotes}

\caption{Downstream guarantees for $\bv{\tilde K}$ obtained from RLS-Nystr{\" o}m (Algorithm \ref{alg:highlevel}). For all problems, the runtime and space cost depends linearly on the number of training data points $n$.} \vspace{.25em}
\label{tab:applications}
\end{threeparttable}

\end{table}

\section{Recursive sampling for efficient RLS-Nystr{\"o}m}
\label{sec:algos}
Having established strong approximation guarantees for RLS-Nystr{\"o}m, it remains to provide an efficient implementation. Specifically, Step \ref{step1} of Algorithm \ref{alg:highlevel} naively requires $\Theta(n^3)$ time. We show that significant acceleration is possible using a recursive sampling approach, which is adapted from techniques developed in \cite{cohen2015uniform} and \cite{cohen2015ridge}.

\subsection{Ridge leverage score approximation via uniform sampling}
The key idea is to approximate the ridge leverage scores of $\bv{K}$ using a uniform sample of the data points. To ensure accuracy, the sample must be large -- a constant fraction of the points. We later show how to recursively approximate this large sample to achieve our final runtimes. 
We first prove:
\begin{lemma}\label{uniformSampling} For any $\bv{B} \in \mathbb{R}^{n \times n}$ with $\bv{B}\bv{B}^T = \bv{K}$ and $\bv{S} \in \mathbb{R}^{n\times s}$ chosen by sampling each data point independently with probability $1/2$, let
\begin{align}\label{approxScoreEquation}
\tilde l_i^{\lambda} = \bv{b}_i^T (\bv{B}^T\bv{S} \bv{S}^T \bv{B} + \lambda \bv{I})^{-1}\bv{b}_i
\end{align}
and $p_i = \min \{1, 16\tilde l_i^{\lambda} \log(\sum_i \tilde l_i^{\lambda}/\delta)\}$ for any $\delta \in (0,1/8)$. Then with probability at least $1-\delta$:
\begin{enumerate}
\item $\tilde l_i^{\lambda} \ge l_i^{\lambda}$ for all $i$.
\item $\sum_i p_i \le 64\sum_i l_i^{\lambda}\log(\sum_i l_i^{\lambda}/\delta) $.
\end{enumerate}
\end{lemma}
The first condition ensures that the approximate scores $\tilde l_i^{\lambda}$ suffice for use in Algorithm \ref{alg:highlevel}. The second ensures that the Nystr{\"o}m approximation obtained will have, up to constant factors, the same size as if we used the true ridge leverage scores. Note that it is not obvious how to compute $\tilde l_i^{\lambda}$ using the formula in \eqref{approxScoreEquation} without explicitly forming $\bv{B}$. We discuss how to do this in Section \ref{sec:generalized}.
\begin{proof}
The first bound follows trivially since $\bv{B}^T\bv{S}\bv{S}^T \bv{B} \preceq \bv{B}^T \bv{B}$ so:
\begin{align*}
\tilde l_i^{\lambda} = \bv{b}_i^T (\bv{B}^T \bv{S}\bv{S}^T\bv{B}+ \lambda \bv{I})^{-1}\bv{b}_i \ge \bv{b}_i^T (\bv{B}^T \bv{B} + \lambda \bv{I})^{-1}\bv{b}_i = l_i^{\lambda}.
\end{align*}
The challenge is the second bound. 
The key observation is that there exists a diagonal reweighting matrix $\bv{W} \in \mathbb{R}^{n \times n}$, $\bv{0} \preceq \bv{W} \preceq \bv{I}$ such that for all $i$, $l_i^{\lambda}(\bv{W}\bv{K}\bv{W}) \le \alpha$ where $\alpha \eqdef \frac{1}{2} \cdot \frac{1}{16 \log(\sum l_i^{\lambda}/\delta)}$. This ensures that uniformly sampling rows with probability $1/2$ from the \emph{reweighted kernel} $\bv{W}\bv{K}\bv{W}$ is a valid ridge leverage score sampling. Additionally, $|\{i : \bv{W}_{i,i} < 1 \}| \le 32\log(\sum l_i^{\lambda}/\delta) \cdot \sum l_i^{\lambda}$. That is, we do not need to reweight too many columns to achieve the ridge leverage score bound.

Although $\bv{W}$ is never actually computed, its existence can be proved algorithmically: we can construct a valid $\bv{W}$ by iteratively considering any $i$ with $l_i^{\lambda}(\bv{W}\bv{K}\bv{W}) \ge \alpha$. Since $\lambda > 0$, it is always possible to decrease the ridge leverage score to exactly $\alpha$ by decreasing $\bv{W}_{i,i}$ sufficiently.

It is clear from the interpretation of Definition \ref{leverageDef} given in \eqref{yintat_view} that decreasing $\bv{W}_{i,i}$, which corresponds to decreasing the weight of row $i$ of $\bv{B}$, only increases the ridge leverage scores of other rows. So, any reweighted row will always maintain leverage score $\ge \alpha$ as other rows are reweighted. Theorem 2 of \cite{cohen2015uniform} demonstrates rigorously that the reweighted rows' leverage scores in fact converge to $\alpha$. Further, since $\bv{W} \preceq \bv{I}$, it is simple to show (see Lemma \ref{decreasingScore2}, Appendix \ref{additional_proofs}):
\begin{align*}
\sum_i l_i^{\lambda}(\bv{W}\bv{K}\bv{W}) \le \sum_i l_i^{\lambda}(\bv{K}) \eqdef \sum_i l_i^{\lambda}.
\end{align*}
Thus, since each reweighted row has $l_i^\lambda(\bv{W}\bv{K}\bv{W}) \ge \alpha$, $\alpha \cdot |\{i : \bv{W}_{i,i} < 1 \}| \le \sum_i l_i^{\lambda}$ and so:
$$|\{i : \bv{W}_{i,i} < 1 \}| \le \frac{1}{\alpha}\sum_i l_i^{\lambda} = 32\log \left (\sum l_i^{\lambda}/\delta \right ) \cdot \sum l_i^{\lambda}.$$

We can now bound $\sum_i p_i$. For any $i$ that is reweighted by $\bv{W}$ we just trivially bound $p_i \le 1$. Since $l_i^{\lambda}(\bv{W}\bv{K}\bv{W}) \le \frac{1}{2} \cdot \frac{1}{16 \log(\sum l_i^{\lambda}/\delta)}$ for all $i$, and since $\bv{S}$ samples each $i$ with probability $1/2$, by the matrix Bernstein bound of Lemma \ref{bernstein}, with probability $1-\delta/2$:
\begin{align*}
\frac{1}{2}(\bv{B}^T\bv{W}^2\bv{B} + \lambda \bv{I}) \preceq (\bv{B}^T\bv{W}\bv{S}\bv{S}^T \bv{W}\bv{B} + \lambda \bv{I}) \preceq \frac{3}{2} (\bv{B}^T\bv{W}^2\bv{B} + \lambda \bv{I}).
\end{align*}
Hence:
\begin{align*}
\tilde l_i^{\lambda} = \bv{b}_i^T (\bv{B}^T\bv{S} \bv{S}^T \bv{B} + \lambda \bv{I})^{-1}\bv{b}_i &\le \bv{b}_i^T (\bv{B}^T \bv{W}\bv{S} \bv{S}^T \bv{W}\bv{B} + \lambda \bv{I})^{-1}\bv{b}_i\\
 &\le 2 \bv{b}_i^T (\bv{B}^T\bv{W}^2 \bv{B} + \lambda \bv{I})^{-1}\bv{b}_i\\
&= 2l_i^{\lambda}(\bv{W}\bv{B}\bv{B}^T\bv{W}) = 2l_i^{\lambda}(\bv{W}\bv{K}\bv{W}).
\end{align*}
Again using that $\bv{W} \preceq \bv{I}$ and Lemma \ref{decreasingScore2}, $\sum_{\{i: \bv{W}_{i,i}=1\}} \tilde l_i^{\lambda} \le 2\sum_i l_i^{\lambda}.$
Overall:
\begin{align*}
\sum_i p_i &= \sum_{\{i: \bv{W}_{i,i} < 1\}} p_i + \sum_{\{i: \bv{W}_{i,i} = 1\}} p_i\\
&\le |\{i: \bv{W}_{i,i} < 1\}| + 32 \log \left (\sum l_i^{\lambda}/\delta \right) \cdot \sum_i l_i^{\lambda} \\
&= 64 \log \left(\sum l_i^{\lambda}/\delta \right) \cdot \sum_i l_i^{\lambda}.
\end{align*}
\end{proof}


\subsection{Computing ridge leverage scores from a sample}\label{sec:generalized}
In order to utilize Lemma \ref{uniformSampling} we must show how to efficiently compute $\tilde l_i^{\lambda}$ via formula \eqref{approxScoreEquation}
\emph{without explicitly forming} either $\bv{K}$ or $\bv{B}$. We prove the following:
\begin{lemma} 
\label{michaels_lemma}
For any sampling matrix $\bv{S} \in \mathbb{R}^{n\times s}$, and any $\lambda > 0$:
\begin{align*}
\tilde l_i^{\lambda} \eqdef \bv{b}_i^T (\bv{B}^T\bv{S} \bv{S}^T \bv{B} + \lambda \bv{I})^{-1}\bv{b}_i = 
\frac{1}{\lambda}\left(\bv{K} - \bv{K}\bv{S}\left(\bv{S}^T\bv{K}\bv{S} + \lambda\bv{I}\right)^{-1}\bv{S}^T\bv{K}\right)_{i,i}. 
\end{align*}
It follows that we can compute $\tilde l_i^{\lambda}$ for all $i$ in $O(ns^2)$ time using just $O(ns)$ kernel evaluations, to compute $\bv{KS}$ and the diagonal of $\bv{K}$.
\end{lemma}
\begin{proof}
Using the SVD write $\bv{S}^T \bv{B} = \bv{\bar U} \bv{\bar \Sigma} \bv{\bar V}^T$. $\bv{\bar V} \in \R^{n \times s}$ forms an orthonormal basis for the row span of $\bv{S}^T \bv{B}$. Let $\bv{\bar V}_\perp$ be span for the nullspace of $\bv{S}^T \bv{B}$. Then we can rewrite $\tilde l_i^{\lambda}$ as:
\begin{align*}
\tilde l_i^{\lambda} = \bv{b}_i^T \left (\bv{B}^T \bv{S} \bv{S}^T \bv{B} + \lambda \bv{I}\right )^{-1} \bv{b}_i &= \bv{b}_i^T \left[\bv{\bar V},\bv{\bar V}_\perp\right] (\bv{\bar \Sigma}^2 + \lambda \bv{I})^{-1} \left[\bv{\bar V},\bv{\bar V}_\perp\right]^T \bv{b}_i.
\end{align*}
Here we abuse notation a by letting $\bv{\bar \Sigma}$ represent an $n\times n$ diagonal matrix whose first $s$ entries are the singular values of $\bv{S}^T \bv{B}$ and whose remaining entries are all equal to 0. Now: 
\begin{align}
\label{split_up_rscore}
\tilde l_i^{\lambda} = \bv{b}_i^T \left[\bv{\bar V},\bv{\bar V}_\perp\right] (\bv{\bar \Sigma}^2 + \lambda \bv{I})^{-1} \left[\bv{\bar V},\bv{\bar V}_\perp\right]^T \bv{b}_i
&= \frac{1}{ \lambda} \bv{b}_i^T\bv{\bar V}_\perp^T\bv{\bar V}_\perp\bv{b}_i + \bv{b}_i^T \bv{\bar V} (\bv{\bar \Sigma}^2 + \lambda \bv{I})^{-1} \bv{\bar V}^T \bv{b}_i^T.
\end{align}
Focusing on the second term of \eqref{split_up_rscore},
\begin{align}
\bv{b}_i^T \bv{\bar V} (\bv{\bar \Sigma}^2 + \lambda \bv{I})^{-1} \bv{\bar V}^T  \bv{b}_i 
&= \bv{b}_i^T \bv{\bar V} \frac{1}{\lambda}\left (\bv{I} - \bv{\bar \Sigma}^2(\bv{\bar \Sigma}^2 + \lambda \bv{I})^{-1} \right ) \bv{\bar V}^T \bv{b}_i \nonumber\\
\label{inverse_term}
&= \frac{1}{\lambda}\bv{b}_i^T \bv{\bar V}\bv{\bar V}^T \bv{b}_i - \frac{1}{\lambda}\bv{b}_i^T\bv{\bar V}\left (\bv{\bar \Sigma}^2(\bv{\bar \Sigma}^2 + \lambda \bv{I})^{-1} \right ) \bv{\bar V}^T \bv{b}_i.
\end{align}
Focusing on the second term of \eqref{inverse_term},
\begin{align*}
\bv{b}_i^T\bv{\bar V}\left (\bv{\bar \Sigma}^2(\bv{\bar \Sigma}^2 + \lambda \bv{I})^{-1} \right ) \bv{\bar V}^T \bv{b}_i 
&= \bv{b}_i^T \bv{\bar V} \bv{\bar \Sigma} \bv{\bar U}^T \bv{\bar U} (\bv{\bar \Sigma}^2 + \lambda \bv{I})^{-1} \bv{\bar U}^T \bv{\bar U}\bv{\bar \Sigma} \bv{\bar V}^T \bv{b}_i^T\\
&=  \bv{b}_i^T \bv{B}^T \bv{S} (\bv{S}^T \bv{K}\bv{S} + \lambda \bv{I})^{-1} \bv{S}^T \bv{B} \bv{b}_i.
\end{align*}
Substituting back into \eqref{inverse_term} and then \eqref{split_up_rscore}, we conclude that:
\begin{align*}
\tilde l_i^{\lambda} &= \frac{1}{ \lambda} \bv{b}_i^T\bv{\bar V}_\perp^T\bv{\bar V}_\perp\bv{b}_i +  \frac{1}{\lambda}\bv{b}_i^T \bv{\bar V}\bv{\bar V}^T \bv{b}_i - \frac{1}{\lambda}\bv{b}_i^T \bv{B}^T \bv{S} (\bv{S}^T \bv{K}\bv{S} + \lambda \bv{I})^{-1} \bv{S}^T \bv{B} \bv{b}_i \\
&= \frac{1}{\lambda} \bv{b}_i^T\bv{b}_i - \frac{1}{\lambda}\bv{b}_i^T \bv{B}^T \bv{S} (\bv{S}^T \bv{K}\bv{S} + \lambda \bv{I})^{-1} \bv{S}^T \bv{B} \bv{b}_i\\
&=\frac{1}{\lambda} \bv{K}_{i,i} - \frac{1}{\lambda} \left ( \bv{K}\bv{S}\left(\bv{S}^T\bv{K}\bv{S} + \lambda\bv{I}\right)^{-1}\bv{S}^T\bv{K}\right)_{i,i}.
\end{align*}

We can compute $(\bv{S}^T \bv{K}\bv{S} + \lambda \bv{I})^{-1}$ in $O(s^3) \le O(ns^2)$ time and $O(s^2) \le O(ns)$ kernel evaluations. Given this inverse, computing the diagonal entries of $\bv{K}\bv{S}\left(\bv{S}^T\bv{K}\bv{S} + \lambda\bv{I}\right)^{-1}\bv{S}^T\bv{K}$ requires just $O(ns)$ kernel evaluations to form $\bv{KS}$ and $O(ns^2)$ time to perform the necessary multiplications. Finally, computing the diagonal entries of $\bv{K}$ requires $n$ additional kernel evaluations.
\end{proof}

\subsection{Recursive RLS-Nystr{\"o}m}
We are finally ready to use Lemmas \ref{uniformSampling} and \ref{michaels_lemma} to give an efficient recursive method for ridge leverage score Nystr{\"o}m approximation.
We show that the output of Algorithm \ref{halvingFixed}, $\bv{S}$, is sampled according to approximate ridge leverage scores for $\bv{K}$ and so satisfies the approximation bound of Theorem \ref{additiveErrorThm}.
\begin{theorem}[Main Result]
\label{thm:main_algo_theorem}
Let $\bv{S} \in \mathbb{R}^{n\times s}$ be computed by
Algorithm \ref{halvingFixed}. With probability $1-3\delta$, $s \le 384 \cdot d_{\text{eff}}^\lambda\log(d_{\text{\normalfont eff}}^\lambda/\delta)$, $\bv{S}$ is sampled by overestimates of the $\lambda$-ridge leverage scores of $\bv{K}$, and thus by Theorem \ref{additiveErrorThm}, the Nystr{\"o}m approximation $\bv{\tilde K} = \bv{K S} (\bv{S}^T\bv{K}\bv{S})^+ \bv{S}^T\bv{K}$ satisfies:
\begin{align*}
\bv{\tilde K} \preceq \bv{K} \preceq \bv{\tilde K} + \lambda \bv{I}.
\end{align*}
 Algorithm \ref{halvingFixed} uses
$O(n s)$ kernel evaluations and $O(n s^2 )$ computation time. 
\end{theorem}

\begin{algorithm}[H]
\caption{\algoname{RecursiveRLS-Nystr{\"o}m}.}
{\bf input}: $\bv{x}_1,\ldots,\bv{x}_m \in \mathcal{X}$, kernel function $K: \mathcal{X}\times \mathcal{X} \rightarrow \mathbb{R}$, ridge $\lambda > 0$, failure prob. $\delta \in (0,1/32)$\\
{\bf output}: weighted sampling matrix $\bv{S} \in \mathbb{R}^{m \times s}$ 
\begin{algorithmic}[1]
\If{$m \leq 192\log(1/\delta)$}\label{if_state}
	\State{\Return{$\bv{S} := \bv{I}_{m\times m}$.}}
\EndIf
\State{Let $\mathcal{\bar S}$ be a random subset of $\left\{1,...,m\right\}$, with each $i$ included independently with prob. $\frac{1}{2}$.\label{sample_step}}\hspace{1em}\phantom{.}\hspace{1em}
\Comment{Let $ \bv{\bar X} = \{\bv{x}_{i_1},\bv{x}_{i_2},...,\bv{x}_{i_{|\mathcal{\bar S}|}}\}$ for $i_j \in \mathcal{\bar S}$ be the data sample corresponding to $\mathcal{\bar S}$.}\hspace{5em}\phantom{.}\hspace{3em}
\Comment{Let $\bv{\bar S} = [\bv{e}_{i_1},\bv{e}_{i_2},...,\bv{e}_{i_{|\mathcal{\bar S}|}}]$ be the sampling matrix corresponding to $\mathcal{\bar S}$.}
	\State{$\bv{\tilde S} := \algoname{RecursiveRLS-Nystr{\"o}m}(\bv{\bar X}, K, \lambda, \delta/3)$.} \label{recursive_call}
	\State{$\bv{\hat S} := \bv{\bar S} \cdot \bv{\tilde S}$.}\label{subset_call}
		\State{Set $\tilde l_i^{\lambda} := \frac{3}{2\lambda}\left(\bv{K} - \bv{K}\bv{\hat S}\left(\bv{\hat S}^T\bv{K}\bv{\hat S} + \lambda\bv{I}\right)^{-1}\bv{\hat S}^T\bv{K} \right)_{i,i}$ for each $i \in \left\{1,\ldots,m\right\}.$\hspace{8em}\phantom{.}\hspace{2em} }
	 \Comment{\textcolor{blue}{By Lemma \ref{michaels_lemma}, equals $\frac{3}{2}(\bv{B}(\bv{B}^T\bv{\hat S}\bv{\hat S}^T \bv{B} +\lambda \bv{I})^{-1}\bv{B}^T)_{i,i}$. $\bv{K}$ denotes the kernel matrix for datapoints $\{\bv{x}_1,\ldots,\bv{x}_m\}$ and kernel function $K$.}}\label{2factor}
\State{Set $p_i := \min \{ 1, \tilde l_i^{\lambda} \cdot 16\log(\sum \tilde l_i^{\lambda}/\delta)\}$ for each $i \in \left\{1,\ldots,m\right\}.$}\label{prob_compute}
\State{Initially set weighted sampling matrix $\bv{S}$ to be empty. For each $i \in \left\{1,\ldots,m\right\}$, with probability $p_i$, append the column $\frac{1}{\sqrt{p_i}}\bv{e}_i$ onto $\bv{S}$.} 
\State{\Return{$\bv{S}$.}}\label{return_line}
\end{algorithmic}
\label{halvingFixed}
\end{algorithm}
Note that in Algorithm \ref{halvingFixed} the columns  of $\bv{S}$ are weighted by $1/\sqrt{p_i}$. The Nystr{\"o}m approximation $\bv{\tilde K} = \bv{K S} (\bv{S}^T\bv{K}\bv{S})^+ \bv{S}^T\bv{K}$ is not effected by column weights (see derivation in Section \ref{nystromPrelim}). However, weighting is necessary when the output is used in recursive calls (i.e., when $\bv{\tilde S}$ is used in Step \ref{subset_call}).

We prove Theorem \ref{thm:main_algo_theorem} via the following intermediate result:

\begin{theorem}\label{thm:halvingFixed} For any inputs $\bv{x}_1,\ldots,\bv{x}_m$, $K$, $\lambda>0$ and $\delta\in(0,1/32)$, let $\bv{K}$ be the kernel matrix for $\bv{x}_1,\ldots,\bv{x}_m$ and kernel function $K$ and let $d_\text{\normalfont eff}^\lambda(\bv{K})$ be the effective dimension of $\bv{K}$ with parameter $\lambda$. With probability $(1-3\delta)$, $\algoname{RecursiveRLS-Nystr{\"o}m}$ outputs $\bv{S}$ with $s$ columns that satisfies:
\begin{align}
\label{thedesiredbound}
\frac{1}{2}(\bv{B}^T\bv{B} + \lambda \bv{I}) \preceq (\bv{B}^T\bv{S}\bv{S}^T \bv{B} + \lambda \bv{I}) \preceq \frac{3}{2} (\bv{B}^T\bv{B} + \lambda \bv{I})\hspace{2em}\text{ for any } \bv{B}\text{ with }\bv{B}\bv{B}^T = \bv{K}.
\end{align}
Additionally, $s \leq s_{\max}(d_\text{\normalfont eff}^\lambda(\bv{K}),\delta)$ where $s_{\max}(w,z) \eqdef 384 \cdot \left(w+1\right)\log\left((w+1)/z\right)$. The algorithm uses $\leq c_1 m s_{\max}(d_\text{\normalfont eff}^\lambda(\bv{K}),\delta)$ kernel evaluations and $\leq c_2 m s_{\max}(d_\text{\normalfont eff}^\lambda(\bv{K}),\delta)^2$ additional computation time where $c_1$ and $c_2$ are fixed universal constants.
\end{theorem}

\begin{proof}
$\algoname{RecursiveRLS-Nystr{\"o}m}$ is a recursive algorithm and we prove Theorem \ref{thm:halvingFixed} via induction on the size of the input, $m$. In particular, we will show that, if Theorem \ref{thm:halvingFixed} holds for any all $m < n$, then it also holds for $m = n$. Our base case is $m = 1$.

 \medskip\noindent \textbf{Base case: Theorem \ref{thm:halvingFixed} holds for any inputs as long as $m = 1$.} \smallskip

Suppose $m = 1$, so the input data set just consists of a single point $\bv{x}_1$. 
Then the if statement on Line \ref{if_state} evaluates to true since $192\log(1/\delta) > 1$. So, $\bv{S}$ is set to a $1 \times 1$ identity matrix and \eqref{thedesiredbound} of Theorem \ref{thm:halvingFixed} holds trivially since  $(\bv{B}^T\bv{B} + \lambda \bv{I}) = (\bv{B}^T\bv{S}\bv{S}^T \bv{B} + \lambda \bv{I})$. Furthermore, $s = 1 \le s_{\max}(d_\text{\normalfont eff}^\lambda(\bv{K}),\delta)$ for any $d_\text{\normalfont eff}^\lambda(\bv{K})$ and $\delta$, as required. The algorithm runs in $O(1)$ time and performs no kernel evaluations, so the runtime requirements of Theorem \ref{thm:halvingFixed} also hold as long as $c_2$ set to a large enough constant. This all holds with probability $1$, and so for any input failure probability $\delta$.

 \medskip \noindent \textbf{Inductive Step: Theorem \ref{thm:halvingFixed} holds for $m = n$ as long as it holds for all $m < n$.  } \smallskip

\noindent Depending on the setting of $\delta$, we split our analysis into $2$ cases: 

\textbf{Case 1: The number of input data points $n$ is $< 192\log(1/\delta)$}.

\noindent In this case, as for the base case, the if statement on Line \ref{if_state} evaluates to true. $\bv{S}$ is set to an $n \times n$ identity matrix so \eqref{thedesiredbound} holds trivially. Furthermore, the number of samples $s$ is equal to $n$, and $n < 192\log(1/\delta) \le s_{\max}(d_\text{\normalfont eff}^\lambda(\bv{K}),\delta)$ as required. Again the algorithm doesn't compute any kernel dot products, the runtime bound required by Theorem \ref{thm:halvingFixed} holds, and all statements hold with probability $1$, which is $> 1-3\delta$ for any input failure probability $\delta$.

\textbf{Case 2: The number of input data points $n$ is $\geq 192\log(1/\delta)$}.

\noindent For this case we will use our inductive assumption since $\algoname{RecursiveRLS-Nystr{\"o}m}$ will call itself recursively at Step \ref{recursive_call}, for a smaller input size $m < n$.

We first note that the expected number of samples taken in Step \ref{sample_step} is $n/2$. I.e. $\E|\mathcal{\bar S}| = n/2$. By a standard multiplicative error Chernoff bound, with high probability the number of samples taken is not much larger than this expectation. This is important because it tells us that our problem size decreases substantially before we make the recursive call in Step \ref{recursive_call}. 
Following the simplified Chernoff bounds in e.g. \cite{mitzenmacher2017probability}, when $n \geq 192\log(1/\delta)$, and thus $\E|\mathcal{\bar S}| \geq 96\log(1/\delta)$, we have :
\begin{align}
\label{number_of_samples_bound}
\Pr\left[1 \leq |\mathcal{\bar S}| \leq .56n\right] \geq (1-\delta)
\end{align}
as long as $\delta < 1/32$, as required by Theorem \ref{thm:halvingFixed}.

So, with probability $(1-\delta)$, on Step \eqref{recursive_call}, $\algoname{RecursiveRLS-Nystr{\"o}m}$ is called recursively on a data set $\bv{\bar X}$ of size $\geq 1$ and $\leq .56n$. Accordingly, we can apply our inductive assumption that Theorem \ref{thm:halvingFixed} holds for all $m$ between $1$ and $n-1$ to conclude that, with probability $(1-3\cdot\delta/3)$\footnote{Note that in Step \ref{recursive_call} we run $\algoname{RecursiveRLS-Nystr{\"o}m}$ with failure probability $\delta/3$}:
\begin{enumerate}
\item Let $\bv{K}_{\mathcal{\bar S}}$ denote the kernel matrix for the data points in $\bv{\bar X}$ (corresponding to the sample $\mathcal{\bar S}$ with kernel function $K$. Then $\bv{B}_\mathcal{\bar S} = \bv{\bar S}^T\bv{B} $ satisfies $\bv{B}_\mathcal{\bar S}\bv{B}_\mathcal{\bar S}^T = \bv{K}_{\mathcal{\bar S}}$. Thus:
\begin{align}
\label{recurs_bound}
\frac{1}{2}(\bv{B}_\mathcal{\bar S}^T\bv{B}_\mathcal{\bar S} + \lambda \bv{I}) \preceq (\bv{B}_\mathcal{\bar S}^T\bv{\tilde S}\bv{\tilde S}^T \bv{B}_\mathcal{\bar S} + \lambda \bv{I}) \preceq \frac{3}{2} (\bv{B}_\mathcal{\bar S}^T\bv{B}_\mathcal{\bar S}+ \lambda \bv{I}).
\end{align}
\item $\bv{\tilde S}$ has $\leq s_{\max}(d_\text{\normalfont eff}^\lambda(\bv{K}_{\mathcal{\bar S}}), \delta/3)$ columns.
\item The recursive call at Step \ref{recursive_call} evaluates $K$, the kernel function, $\le c_1 \cdot |\mathcal{\bar S}| \cdot s_{\max}(d_\text{\normalfont eff}^\lambda(\bv{K}_{\mathcal{\bar S}}),\delta/3)$ times and uses $\le c_2\cdot |\mathcal{\bar S}| \cdot s_{\max}(d_\text{\normalfont eff}^\lambda(\bv{K}_{\mathcal{\bar S}}),\delta/3)^2$ additional runtime steps.
\end{enumerate}
We first use \eqref{recurs_bound} to prove \eqref{thedesiredbound}. We can write $\bv{K}_{\mathcal{\bar S}} = \bv{\bar S}^T \bv{K} \bv{\bar S}$. For all $i \in \{1,\ldots n\}$ let 
\begin{align*}
\bar{\ell}_i^\lambda &= \left(\bv{B}\left(\bv{B}^T \bv{\bar S} \bv{\bar S}^T\bv{B} + \lambda\bv{I}\right)^{-1}\bv{B}^T\right)_{i,i} & &\text{and} &
\bar{p}_i &= \min \{1, 16\bar l_i^{\lambda} \log(\sum_i \bar l_i^{\lambda}/\delta)\}.
\end{align*}
By Lemma \ref{uniformSampling}, since $\bv{\bar S}$ is constructed by sampling with probability $\frac{1}{2}$, with probability $1- \delta$, $\forall \ i$:
\begin{align}
\label{induc_tbound}
\bar{\ell}_i^\lambda &\geq \ell_i^\lambda(\bv{K}) & &\text{and}
&
\sum_{i=1}^n \bar{p}_i &\leq 64 d_{\text{\normalfont eff}}^{\lambda}(\bv{K}) \log\left(d_{\text{\normalfont eff}}^{\lambda}(\bv{K}) /  \delta\right).
\end{align}
Here $\ell_i^\lambda(\bv{K})$ is the exact $i^\text{th}$ $\lambda$-ridge leverage score of $\bv{K}$.

Now, since $\bv{B}_\mathcal{\bar S} = \bv{\bar S}^T\bv{B} $, it follows from \eqref{recurs_bound} and from the well known fact that $\bv{M} \preceq \bv{N} \implies \bv{N}^{-1} \preceq \bv{M}^{-1}$, that for any vector $\bv{z}$,
\begin{align*}
\frac{2}{3} \bv{z}^T\left(\bv{B}^T \bv{\bar S} \bv{\bar S}^T\bv{B} + \lambda\bv{I}\right)^{-1}\bv{z} \leq \bv{z}^T\left(\bv{B}^T \bv{\bar S}\bv{\tilde S}\bv{\tilde S}^T \bv{\bar S}^T\bv{B} + \lambda\bv{I}\right)^{-1}\bv{z}  \leq 2\bv{z}^T\left(\bv{B}^T \bv{\bar S} \bv{\bar S}^T\bv{B} + \lambda\bv{I}\right)^{-1}\bv{z}.
\end{align*}
Accordingly, since we set $\bv{\hat S} := \bv{\bar S} \cdot \bv{\tilde S}$, for all $i \in \{1, \ldots, n\}$
\begin{align}
\label{second_level_appox}
\bar{\ell}_i^\lambda \leq \frac{3}{2}\left(\bv{B}\left(\bv{B}^T\bv{\hat S}\bv{\hat S}^T\bv{W}\bv{B} + \lambda\bv{I}\right)^{-1}\bv{B}^T\right)_{i,i} \leq 3\bar{\ell}_i^\lambda.
\end{align}
By Lemma \ref{michaels_lemma}, the middle term is exactly equal to $\tilde l_i^{\lambda}$ as computed in Step \ref{2factor} of $\algoname{RecursiveRLS-Nystr{\"o}m}$. So combining \eqref{second_level_appox} and \eqref{induc_tbound} we have that:
\begin{align}\label{camManBound}
\tilde{\ell}_i^\lambda &\geq \ell_i^\lambda(\bv{K}) & &\text{and}
&
\sum_{i=1}^n p_i &\leq 192 d_{\text{\normalfont eff}}^{\lambda}(\bv{K}) \log\left(d_{\text{\normalfont eff}}^{\lambda}(\bv{K}) /  \delta\right).
\end{align}
The second bound holds because, as computed on Step \ref{prob_compute} of $\algoname{RecursiveRLS-Nystr{\"o}m}$, 
\begin{align*}
p_i = \min\{1,\tilde  l_i^\lambda \cdot 16 \log(\sum \tilde  l_i^\lambda/\delta) \} \leq 3 \min\{1,\bar  l_i^\lambda \cdot 16 \log(\sum \bar  l_i^\lambda/\delta) \} = 3 \bar p_i \le 192 d_{\text{\normalfont eff}}^{\lambda}(\bv{K}) \log\left(d_{\text{\normalfont eff}}^{\lambda}(\bv{K}) /  \delta\right)
\end{align*}
by \eqref{second_level_appox}.
Equation \eqref{camManBound} guarantees that $\bv{S}$ is sampled by over-estimates of the ridge leverage scores and we have a bound on the sum of the sampling probabilities. So, to establish \eqref{thedesiredbound}, we just apply the matrix Bernstein results of Lemma \ref{bernstein}. We conclude that, with probability $(1-\delta)$,
\begin{align*}
\frac{1}{2}(\bv{B}^T\bv{B} + \lambda \bv{I}) \preceq (\bv{B}^T\bv{ S}\bv{ S}^T \bv{B} + \lambda \bv{I}) \preceq \frac{3}{2} (\bv{B}^T\bv{B}  + \lambda \bv{I})\hspace{2em}\text{ for any } \bv{B}\text{ with }\bv{B}\bv{B} ^T = \bv{K}.
\end{align*}
The same lemma guarantees that $\bv{S}$ will have $s$ columns where 
\begin{align}
\label{count_bound}
\frac{1}{2}\sum p_i \leq s \leq 2\sum p_i.
\end{align}
$2\sum p_i \leq 384 d_{\text{\normalfont eff}}^{\lambda}(\bv{K}) \log\left(d_{\text{\normalfont eff}}^{\lambda}(\bv{K}) /  \delta\right) \leq s_{\max}(d_\text{\normalfont eff}^\lambda(\bv{K}),\delta)$ columns.

To finish our proof of Theorem \ref{thm:halvingFixed},
 we still need a bound on the number of kernel function evaluations used by the algorithm and on its overall runtime.

Kernel evaluations are performed both during the recursive call at Step \ref{recursive_call} and when computing approximate leverage scores at Step \ref{2factor}. Let $\tilde{s}$ be the number of columns in $\bv{\tilde S}$, and hence in $\bv{\hat S}$. At Step \ref{2factor}, $K$ needs to be evaluated $n\cdot (\tilde{s}+1)$ times: $n\tilde{s}$ times to compute $\bv{K}\bv{\hat S}$ and $n$ times to compute the diagonal of $\bv{K}$. Additionally, by the 3rd guarantee that comes from our inductive assumption, we need at most $c_1 \cdot |\mathcal{\bar S}| \cdot s_{\max}(d_\text{\normalfont eff}^\lambda(\bv{K}_{\mathcal{\bar S}}),\delta/3)$ kernel evaluations for the recursive call.
We claim that: 
\begin{align}
\label{deff_down_bound}
s_{\max}(d_\text{\normalfont eff}^\lambda(\bv{K}_{\mathcal{\bar S}}),\delta/3) \leq 1.317 s_{\max}(d_\text{\normalfont eff}^\lambda(\bv{K}), \delta).
\end{align}
This follows from Lemma \ref{decreasingScore2}: since $\bv{K}_{\mathcal{\bar S}} = \bv{\bar S}^T\bv{K}\bv{\bar S}$ and $ \bv{\bar S}\bv{\bar S}^T \preceq \bv{I}$ for any sampling matrix, $d_\text{\normalfont eff}^\lambda(\bv{K}_{\mathcal{\bar S}}) \leq d_\text{\normalfont eff}^\lambda(\bv{K})$. Additionally, we use that $\log(3/\delta) \leq 1.317\log(1/\delta) $ when $\delta  \leq 1/32$.

Using this bound and \eqref{number_of_samples_bound} we see that our total number of kernel evaluations is bounded by:
\begin{align*}
n\cdot (\tilde{s}+1) + c_1 \cdot |\mathcal{\bar S}| \cdot s_{\max}(d_\text{\normalfont eff}^\lambda(\bv{K}_{\mathcal{\bar S}}),\delta/3)
&\leq  n\cdot (s_{\max}(d_\text{\normalfont eff}^\lambda(\bv{K}_{\mathcal{\bar S}}),\delta/3) +1) + c_1 \cdot .56 n \cdot s_{\max}(d_\text{\normalfont eff}^\lambda(\bv{K}_{\mathcal{\bar S}}),\delta/3)\\
&\leq  \left(2.317 + .74 c_1\right)n \cdot s_{\max}(d_\text{\normalfont eff}^\lambda(\bv{K}), \delta).
\end{align*}
As long as $c_1 > 9$, the above is $< c_1 n s_{\max}(d_\text{\normalfont eff}^\lambda(\bv{K}), \delta)$, so we see that $\algoname{RecursiveRLS-Nystr{\"o}m}$ run on a data set of size $n$ performs no more kernel evaluations than that allowed by Theorem \ref{thm:halvingFixed}.

We finally bound runtime, accounting for the recursive call to $\algoname{RecursiveRLS-Nystr{\"o}m}$ and all other steps. Again, using the 3rd guarantee from our inductive assumption, \eqref{deff_down_bound}, and \eqref{number_of_samples_bound}  to bound $|\mathcal{\bar S}|$, the recursive call that computes $\bv{\tilde S}$ has runtime at most:
\begin{align*}
c_2\cdot |\mathcal{\bar S}| \cdot s_{\max}(d_\text{\normalfont eff}^\lambda(\bv{K}_{\mathcal{\bar S}}),\delta/3)^2
\leq .972c_2n \cdot s_{\max}(d_\text{\normalfont eff}^\lambda(\bv{K}), \delta)^2.
\end{align*}

In addition to the recursive call, the remaining runtime of the algorithm is dominated by the time to compute $\left(\bv{\hat S}^T\bv{K}\bv{\hat S} + \lambda\bv{I}\right)^{-1}$ and then to multiply this matrix by the $n \times \tilde{s}$ matrix $\bv{K}\bv{\hat S}$ at Step \ref{2factor}. Both of these operations and all other steps can be performed in $O(\tilde{s}^3 + n\tilde{s}^2)$ time. Since $\tilde{s} \leq n$, there is a constant $c$ such that the number of steps required for the algorithm besides the recursive call is $cn\tilde{s}^2 \leq cns_{\max}(d_\text{\normalfont eff}^\lambda(\bv{K}_{\mathcal{\bar S}}),\delta/3)^2$. Again applying \eqref{deff_down_bound}, our runtime is bounded by:
\begin{align*}
.972c_2n \cdot s_{\max}(d_\text{\normalfont eff}^\lambda(\bv{K}), \delta)^2 + cns_{\max}(d_\text{\normalfont eff}^\lambda(\bv{K}_{\mathcal{\bar S}}),\delta/3)^2
\end{align*}
which is $\leq c_2n \cdot s_{\max}(d_\text{\normalfont eff}^\lambda(\bv{K}), \delta)^2$ as long as $c_2 \geq 40 c$. 

The proof of our statements above relied on three events succeeding: \eqref{number_of_samples_bound}, \eqref{induc_tbound}, that the recursive call satisfies \eqref{recurs_bound} and the two following guarantees. Each of these events fails with probability at most $\delta$, so we conclude via a union bound that they all succeed with probability $1- 3\delta$. 

Accordingly, we have proven that Theorem \eqref{thm:halvingFixed} holds for fixed universal constants $c_1$ and $c_2$ for any input data set of size $n$ as long as it holds for any input data set of size $m$ with $1\leq m <n$. Along with our base case, this establishes the theorem for all input sizes.
\end{proof}

\begin{proof}[Proof of Theorem \ref{thm:main_algo_theorem}] 
Theorem \ref{thm:main_algo_theorem} is nearly a direct corollary of Theorem \ref{thm:halvingFixed}. In our proof of Theorem \ref{additiveErrorThm} we show that if 
\begin{align*}
\frac{1}{2}(\bv{B}^T\bv{B} + \lambda \bv{I}) \preceq (\bv{B}^T\bv{S}\bv{S}^T \bv{B} + \lambda \bv{I}) \preceq \frac{3}{2} (\bv{B}^T\bv{B} + \lambda \bv{I})
\end{align*}
for a weighted sampling matrix $\bv{S}$, then even if we remove the weights from $\bv{S}$ so that it has all unit entries (they don't effect the Nystr\"{o}m approximation), $\bv{\tilde K} = \bv{K S} (\bv{S}^T\bv{K}\bv{S})^+ \bv{S}^T\bv{K}$ satisfies:
\begin{align*}
\bv{\tilde K} \preceq \bv{K} \preceq \bv{\tilde K} + \lambda \bv{I}.
\end{align*}
The runtime bounds also follow nearly directly from Theorem \ref{thm:halvingFixed}. In particular, we have established that $O\left(ns_{\max}(d_\text{\normalfont eff}^\lambda(\bv{K}), \delta)\right)$ kernel evaluations and $O\left(ns_{\max}(d_\text{\normalfont eff}^\lambda(\bv{K}), \delta)^2\right)$ additional runtime are required by $\algoname{RecursiveRLS-Nystr{\"o}m}$. We only needed the upper bound to prove Theorem \ref{thm:halvingFixed}, but along the way \eqref{count_bound} actually showed that in a successful run of $\algoname{RecursiveRLS-Nystr{\"o}m}$, $\bv{S}$ has $\Theta\left(d_{\text{\normalfont eff}}^{\lambda}(\bv{K}) \log\left(d_{\text{\normalfont eff}}^{\lambda}(\bv{K}) /  \delta\right)\right)$ columns. Additionally, we may assume that $d_{\text{\normalfont eff}}(\bv{K}) \geq 1/2$. If it is not, then it's not hard to check (see proof of Lemma \ref{decreasingScore2}) that $\lambda$ must be $\geq \|\bv{K}\|$. If this is the case, the guarantee of Theorem \ref{thm:main_algo_theorem} is vacuous: \emph{any} Nystr\"{o}m approximation $\bv{\tilde K}$ satisfies $\bv{\tilde K} \preceq \bv{K} \preceq \bv{\tilde K} + \lambda\bv{I}$. With  $d_{\text{\normalfont eff}}(\bv{K}) \geq 1/2$, $d_{\text{\normalfont eff}}^{\lambda}(\bv{K}) \log\left(d_{\text{\normalfont eff}}^{\lambda}(\bv{K}) /  \delta\right)$ and thus $s$ are $\Theta(s_{\max}(d_\text{\normalfont eff}^\lambda(\bv{K}),\delta)$ so we conclude that Theorem \ref{thm:main_algo_theorem} uses $O(ns)$ kernel evaluations and $O(ns^2)$ additional runtime.
\end{proof}

\section{Empirical Evaluation}\label{sec:experiments}

We conclude with an empirical evaluation of our recursive Nystr{\"o}m method. We first introduce a variant of Algorithm \ref{halvingFixed} where, instead of choosing a regularization parameter $\lambda$, the user sets a sample size $s$ and $\lambda$ is automatically determined such that $s = \Theta(d_{\text{eff}}^\lambda \cdot \log(d_{\text{eff}}^\lambda/\delta))$. This variant is practically appealing as it essentially yields the best possible approximation to $\bv{K}$ for a fixed sample budget. Additionally, it is necessary in applications to kernel rank-$k$ PCA and $k$-means clustering, when $\lambda$ is unknown, but where we set $s \approx k\log k$ (see Appendices \ref{sec:pcp} and \ref{sec:apps}).

\subsection{Recursive RLS-Nystr{\"o}m algorithm for fixed sample size}

Given a fixed sample size $s$, we will control $\lambda$ using the following fact:
\begin{fact}[Proven in \eqref{sum_to_k}, Appendix \ref{sec:pcp}]\label{ridgeScore2kBound} For any $\bv{K}$ and integer $k$, for $\lambda = \frac{1}{k}\sum_{i=k+1}^n \sigma_i(\bv{K})$,
$
d_{\text{eff}}^\lambda \le 2k.
$
\end{fact}
If we choose $k$ such that $s \approx k\log k$ then setting $\lambda$ as above will yield an RLS-Nystr{\"o}m approximation with approximately $s$ sampled columns. The details are given in Algorithm \ref{halvingK}.

\begin{algorithm}[H]
\caption{\algoname{Recursive RLS-Nystr{\"o}m sampling, fixed sample size.}}
{\bf input}: $\bv{x}_1,\ldots,\bv{x}_m \in \mathcal{X}$, kernel function $K: \mathcal{X}\times \mathcal{X} \rightarrow \mathbb{R}$, sample size $s$, failure prob. $\delta \in (0,1/32)$\\
{\bf output}: sampling matrix $\bv{S} \in \mathbb{R}^{m \times s'}$.
\begin{algorithmic}[1]
\If{$m \leq s}$
	\State{\Return{$\bv{S} := \bv{I}_{m\times m}$.}}
\EndIf
\State{Let $\mathcal{\bar S}$ be a random subset of $\left\{1,...,m\right\}$, with each $i$ included independently with probability $\frac{1}{2}$.\label{sample_step2}}\
\Comment{Let $ \bv{\bar X} = \{\bv{x}_{i_1},\bv{x}_{i_2},...,\bv{x}_{i_{|\mathcal{\bar S}|}}\}$ for $i_j \in \mathcal{\bar S}$ be the data sample corresponding to $\mathcal{\bar S}$.}\hspace{5em}\phantom{.}\hspace{3em}
\Comment{Let $\bv{\bar S} = [\bv{e}_{i_1},\bv{e}_{i_2},...,\bv{e}_{i_{|\mathcal{\bar S}|}}]$ be the sampling matrix corresponding to $\mathcal{\bar S}$.}
\State{$\bv{\tilde S} := \algoname{RecursiveRLS-Nystr{\"o}m}(\bv{\bar X}, K, s, \delta/3)$.} \label{recursive_call2}
\State{$\bv{\hat S} := \bv{\bar S}\cdot  \bv{\tilde S}$.}\label{subset_call2}
\State{Set $k$ to the maximum integer with $c k \log(2k/\delta) \le s$, where $c$ is some fixed constant.}
\State{$\tilde \lambda := \frac{1}{k} \sum_{i=k+1}^n \sigma_i(\bv{\hat S}^T \bv{K}\bv{\hat S})$}\Comment{\textcolor{blue}{Approximate $\lambda$}}
\State{Set $\tilde l_i^{\lambda} := \frac{5}{\tilde \lambda}\left(\bv{K} - \bv{K}\bv{\hat S}\left(\bv{\hat S}^T\bv{K}\bv{\hat S} + \tilde \lambda\bv{I}\right)^{-1}\bv{\hat S}^T\bv{K} \right)_{i,i}$ for each $i \in \{1,...,m\}$.}\hspace{9em}\phantom{.}\hspace{2em}
\Comment{\textcolor{blue}{By Lemma \ref{michaels_lemma}, equals $\frac{3}{2}(\bv{B}(\bv{B}^T\bv{\hat S}\bv{\hat S}^T \bv{B} +\tilde \lambda \bv{I})^{-1}\bv{B}^T)_{i,i}$. $\bv{K}$ denotes the kernel matrix for datapoints $\{\bv{x}_1,\ldots,\bv{x}_m\}$ and kernel function $K$.}}\label{step5}
\State{Set  $p_i := \min \{ 1, \tilde l_i^\lambda \cdot 16\log(2k/\delta)\}$ for each $i \in \{1,...,,\}$.}
\State{Initially set weighted sampling matrix $\bv{S}$ to be empty. For each $i \in \left\{1,\ldots,m\right\}$, with probability $p_i$, append the column $\frac{1}{\sqrt{p_i}}\bv{e}_i$ onto $\bv{S}$.} 
\State{\Return{$\bv{S}$}}
\end{algorithmic}
\label{halvingK}
\end{algorithm}\begin{theorem}
\label{thm:main_algo_k}
For sufficiently  large universal constant $c$, let $k$ be any positive integer with $s \geq c k \log(2k/\delta)$ and $\lambda = \frac{1}{k} \sum_{i=k+1}^n \sigma_i(\bv{K})$. 
Let $\bv{S} \in \mathbb{R}^{n\times s'}$ be computed by
Algorithm \ref{halvingK}. With probability $1-3\delta$, $s' \leq 2s$, $\bv{S}$ is sampled by overestimates of the $\lambda$-ridge leverage scores of $\bv{K}$, and the Nystr{\"o}m approximation $\bv{\tilde K} = \bv{K S} (\bv{S}^T\bv{K}\bv{S})^+ \bv{S}^T\bv{K}$ satisfies the guarantee of Theorem \ref{additiveErrorThm}. Algorithm \ref{halvingK} uses
$O(n s)$ kernel evaluations and $O(n s^2 )$ runtime. 
\end{theorem}
For the $\lambda$ given in Theorem \ref{thm:main_algo_k}, we have $d_{\text{eff}}^\lambda = \Theta(k)$. Hence, since we set $s =  \Theta(k\log k/\delta)$, additive error $\lambda$ is essentially the smallest we can obtain using an $s$ sample Nystr{\"o}m approximation.
The proof of Theorem \ref{thm:main_algo_k} is similar to that of Theorem \ref{thm:main_algo_theorem}. We defer it to Appendix \ref{sec:additional}.

\subsection{Performance of Recursive RLS-Nystr{\"o}m for kernel approximation}\label{sec:kernelPerf}
We evaluate Algorithm \ref{halvingK} on the datasets listed in Table \ref{tab:datasets}, comparing against the classic Nystr{\"o}m method with uniform sampling \cite{williams2001using} and the random Fourier features method \cite{rahimi2007random}. Implementations were in MATLAB and run on a 2.6 GHz Intel Core i7 with $16$GB of memory.
\begin{table}[h]
\small
\centering 
\begin{threeparttable}[b]
\begin{tabular}{||c||c|c|c||}
\hhline{|t:=:t:===:t|}
Dataset  & \specialcell{\# of Data Points\\ $n$} & \specialcell{\# of Features \\ $d$} & Link
\\\hhline{|:=::===:|}
\texttt{YearPredictionMSD} & 515345 & 90 & \tiny{\url{https://archive.ics.uci.edu/ml/datasets/YearPredictionMSD}}
\\\hhline{||-||-|-|-||}
\texttt{Covertype} & 581012 & 54 & \tiny{\url{https://archive.ics.uci.edu/ml/datasets/Covertype}}
\\\hhline{||-||-|-|-||}
\texttt{Cod-RNA} & 331152 & 8 & \tiny{\url{https://www.csie.ntu.edu.tw/~cjlin/libsvmtools/datasets/}}
\\\hhline{||-||-|-|-||}
\texttt{Adult} & 48842 & 110 & \tiny{\url{https://archive.ics.uci.edu/ml/datasets/Adult}}
\\\hhline{|b:=:b:===:b|}
\end{tabular}
\caption{Datasets downloaded from UCI ML Repository \cite{Lichman:2013}, except \texttt{Cod-RNA} \cite{uzilov2006detection}.} \vspace{.25em}
\label{tab:datasets}
\end{threeparttable}
\end{table}

For each dataset, we split categorical features into binary indicatory features and mean center and normalize all features to have variance 1. We use a Gaussian kernel for all tests, with the width parameter $\sigma$ selected via cross validation on regression and classification tasks. To compute $\|\bv{K} - \bv{\tilde K}\|_2$, we only process a random subset of 20k data points since otherwise multiplying by the full kernel matrix $\bv{K}$ to compute $\|\bv{K} - \bv{\tilde K}\|_2$ is prohibitively expensive. Experiments on the full kernel matrices are discussed in Section \ref{subsec:learning}.
\begin{figure}[h]
\centering
\begin{subfigure}{0.24\textwidth}
\centering
\includegraphics[width=\textwidth]{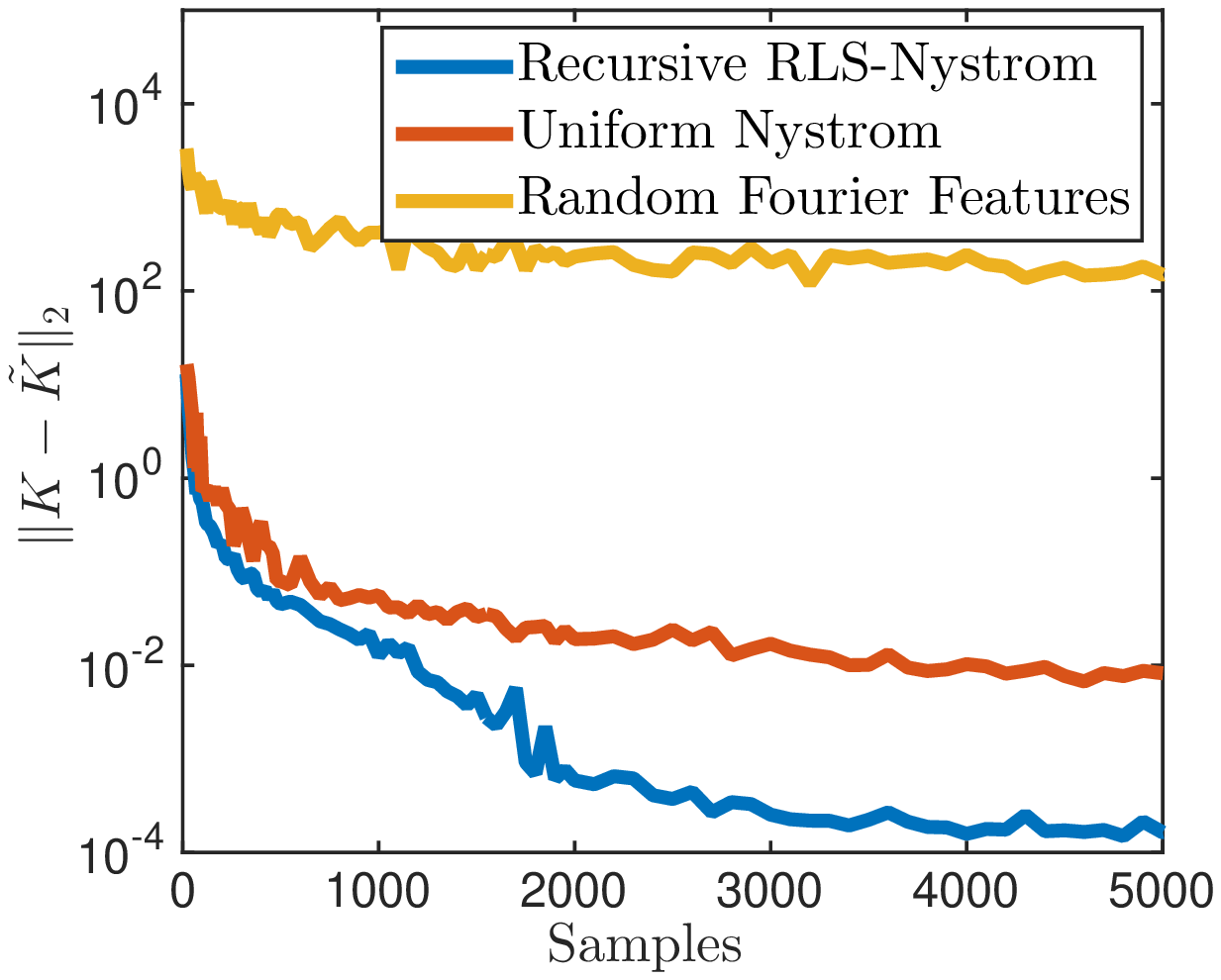} 
\caption{\texttt{Adult} }
\end{subfigure}
\begin{subfigure}{0.24\textwidth}
\centering
\includegraphics[width=\textwidth]{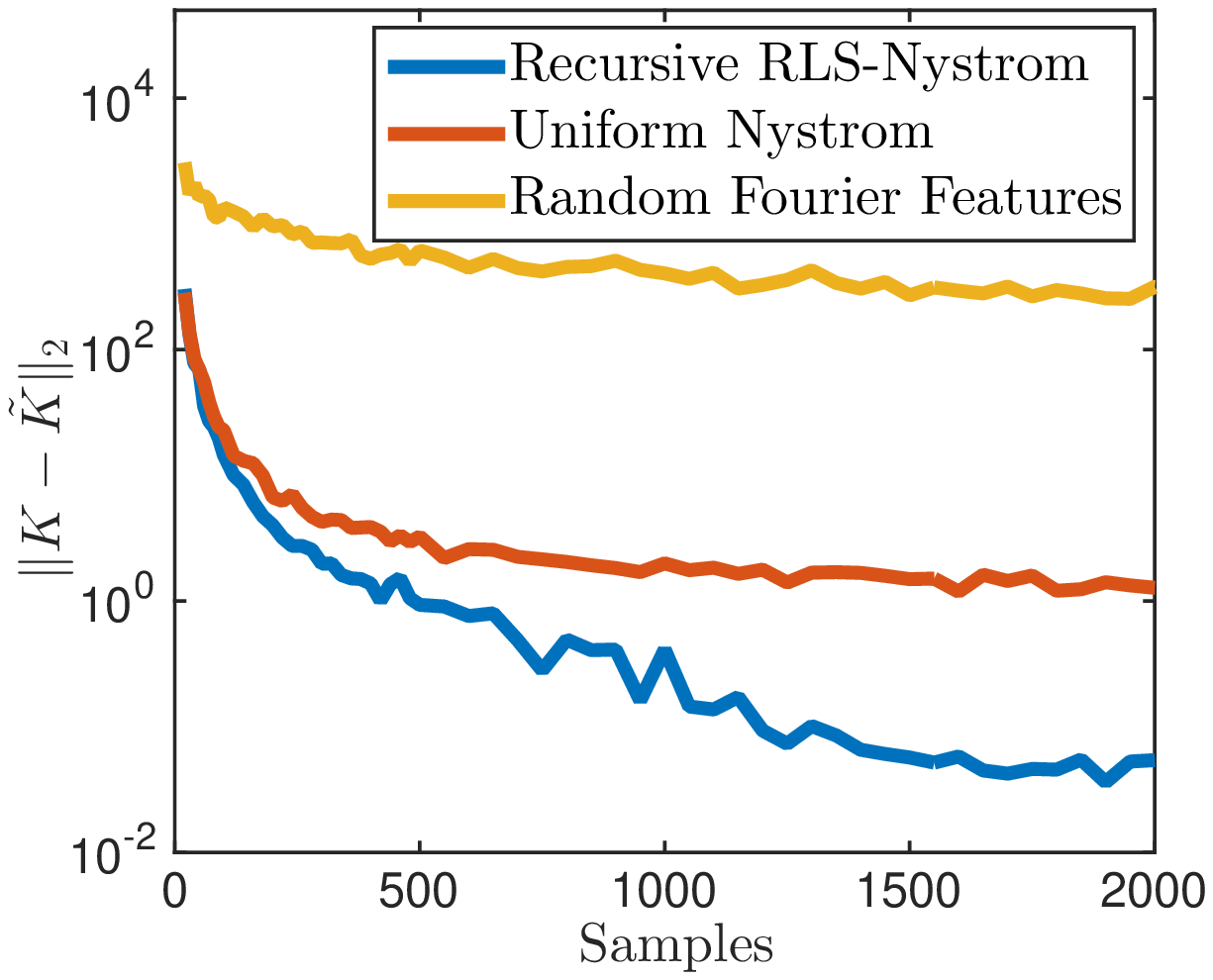} 
\caption{\texttt{Covertype} }
\end{subfigure}
\begin{subfigure}{0.24\textwidth}
\centering
\includegraphics[width=\textwidth]{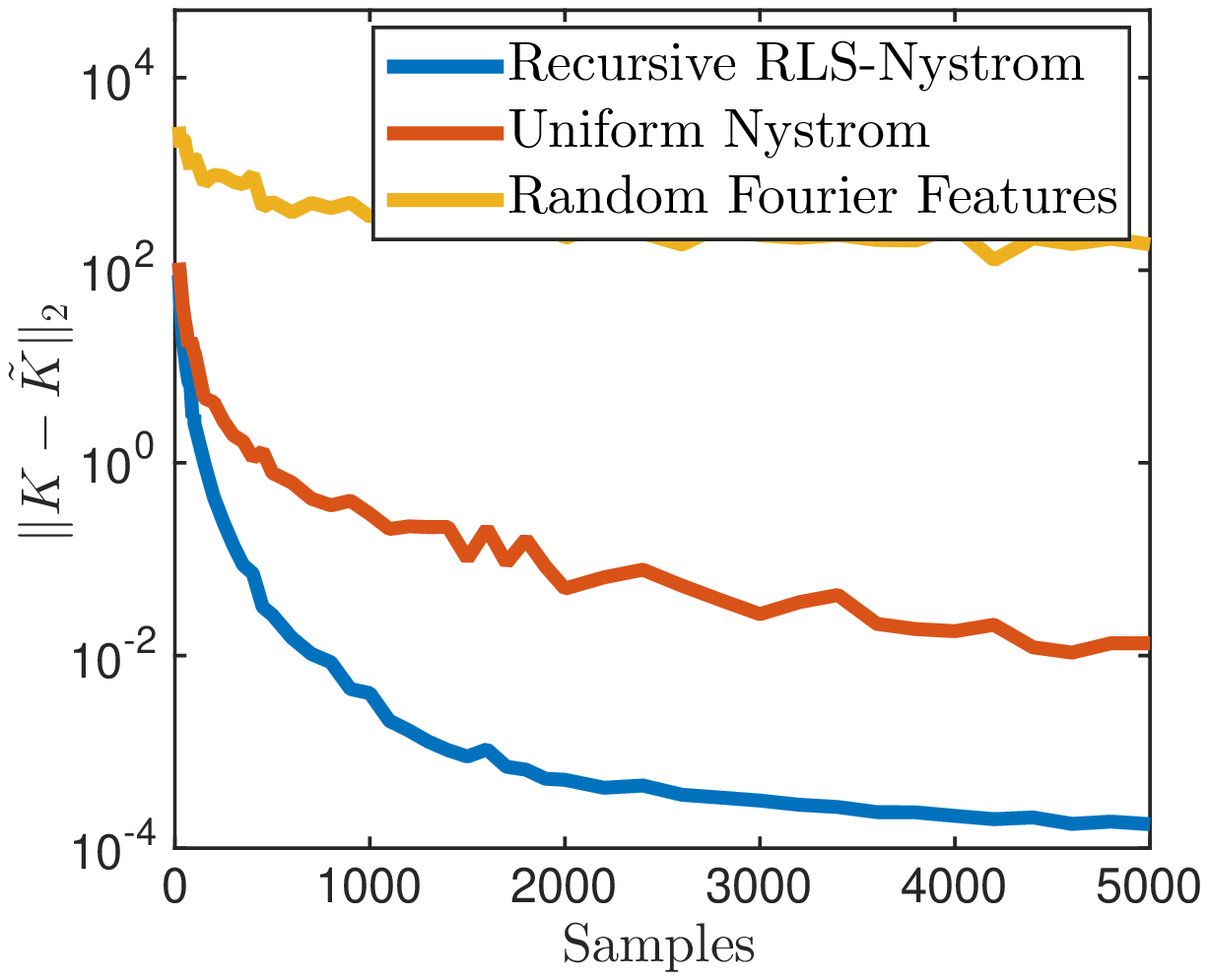} 
\caption{\texttt{Cod-RNA} }
\end{subfigure}
\begin{subfigure}{0.24\textwidth}
\centering
\includegraphics[width=\textwidth]{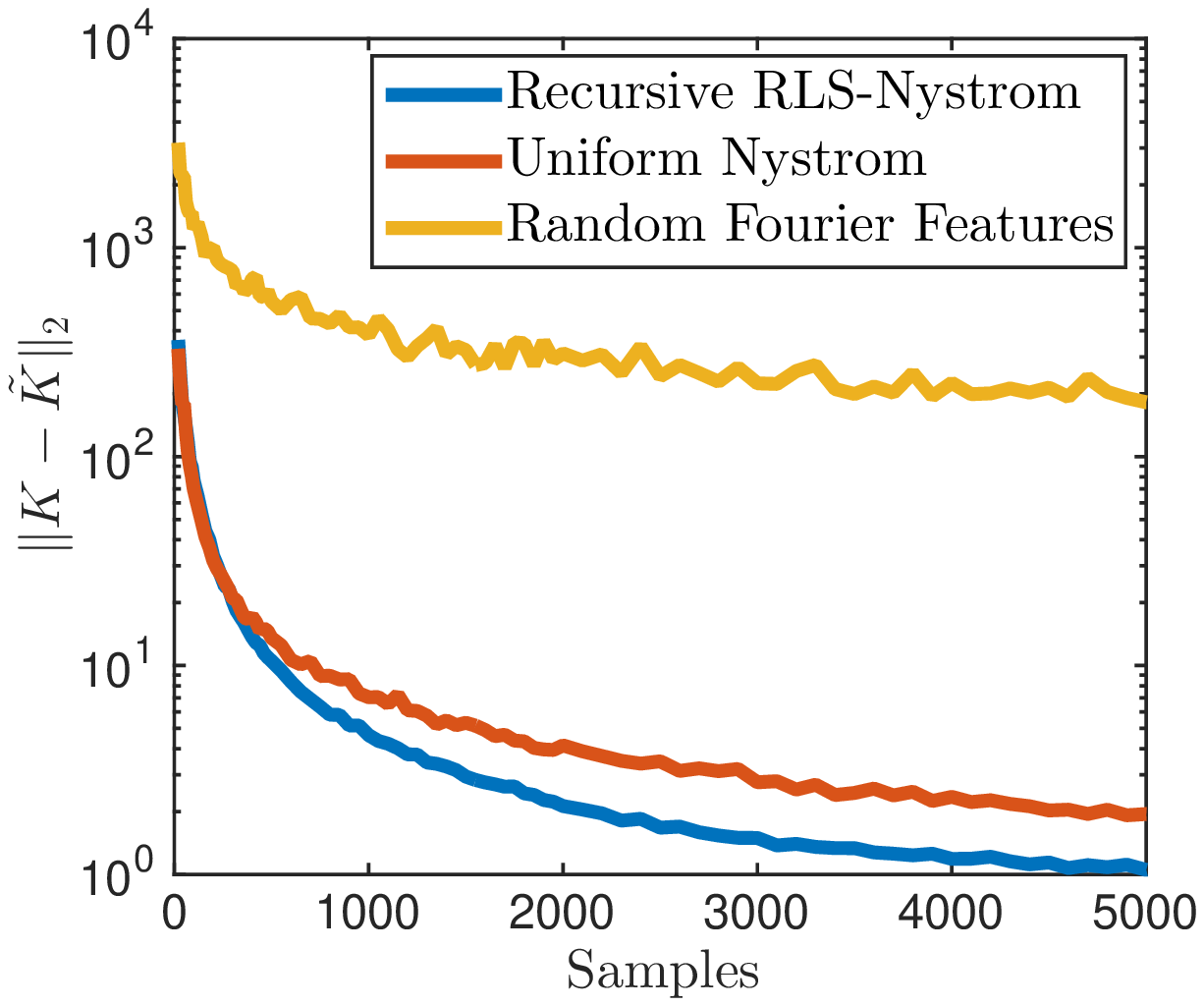} 
\caption{\texttt{YearPredictionMSD} }
\end{subfigure}
\caption{For a given number of samples, Recursive RLS-Nystr{\"o}m yields approximations with lower error, measured by $\|\bv{K} - \bv{\tilde K}\|_2$. Error is plotted on a logarithmic scale, averaged over 10 trials.}
\label{fig:samples}
\end{figure}

Figure \ref{fig:samples} confirms that Recursive RLS-Nystr{\"o}m consistently obtains better kernel approximation error than the other methods. The advantage of Nystr{\"o}m over random Fourier features is substantial -- this is unsurprising as the Nystr{\"o}m methods are data dependent and based on data projection, as opposed to pointwise approximation of $\bv{K}$. 
Even between the Nystr{\"o}m methods there is a substantial difference in kernel approximation, especially for large sample sizes. 

As we can see in Figure \ref{fig:runtimes}, with the exception of \texttt{YearPredictionMSD}, the better quality of the landmarks obtained with Recursive RLS-Nystr{\" o}m translates into runtime improvements. While the cost \emph{per sample} is higher for our method at $O(nd + ns)$ time versus $O(nd + s^2)$ for uniform Nystr{\"o}m and $O(nd)$ for random Fourier features, since RLS-Nystr{\"o}m requires fewer samples it more quickly obtains $\bv{\tilde K}$ with a given accuracy. $\bv{\tilde K}$ will also have lower rank, which can accelerate processing in downstream applications. For example, to achieve $\|\bv{K} - \bv{\tilde K}\|_2 \leq 1$ for the \texttt{Covertype} dataset, Recursive RLS-Nystr{\"o}m requires 650 samples in comparison to 3800 for uniform Nystr{\"o}m.

\begin{figure}[h!]
\centering
\begin{subfigure}{0.24\textwidth}
\centering
\includegraphics[width=\textwidth]{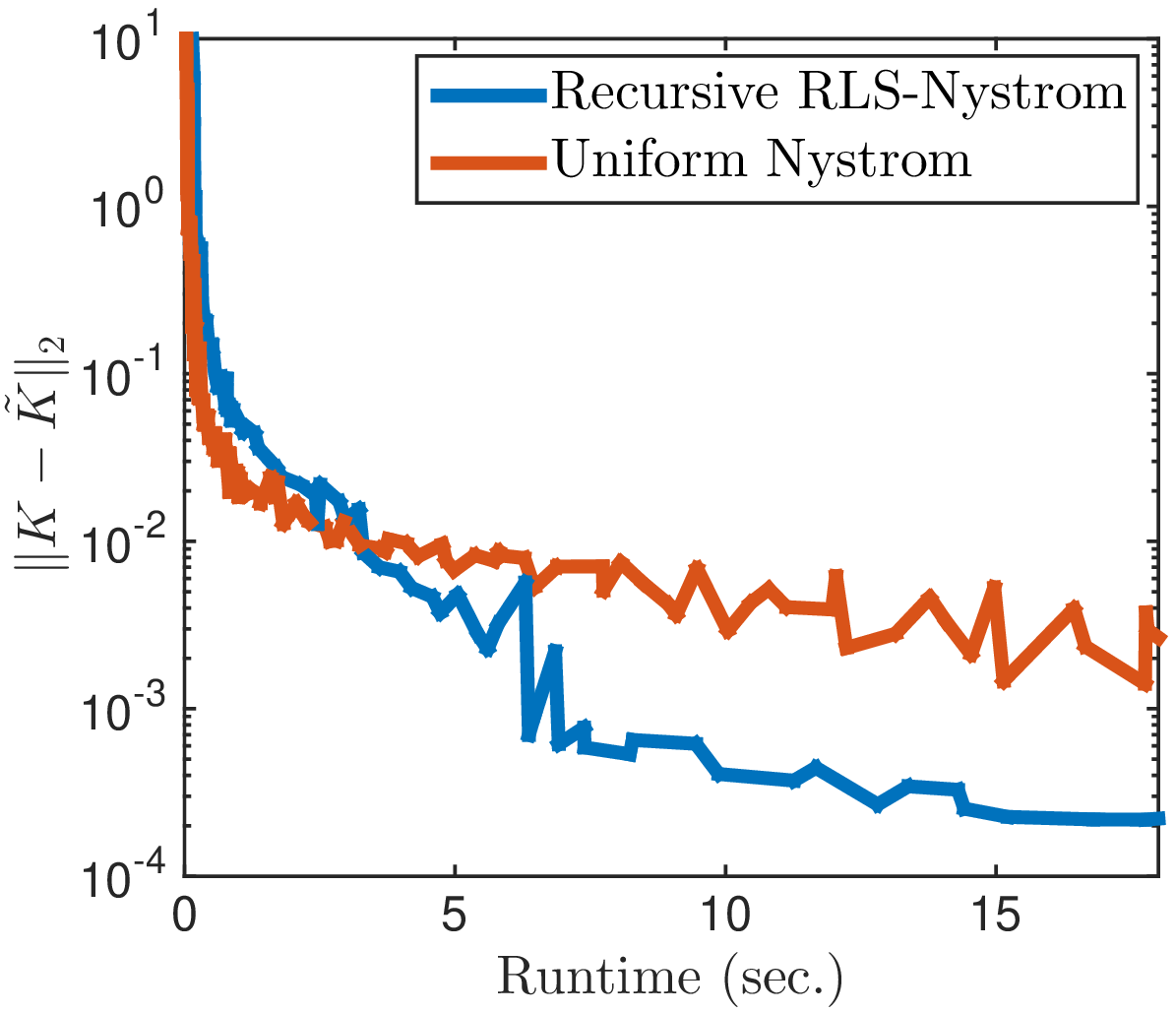} 
\caption{\texttt{Adult} }
\end{subfigure}
\begin{subfigure}{0.24\textwidth}
\centering
\includegraphics[width=\textwidth]{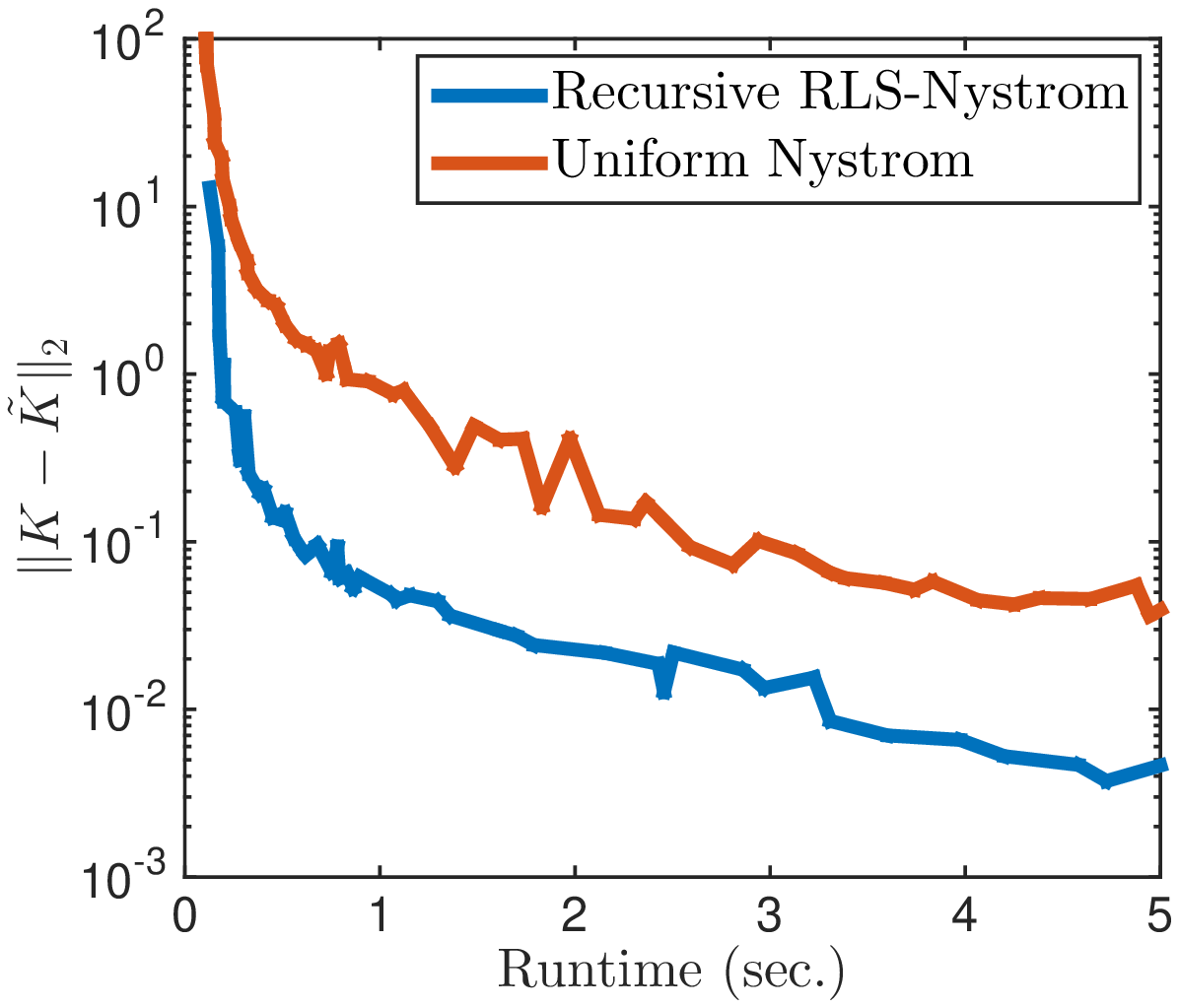} 
\caption{\texttt{Covertype} }
\end{subfigure}
\begin{subfigure}{0.24\textwidth}
\centering
\includegraphics[width=\textwidth]{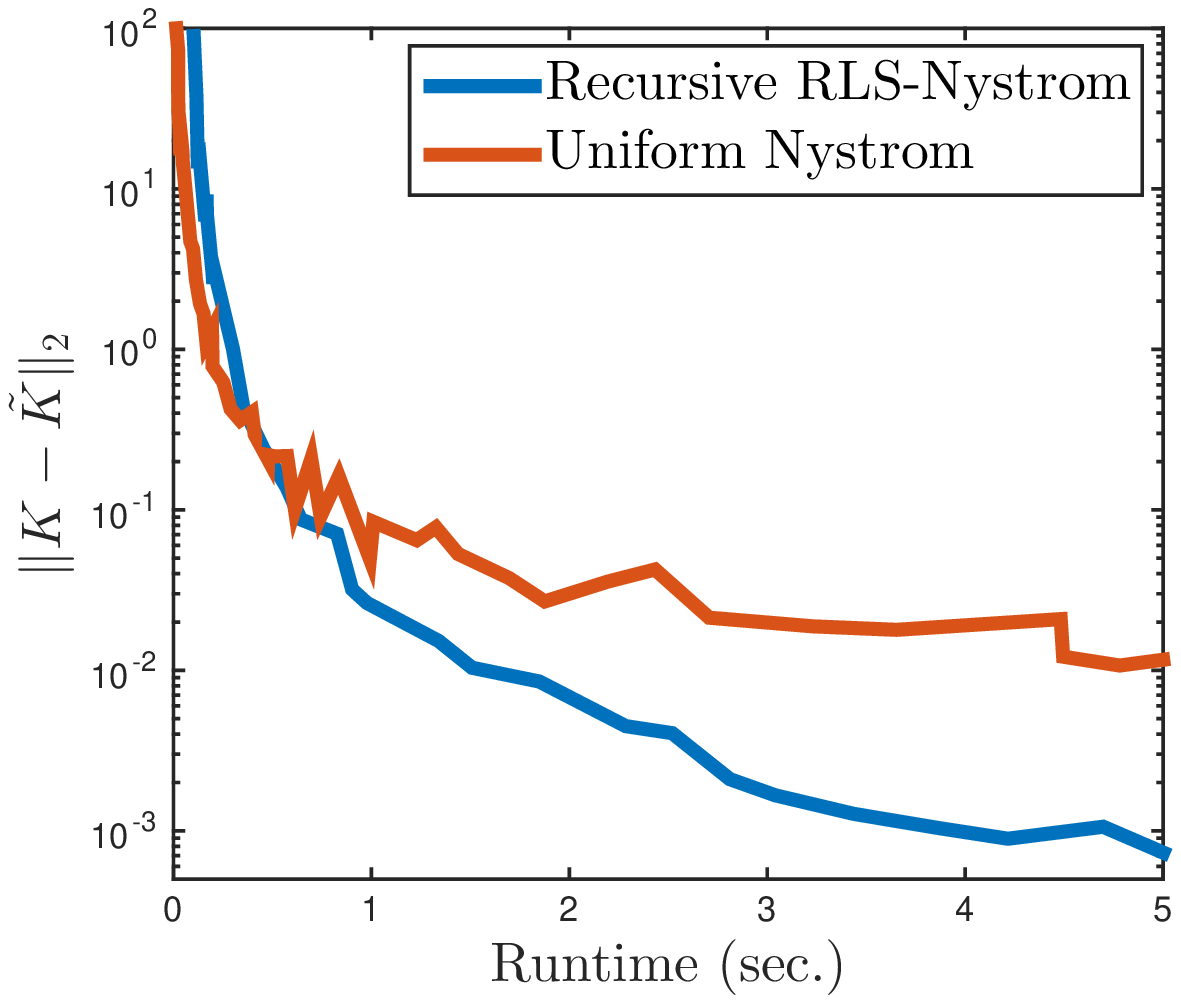} 
\caption{\texttt{Cod-RNA} }
\end{subfigure}
\begin{subfigure}{0.24\textwidth}
\centering
\includegraphics[width=\textwidth]{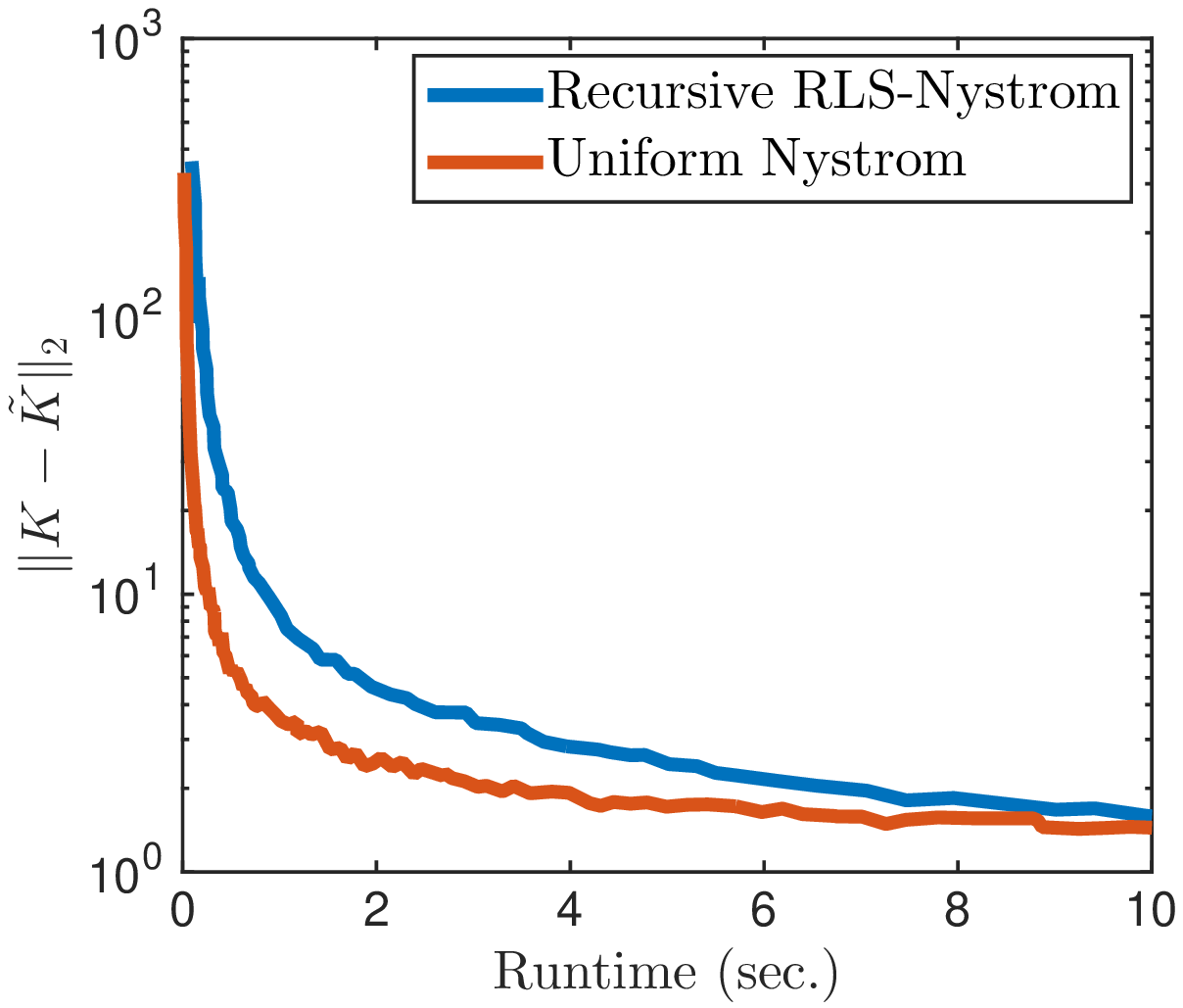} 
\caption{\texttt{YearPredictionMSD} }
\end{subfigure}
\caption{Especially for small error, Recursive RLS-Nystr{\"o}m typically obtains a fixed level of approximation faster than uniform sampling. It only underperformed uniform sampling for the \texttt{YearPredictionMSD} dataset. Results for random Fourier features are excluded from this plot: while the method is faster than Nystr{\"o}m, it never obtained high enough accuracy to be directly comparable. Error is plotted on a log scale, with results averaged over 10 trials.}
\label{fig:runtimes}
\end{figure}

\vspace{-1.25em}
\subsubsection{Accelerated recursive method}
While Recursive RLS-Nystr{\"o}m typically outperforms classic Nystr{\"o}m,
 on datasets with relatively uniform ridge leverage scores, such as \texttt{YearPredictionMSD}, it only narrowly beats uniform sampling in terms accuracy. As a result it incurs a higher runtime cost since it is slower per sample. 

To combat this issue we implement a simple heuristic modification of our algorithm. We note that  the final cost of computing the Nystr{\"o}m factors $\bv{K S}$ and $(\bv{S}^T\bv{K}\bv{S})^+$ is $O(ns + s^3)$ for both methods. Recursive RLS-Nystr{\"o}m is only slower because computing leverage scores at intermediate levels of recursion takes $O(ns^2)$ time (Step \ref{step5}, Algorithm \ref{halvingK}) . This cost can be improved by simply adjusting the regularization $\lambda$ to restrict the sample size on each recursive call to be $< s$. Specifically, we can balance runtimes by taking $\approx \sqrt{(ns + s^3)/n}$ samples on lower levels. 

Doing so improves our runtime, bringing the per sample cost down to approximately that of random Fourier features and uniform Nystr{\"o}m (Figure \ref{matchingPerSample}) while nearly maintaining the same approximation quality. For datasets such as \texttt{Covertype} in which Recursive RLS-Nystr{\"o}m performs significantly better than uniform sampling, so does the accelerated method (see Figure \ref{matchingError}). However, the performance of the accelerated method does not degrade when leverage scores are relatively uniform -- it still offers the best runtime to approximation quality tradeoff (Figure \ref{tradeoff}).

We note that further runtime improvements may be possible. Subsequent work extends fast ridge leverage score methods to distributed and streaming environments \cite{calandriello:hal-01482760}. Empirical evaluation of these techniques could lead to even more scalable, high accuracy Nystr{\"o}m methods. 
%

\begin{figure}[h]
\centering
\begin{subfigure}{0.32\textwidth}
\centering
\includegraphics[width=\textwidth]{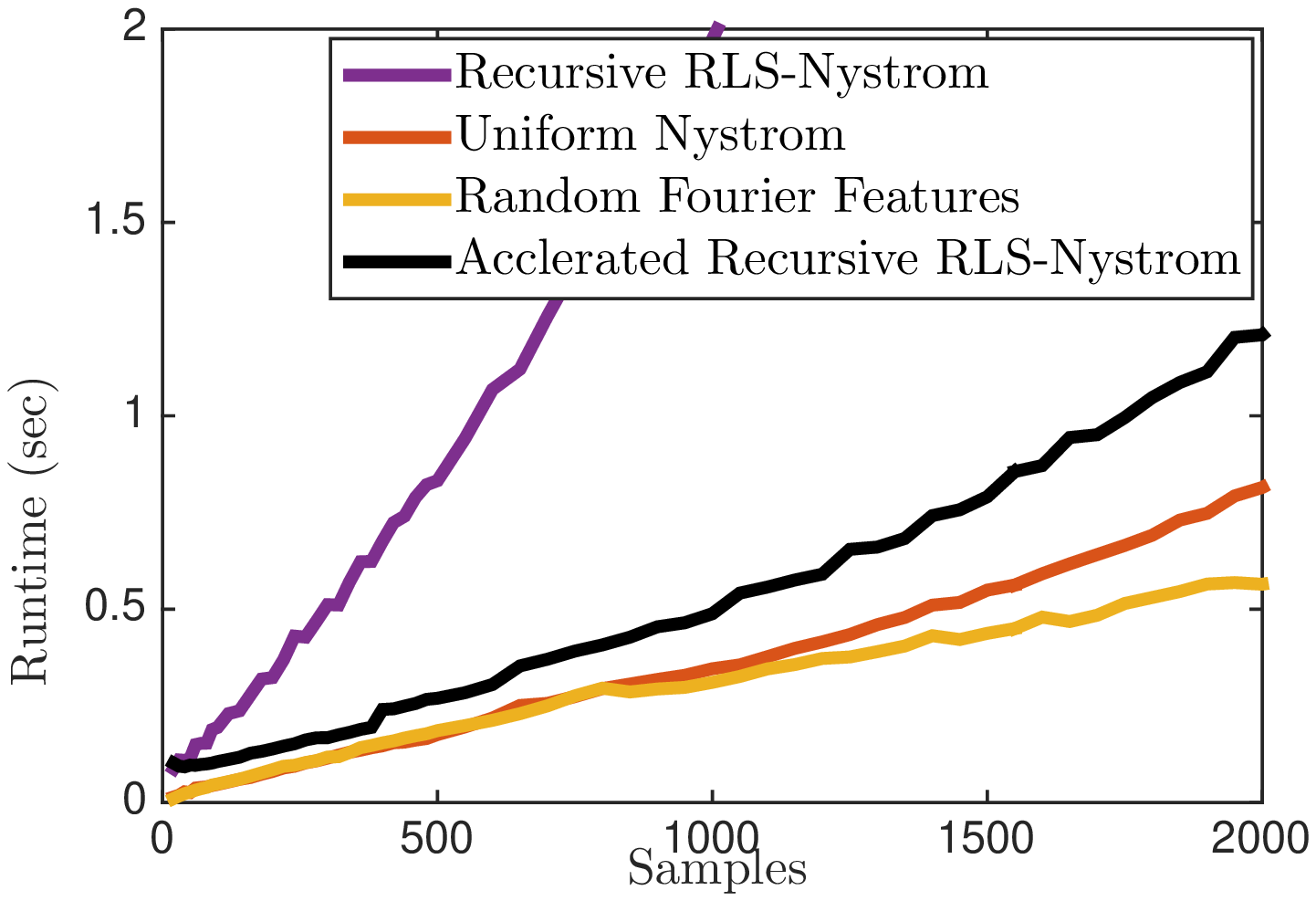} 
\caption{Runtimes for \texttt{Covertype}.\vspace{1em}}
\label{matchingPerSample}
\end{subfigure}
\begin{subfigure}{0.32\textwidth}
\centering
\includegraphics[width=\textwidth]{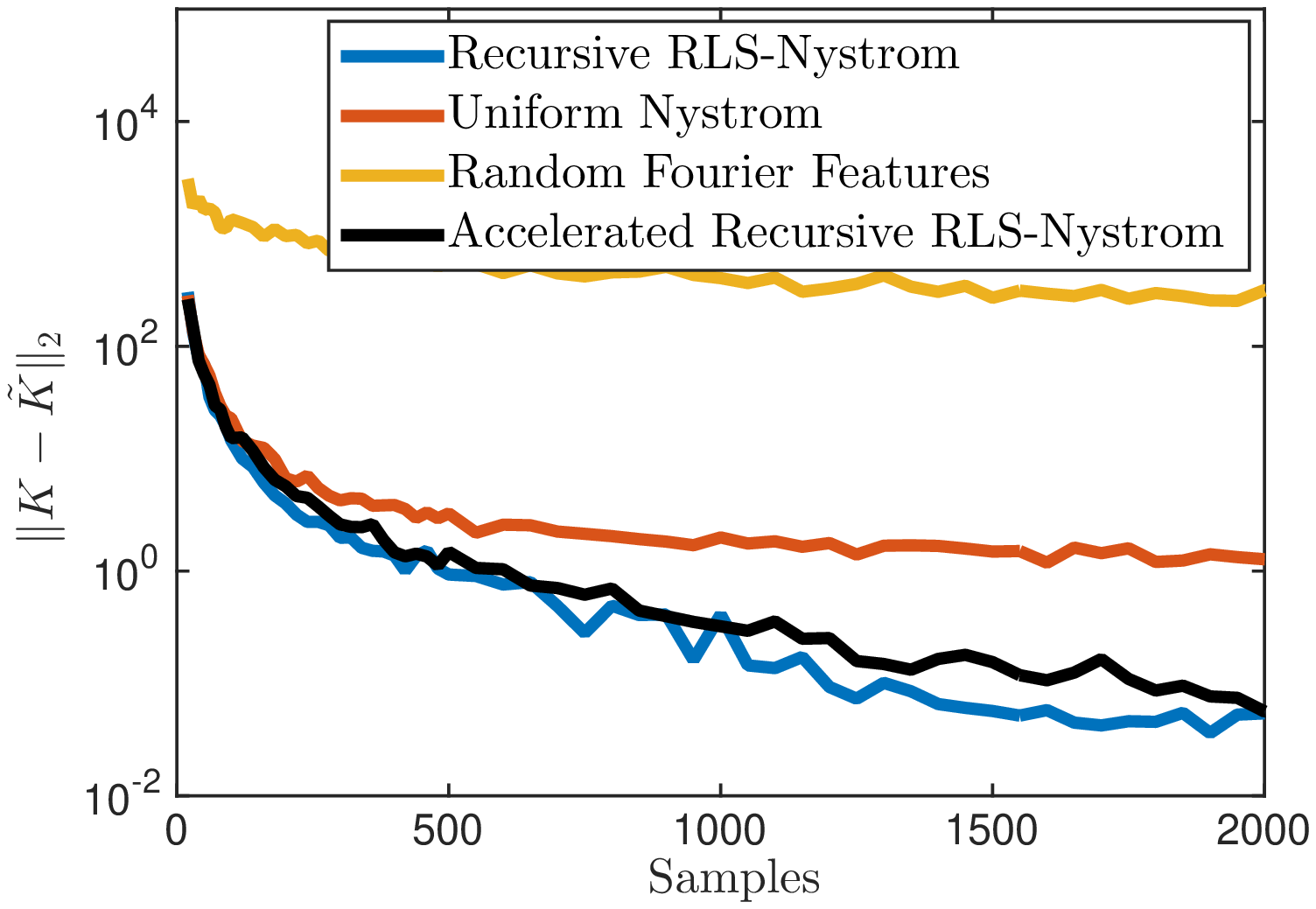} 
\caption{Errors for \texttt{Covertype}.\vspace{1em}}
\label{matchingError}
\end{subfigure}
\begin{subfigure}{0.32\textwidth}
\centering
\includegraphics[width=\textwidth]{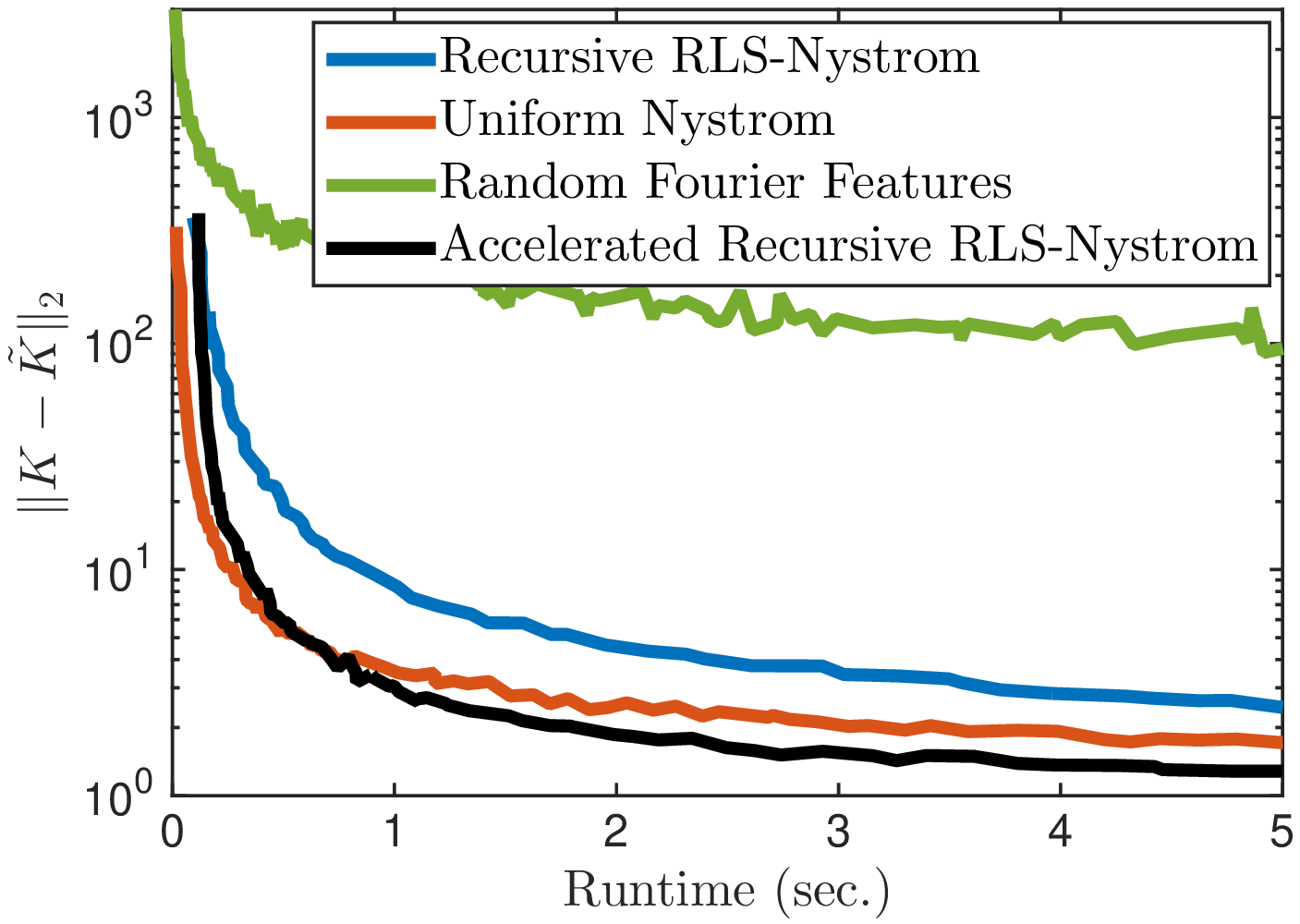} 
\caption{\centering Runtime/error tradeoff for \texttt{YearPredictionMSD}.}
\label{tradeoff}
\end{subfigure}
\caption{Our accelerated Recursive RLS-Nystr{\"o}m, which undersamples at intermediate recursive calls, nearly matches the \emph{per sample runtime} of random Fourier features and uniform Nystr{\"o}m while still providing approximation nearly as good as the standard Recursive RLS-Nystr{\"o}m. For datasets like  \texttt{YearPredictionMSD} with relatively uniform kernel leverage scores, the accelerated version offers the best runtime vs. approximation tradeoff. All results are averaged over 10 trials.}
\label{fig:runtimesFast}
\end{figure}

\subsection{Performance of Recursive RLS-Nystr{\"o}m for learning tasks}
\label{subsec:learning}

\begin{figure}[h!]
\centering
\begin{subfigure}{0.33\textwidth}
\centering
\includegraphics[width=\textwidth]{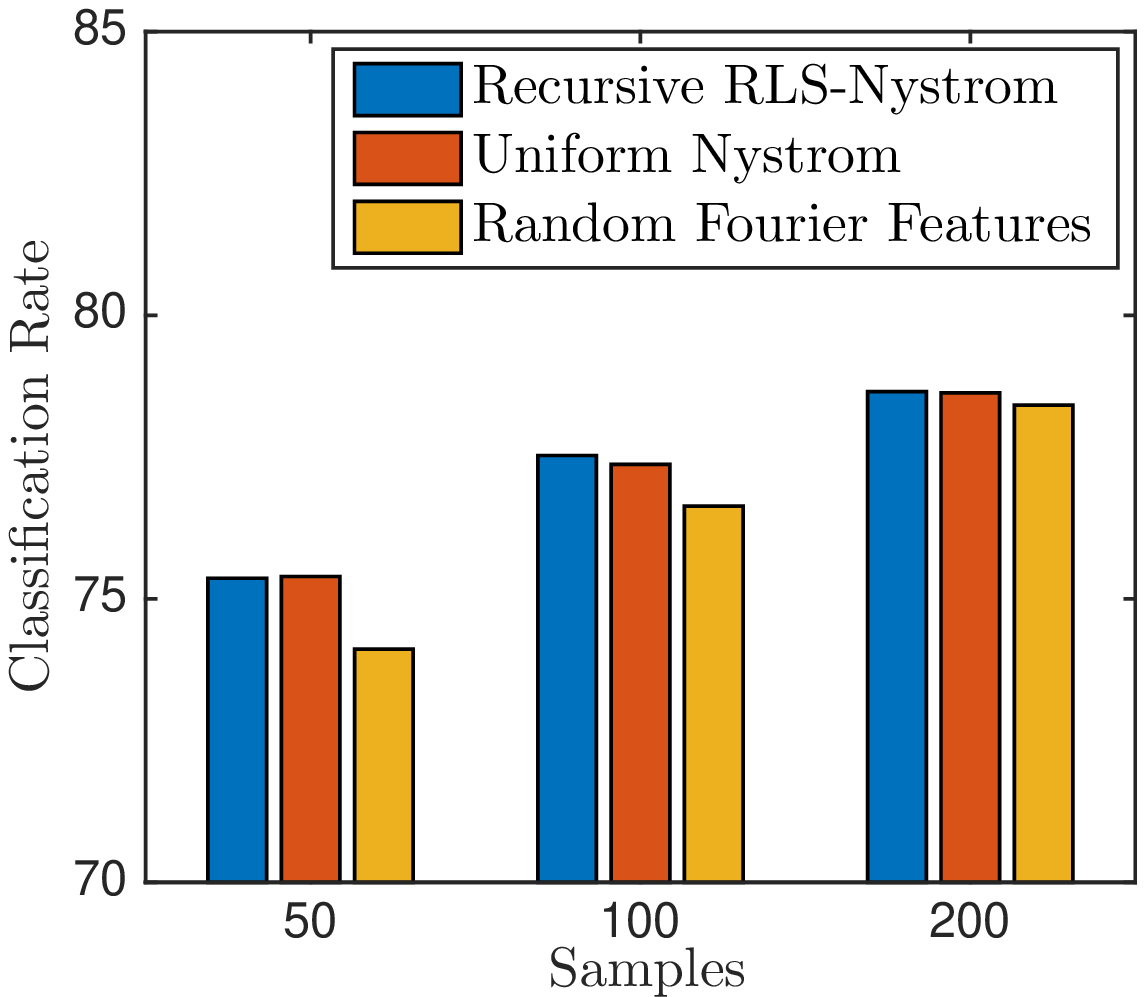} 
\caption{\texttt{Covertype} }
\end{subfigure}
\begin{subfigure}{0.33\textwidth}
\centering
\includegraphics[width=\textwidth]{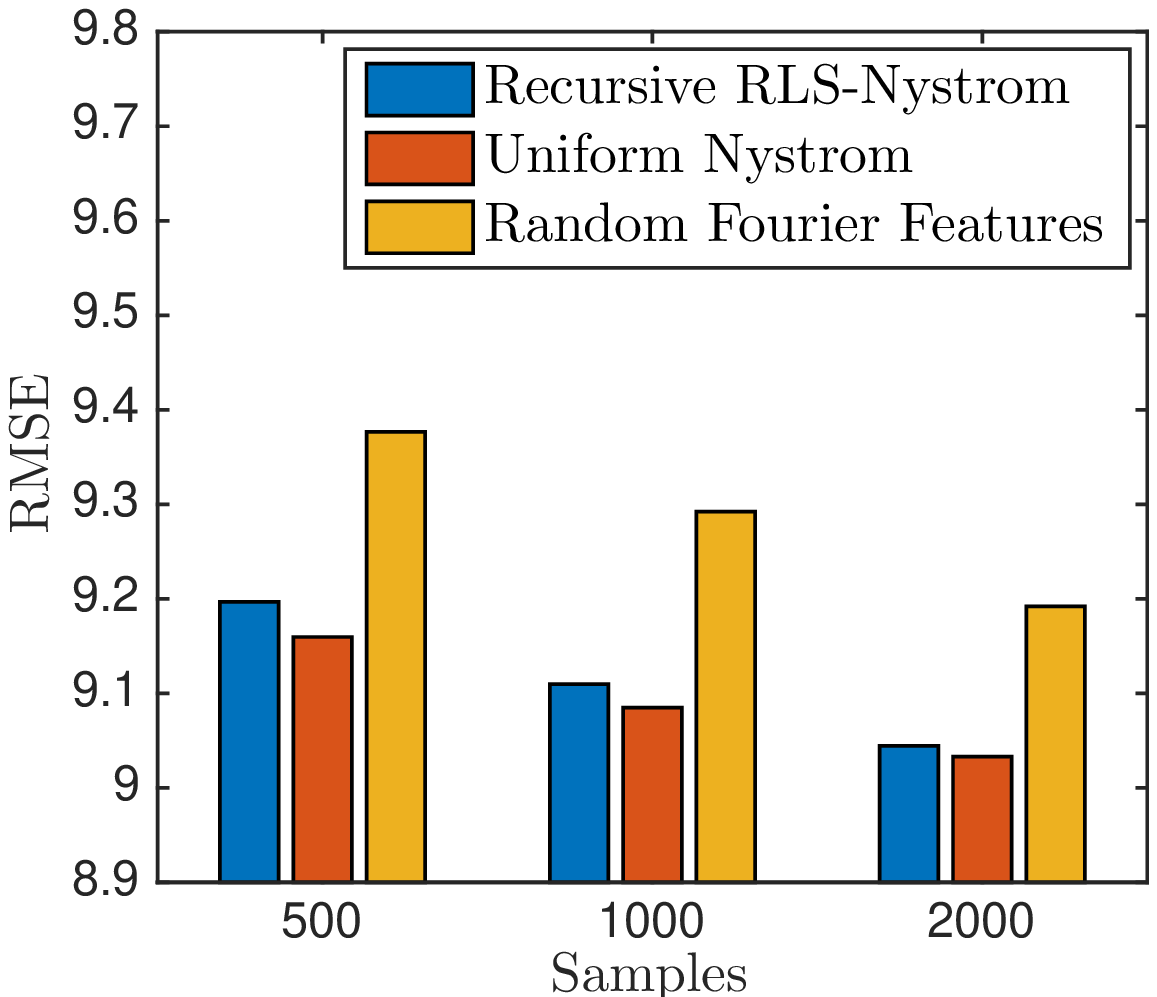} 
\caption{\texttt{YearPredictionMSD} }
\end{subfigure}
\caption{Performance of kernel approximation methods for classification and clustering. For \texttt{Covertype}, classification error is measured in separating Class 2 from the remaining classes. For \texttt{YearPredictionMSD}, RMSE is for the unnormalized output. Regularization and kernel parameters are obtained via cross validation on training data. Test results are averaged over 10 trials with a fixed test set, as all three algorithms are randomized.}
\label{fig:stats}
\end{figure}

We conclude by verifying the usefulness of our kernel approximations in downstream learning tasks. We focus on \texttt{Covertype} and \texttt{YearPredictionMSD}, which each have approximately $n = 500,000$ data points. While full kernel methods do not scale in this regime, Recursive RLS-Nystr{\"o}m does since its runtime depends linearly on $n$. For example, on \texttt{YearPredictionMSD} the method requires $307$ sec. (averaged over $5$ trials) to build a $2,000$ landmark Nystr{\"o}m approximation for $463,716$ training points. Ridge regression using the approximate kernel then requires $208$ sec. for a total of $515$ sec. In comparison, the fastest method, random Fourier features, required $43$ sec. to build a rank $2,000$ kernel approximation and $222$ sec. for regression, for a total time of $265$ sec.

For \texttt{Covertype} we performed classification using the LIBLINEAR support vector machine library. For all sample sizes the SVM dominated runtime cost, so Recursive RLS-Nystr{\"o}m was only marginally slower than uniform Nystr{\"o}m and random Fourier features for a fixed sample size.

In terms of classification rate for \texttt{Covertype} and RMSE error for \texttt{YearPredictionMSD}, as can be seen in Figure \ref{fig:stats}, both Nystr{\"o}m methods outperform random features. However, we do not see much difference between the two Nystr{\"o}m methods. We leave open understanding why the significantly better kernel approximations discussed in Section \ref{sec:kernelPerf} do not necessarily translate to much better learning performance, or whether they would make a larger difference for other problems.
%

\section*{Acknowledgements}
We would like to thank Michael Mahoney for bringing the potential of ridge leverage scores to our attention and suggesting their possible approximation via iterative sampling schemes. We would also like to thank Michael Cohen for pointing out (and fixing) an error in our original manuscript and generally for his close collaboration in our work on leverage score sampling algorithms. Finally, thanks to Haim Avron for pointing our an error in our original analysis. 

\bibliographystyle{alpha}
\bibliography{linear_time_kernels}

\appendix

\section{Ridge leverage score sampling bounds}
Here we give the primary matrix concentration results used to bound the performance of ridge leverage score sampling in Theorems \ref{additiveErrorThm}, \ref{thm:main_algo_theorem}, and \ref{thm:main_algo_k}.
\label{sampling_proofs}
\begin{lemma}\label{bernstein}
For any $\lambda > 0$ and $\delta \in (0,1/8)$, given ridge leverage score approximations $\tilde l_i^{\lambda} \ge \l_i^{\lambda}$ for all $i$, let $p_i = \min \left \{1,16\tilde{l}_i^{\lambda} \log(\sum \tilde{l}_i^{\lambda}/\delta) \right \}$. Let $\bv{S} \in \mathbb{R}^{n \times s}$ be selected by sampling the standard basis vectors $\bv{e}_1,\ldots, \bv{e}_n$ each independently with probability $p_i$ and rescaling selected columns by $1/\sqrt{p_i}$. With probability $1-\delta$, $1/2 \cdot \sum_i p_i \leq s \leq 2\sum_i p_i$ and:
\begin{align}\label{bernsteinPrimaryBound}
\frac{1}{2}\bv{B}^T\bv{B} - \frac{1}{2} \lambda \bv{I} \preceq \bv{B}^T\bv{SS}^T\bv{B} \preceq \frac{3}{2}\bv{B}^T\bv{B} + \frac{1}{2} \lambda \bv{I},
\end{align}
\end{lemma}
\begin{proof}
Let $\bv{B} = \bv{U} \bv{\Sigma} \bv{V}^T$ be the singular value decomposition of $\bv{B}$. By Definition \ref{leverageDef}:
\begin{align*}
l_i^{\lambda}  = \bv{b}_i^T \left(\bv{B}^T\bv{B} + \lambda \bv{I}\right)^{-1} \bv{b}_i
&= \bv{b}_i^T \left (\bv{V} \bv{\Sigma}^2 \bv{V}^T +  \lambda \bv{V} \bv{V}^T \right)^{-1} \bv{b}_i\\
&= \bv{b}_i^T \left (\bv{V} \bv{\bar \Sigma}^2 \bv{V}^T \right )^{-1} \bv{b}_i\\
&= \bv{b}_i^T \left (\bv{V} \bv{\bar \Sigma}^{-2} \bv{V}^T \right ) \bv{b}_i,
\end{align*}
where $\bv{\bar \Sigma}^2_{i,i} = \sigma^2_i(\bv{B}) + \lambda$.
For each $i \in 1,\ldots,n$ define the matrix valued random variable:
\begin{align*}
\bv{X}_i = 
\begin{cases}
\left (\frac{1}{p_i} -1 \right ) \bv{\bar \Sigma}^{-1} \bv{V}^T\bv{b}_i \bv{b}_i^T \bv{V} \bv{\bar \Sigma}^{-1} \text{ with probability } p_i\\
-\bv{\bar \Sigma}^{-1} \bv{V}^T\bv{b}_i \bv{b}_i^T \bv{V} \bv{\bar \Sigma}^{-1} \text{ with probability } (1-p_i)
\end{cases}
\end{align*}

Let $\bv{Y} = \sum_i \bv{X}_i$. We have $\E \bv{Y} = \bv{0}$. Furthermore, $\bv{B}^T \bv{S}\bv{S}^T\bv{B} = \bv{V}\bv{\bar \Sigma} \bv{Y} \bv{\bar \Sigma} \bv{V}^T + \bv{B}^T\bv{B}$. If we can show that $\norm{\bv{Y}}_2 \le \frac{1}{2}$, then since $\bv{V}\bv{\bar \Sigma}^2 \bv{V}^T = \bv{B}^T\bv{B} + \lambda \bv{I}$ this would give the desired bound:
\begin{align*}
\frac{1}{2}\bv{B}^T\bv{B} - \frac{1}{2} \lambda \bv{I} \preceq \bv{B}^T\bv{SS}^T\bv{B} \preceq \frac{3}{2}\bv{B}^T\bv{B} + \frac{1}{2} \lambda \bv{I}.
\end{align*}

To prove that $\norm{\bv{Y}}_2$ is small we use an intrinsic dimension matrix Bernstein inequality. This inequality will bound the deviation of $\bv{Y}$ from its expectation as long as we can bound each $\|\bv{X}_i\|_2$ and we can bound the matrix variance $\E(\bv{Y}^2)$.

\begin{theorem}[Theorem 7.3.1, \cite{tropp2015introduction}]
\label{mbernstein}
Let $\bv{X}_1, \ldots, \bv{X}_n$ be random symmetric matrices such that for all $i$, $\E\bv{X} = \bv{0}$ and $\|\bv{X}_i\|_2 \leq L$. Let $\bv{Y} = \sum_{i=1}^n \bv{X}_i$. As long we can bound the matrix variance:
\begin{align*}
\E(\bv{Y}^2) \preceq \bv{Z},
\end{align*}
then for for $t \geq \sqrt{\|\bv{Z}\|_2} + L/3$,
\begin{align*}
\Pr \left [\norm{\bv Y} \ge t \right ] &\le 4 \frac{\tr(\bv{Z})}{\|\bv{Z}\|_2}e^\frac{-t^2/2}{\|\bv{Z}\|_2 + Lt/3}.
\end{align*}
\end{theorem}

If $p_i = 1$ (i.e. $c \tilde l_i^{ \lambda}\log(\sum l_i^{\tilde \lambda}/\delta) \ge 1$) then $\bv{X}_i = \bv{0}$ so $\norm{\bv{X}_i}_2 = 0$. Otherwise, we use the fact that: 
\begin{align}\label{leverageScoreSpectralBound}\frac{1}{\tilde l_i^{\lambda}} \bv{b}_i \bv{b}_i^T \preceq \frac{1}{l_i^{\lambda}} \bv{b}_i \bv{b}_i^T \preceq \bv{B}^T\bv{B} + \lambda \bv{I}.
\end{align}
This follows because we can write any $\bv{x}$ as $\bv{x} = (\bv{B}^T\bv{B} + \lambda \bv{I})^{-1/2} \bv{y}$ for some $\bv{y}$. We can then write:
\begin{align*}
\bv{x}^T \bv{b}_i \bv{b}_i^T \bv{x} &= \bv{y}^T (\bv{B}^T\bv{B} + \lambda \bv{I})^{-1/2} \bv{b}_i \bv{b}_i^T (\bv{B}^T\bv{B} + \lambda \bv{I})^{-1/2} \bv{y}\\ &\le \norm{\bv{y}}_2^2 \cdot \norm{(\bv{B}^T\bv{B} + \lambda \bv{I})^{-1/2} \bv{b}_i \bv{b}_i^T (\bv{B}^T\bv{B} + \lambda \bv{I})^{-1/2}}_2.
\end{align*}
Since $(\bv{B}^T\bv{B} + \lambda \bv{I})^{-1/2} \bv{b}_i \bv{b}_i^T (\bv{B}^T\bv{B} + \lambda \bv{I})^{-1/2}$ is rank $1$, we have:
\begin{align}
\norm{(\bv{B}^T\bv{B} + \lambda \bv{I})^{-1/2} \bv{b}_i \bv{b}_i^T (\bv{B}^T\bv{B} + \lambda \bv{I})^{-1/2}}_2 &= \tr \left ((\bv{B}^T\bv{B} + \lambda \bv{I})^{-1/2} \bv{b}_i \bv{b}_i^T (\bv{B}^T\bv{B} +  \lambda \bv{I})^{-1/2} \right )\nonumber\\
&= \bv{b}_i^T (\bv{B}^T\bv{B} + \lambda \bv{I})^{-1} \bv{b}_i = l_i^{\lambda}\label{spectralNormMult}
\end{align}
where in the last step we use the cyclic property of the trace. Writing $\bv{y} = (\bv{B}^T\bv{B} + \lambda \bv{I})^{1/2} \bv{x}$ and plugging back into \eqref{spectralNormMult} gives:
\begin{align*}
\bv{x}^T \bv{b}_i \bv{b}_i^T \bv{x} \le \norm{\bv{y}}_2^2 \cdot l_i^{\lambda} = \bv{x}^T (\bv{B}^T\bv{B} + \lambda \bv{I}) \bv{x} \cdot l_i^{\lambda}.
\end{align*}
Rearranging and using that $\tilde{l}_i^\lambda \geq l_i^\lambda$ gives \eqref{leverageScoreSpectralBound}. With this bound in place we get:
\begin{align*}
\frac{1}{\tilde l_i^{\lambda}} \cdot \bv{\bar \Sigma}^{-1} \bv{V}^T\bv{b}_i \bv{b}_i^T \bv{V} \bv{\bar \Sigma}^{-1} \preceq \bv{\bar \Sigma}^{-1} \bv{V}^T\left(\bv{B}^T\bv{B} + \lambda \bv{I}\right) \bv{V} \bv{\bar \Sigma}^{-1} = \bv{I}.
\end{align*}
 So we have:
\begin{align*}
\bv{X}_i \preceq \frac{1}{p_i} \bv{\bar \Sigma}^{-1} \bv{V}^T\bv{b}_i \bv{b}_i^T \bv{V} \bv{\bar \Sigma}^{-1} \preceq \frac{\tilde l_i^{\lambda}}{p_i} \bv{I} = \frac{1}{16 \log \left (\sum l_i^{\tilde\lambda}/\delta \right )} \bv{I} \preceq \frac{1}{16 \log \left (\sum l_i^{\lambda}/\delta \right )} \bv{I}.
\end{align*}
Next we bound the variance of $\bv{Y}$. 
\begin{align}
\label{eq:var_bound2}
\E (\bv{Y}^2) = \sum \E (\bv{X}_i^2 ) &\preceq \sum \left [p_i \left (\frac{1}{p_i}-1\right )^2 + (1-p_i) \right ] \cdot \bv{\bar \Sigma}^{-1} \bv{V}^T\bv{b}_i \bv{b}_i^T \bv{V} \bv{\bar \Sigma}^{-2} \bv{V}^T\bv{b}_i \bv{b}_i^T \bv{V} \bv{\bar \Sigma}^{-1}\nonumber\\
&\preceq \sum \frac{1}{p_i} \cdot  l_i^{\tilde \lambda}\cdot \bv{\bar \Sigma}^{-1} \bv{V}^T\bv{b}_i \bv{b}_i^T \bv{V} \bv{\bar \Sigma}^{-1}
\preceq \frac{1}{16 \log \left (\sum l_i^{\lambda}/\delta \right )} \bv{\bar \Sigma}^{-1} \bv{V}^T\bv{B}^T\bv{B} \bv{V} \bv{\bar \Sigma}^{-1}\nonumber\\
&\preceq  \frac{1}{16 \log \left (\sum l_i^{\lambda}/\delta \right )} \bv \Sigma^2 \bv{\bar \Sigma}^{-2} \preceq  \frac{1}{16 \log \left (\sum l_i^{\lambda}/\delta \right )} \bv D.
\end{align}
where $\bv{D}_{1,1} = 1$ and $\bv{D}_{i,i} = (\bv{\Sigma}^2 \bv{\bar \Sigma}^{-2})_{i,i} = \frac{\sigma_i^2(\bv{B})}{\sigma_i^2(\bv{B}) + \lambda}$ for all $i \ge 2$. Note that $\norm{\bv D}_2 = 1$. 

Then applying Theorem \ref{mbernstein} with $\bv{Z} = \bv{D}/16 \log \left (\sum l_i^{\lambda}/\delta \right )$ we see that:
\begin{align}
\label{eq:almost_chernoff2}
\Pr \left [\norm{\bv Y}_2 \ge \frac{1}{2} \right ] &\le 4\tr(\bv D) e^{\frac{-1/8}{\frac{1}{16\log(\sum l_i^{\lambda}/\delta)}+ \frac{1}{192\log(\sum l_i^{\lambda}/\delta)}}}.
\end{align}
Then we observe that:
\begin{align*}
\tr(\bv{D}) \le 1 + \tr(\bv{\Sigma}^2 \bv{\bar \Sigma}^{-2}) = 1+ \tr \left (\bv{K}(\bv{K}+\lambda \bv{I})^{-1}\right ) = 1+ \sum_i l_i^{\lambda}.
\end{align*}
Plugging into \eqref{eq:almost_chernoff2}, establishes \eqref{bernsteinPrimaryBound}:
\begin{align*}
\Pr \left [\norm{\bv Y} \ge \frac{1}{2} \right ] &\le 4 \left (1 + \sum_i l_i^{\lambda} \right ) \cdot e^{- 2\log(\sum l_i^{\lambda}/\delta)} \le \delta/2.
\end{align*}
Note that here we make the extremely mild assumption that $\sum_i l_i^{\lambda} \ge 1$. If not, we can simply use a smaller $\lambda$ that makes this condition true, and will have $s = O(1)$. 

All that remains to show is that the sample size $s$ is bounded with high probability. If $p_i = 1$, we always sample $i$ so there is no variance in $s$. Let $S \subseteq [1,...,n]$ be the set of indices with $p_i < 1$. The expected number of points sampled from $S$ is $\sum_{i \in S} p_i = 16\log(\sum \tilde{l}_i^{\lambda}/\delta) \sum_{i \in S} \tilde l_i^{\lambda}$. Assume without loss of generality that $\sum_{i \in S} \tilde l_i^{\lambda} \ge 1$ -- otherwise can just increase our leverage score estimates and increase the expected sample size by at most $1$. Then, by a standard Chernoff bound, with probability at least $1-\delta/2$, 
$$\frac{1}{2} \cdot 16\log(\sum \tilde{l}_i^{\lambda}/\delta) \sum_{i \in S} \tilde{l}_i^{ \lambda} \le s \le 2\cdot16\log(\sum \tilde{l}_i^{\lambda}/\delta) \sum_{i \in S} \tilde{l}_i^{ \lambda}.$$
Union bounding over failure probabilities gives the lemma.
\end{proof}

Lemma \ref{bernstein} yields an easy corollary about sampling \emph{without rescaling} the columns in $\bv{S}$:
\begin{corollary}\label{bernstein2}

For any $\lambda > 0$ and $\delta \in (0,1/8)$, given ridge leverage score approximations $\tilde l_i^{\lambda} \ge \l_i^{\lambda}$ for all $i$, let $p_i = \min \left \{16\tilde{l}_i^{\lambda} \log(\sum \tilde{l}_i^{\lambda}/\delta), 1 \right \}$. Let $\bv{S} \in \mathbb{R}^{n \times s}$ be selected by sampling,  \textbf{but not rescaling}, the standard basis vectors $\bv{e}_1,\ldots, \bv{e}_n$ each independently with probability $p_i$. With probability $1-\delta$, $1/2 \cdot \sum_i p_i\leq s \leq 2\sum_i p_i$ and there exists some scaling factor $C >0$ such that
\begin{align}\label{bernsteinPrimaryBound}
\bv{B}^T\bv{B} \preceq C \cdot \bv{B}^T\bv{S}\bv{S}^T\bv{B} + \lambda \bv{I}.
\end{align}
\end{corollary}
\begin{proof}
By Lemma \ref{bernstein}, if we set $C' = \frac{1}{\min_i p_i}$ we have:
\begin{align*}
\frac{1}{2} \bv{B}^T\bv{B} - \frac{1}{2} \lambda \bv{I} &\preceq C' \cdot \bv{B}^T \bv{S}\bv{S}^T \bv{B}\\
\bv{B}^T\bv{B} &\preceq 2 C' \cdot \bv{B}^T \bv{S}\bv{S}^T \bv{B} + \lambda \bv{I}
\end{align*}
which gives the corollary by setting $C = 2C'$.
\end{proof}

\section{Projection-cost preserving kernel approximation}
\label{sec:pcp}
In addition to the basic spectral approximation guarantee of Theorem \ref{additiveErrorThm}, we prove that, with high probability, the RLS-Nystr{\" o}m method presented in Algorithm \ref{alg:highlevel} outputs an approximation $\bv{\tilde K}$ satisfying what is known as a \emph{projection-cost preservation guarantee}. This approximation also immediately holds for the efficient implementation of sampling in Algorithm \ref{halvingK}.

\begin{theorem}[Projection-cost preserving kernel approximation]
\label{pcpTheorem} Let $\lambda = \frac{\epsilon}{k}\sum_{i=k+1}^n \sigma_i(\bv{K})$. For any $\epsilon \in (0,1), \delta \in (0,1/8)$, RLS-Nystr{\" o}m returns an $\bv{S} \in \R^{n\times s}$ such that with probability $1-\delta$, $1/2 \sum_i p_i \leq s \leq 2\sum_i p_i$ and the approximation $\bv{\tilde K} =  \bv{KS} (\bv{S}\bv{K}\bv{S})^+ \bv{S}\bv{K}$ satisfies, for any rank $k$ orthogonal projection $\bv{X}$ and a positive constant $c$ independent of $\bv{X}$:
\begin{align}\label{actualPCPBound}
\tr(\bv{K} - \bv{X}\bv{K}\bv{X})\le \tr(\bv{\tilde K} - \bv{X}\bv{\tilde K}\bv{X}) + c \le (1+\epsilon)\tr(\bv{K} - \bv{X}\bv{K}\bv{X}).
\end{align}
When ridge leverage scores are computed exactly, $\sum_i p_i = O\left(\frac{k}{\epsilon}\log\frac{k}{\delta\epsilon}\right)$.
\end{theorem}
Intuitively, Theorem \ref{pcpTheorem} ensures that the distance from $\bv{\tilde K}$ to any low dimensional subspace closely approximates the distance from $\bv{K}$ to the subspace. Accordingly, $\bv{\tilde K}$ can be used in place of $\bv{K}$ to approximately solve low-rank approximation problems, both constrained (e.g. $k$-means clustering) and unconstrained (e.g. principal component analysis). See Theorems \ref{thm:app_kmeans} and \ref{thm:app_pca}.

\begin{proof}
Set $c = \tr(\bv{K}) - \tr(\bv{\tilde K})$, which is $\geq 0$ since $\bv{\tilde K} \preceq \bv{K}$ by Theorem \ref{additiveErrorThm}. By linearity of trace:
\begin{align*}
\tr(\bv{\tilde K} - \bv{X}\bv{\tilde K}\bv{X}) + c = \tr(\bv{K}) - \tr(\bv{X}\bv{\tilde K}\bv{X}).
\end{align*}
So to obtain \eqref{actualPCPBound} it suffices to show:
\begin{align}\label{lastTraceBound}
\tr(\bv{X}\bv{ K}\bv{X}) - \epsilon \tr(\bv{K} - \bv{X}\bv{K}\bv{X}) \le \tr(\bv{X}\bv{\tilde K}\bv{X}) \le \tr(\bv{X}\bv{ K}\bv{X}).
\end{align}
Since $\bv{X}$ is a rank $k$ orthogonal projection we can write $\bv{X} = \bv{Q} \bv{Q}^T$ where $\bv{Q} \in \mathbb{R}^{n \times k}$ has orthonormal columns. Applying the cyclic property of the trace, and the spectral bound of Theorem \ref{additiveErrorThm}:
\begin{align*}
\tr(\bv{X} \bv{\tilde K} \bv{X}) = \tr(\bv{Q}^T \bv{\tilde K} \bv{Q}) = \sum_{i=1}^k \bv{q}_i^T \bv{\tilde K} \bv{q}_i \le \sum_{i=1}^k \bv{q}_i^T \bv{ K} \bv{q}_i = \tr(\bv{Q}^T \bv{ K} \bv{Q}) = \tr(\bv{X} \bv{ K} \bv{X}).
\end{align*}
This gives us the upper bound of \eqref{lastTraceBound}. For the lower bound we apply Corollary \ref{main_cor}:
\begin{align}\label{almostTraceBound}
\tr(\bv{X} \bv{\tilde K} \bv{X}) = \sum_{i=1}^k \bv{q}_i^T \bv{\tilde K} \bv{q}_i \ge \sum_{i=1}^k \bv{q}_i^T \bv{ K} \bv{q}_i - k \epsilon \lambda = \tr(\bv{X} \bv{ K} \bv{X}) - k\epsilon \lambda.
\end{align}
Finally, $k\epsilon \lambda = \epsilon \sum_{i=k+1}^n \sigma_i(\bv{K}) \le \epsilon \tr(\bv{K}-\bv{X}\bv{K}\bv{X})$ since $\tr(\bv{K}) = \sum_{i=1}^n \sigma_i(\bv{K})$ and  $\tr(\bv{X}\bv{K}\bv{X}) \le \sum_{i=1}^k \sigma_i(\bv{K})$ by the Eckart-Young theorem. Plugging into \eqref{almostTraceBound} gives \eqref{lastTraceBound}, completing the proof.


We conclude by showing that $s$ is not too large. As in the proof of Theorem \ref{additiveErrorThm}, $s \leq 2\sum_i p_i$ with probability $1-\delta$. When ridge leverage scores are computed exactly $\sum_i p_i \leq  16\sum l_i^{\lambda}\log(\sum l_i^{\lambda}/\delta)$.
\begin{align}
\sum_i l_i^{\lambda} &= \tr(\bv{K}(\bv{K} + \epsilon \left(\frac{1}{k}\sum_{i=k+1}^n \sigma_i(\bv{K}\right) \bv{I})^{-1}) \nonumber\\
&\le \frac{1}{\epsilon} \tr(\bv{K}(\bv{K} + \left(\frac{1}{k}\sum_{i=k+1}^n \sigma_i(\bv{K}\right)  \bv{I})^{-1})\nonumber\\
&= \frac{1}{\epsilon} \sum_{i=1}^n \frac{\sigma_i(\bv{K})}{\sigma_i(\bv{K}) + \frac{1}{k} \sum_{i=k+1}^n \sigma_i(\bv{K})}\nonumber\\
&= \frac{1}{\epsilon}\left( \sum_{i=1}^k \frac{\sigma_i(\bv{K})}{\sigma_i(\bv{K}) + \frac{1}{k} \sum_{i=k+1}^n \sigma_i(\bv{K})} + \sum_{i=k+1}^n \frac{\sigma_i(\bv{K})}{\sigma_i(\bv{K}) + \frac{1}{k} \sum_{i=k+1}^n \sigma_i(\bv{K})} \right )\nonumber\\
&\le \frac{1}{\epsilon} \left (k + \sum_{i=k+1}^n \frac{\sigma_i(\bv{K})}{\frac{1}{k}\sum_{i=k+1}^n \sigma_i(\bv{K})}\right ) = \frac{2k}{\epsilon}. \label{sum_to_k}
\end{align}
Accordingly, $\sum_i p_i =  32\frac{k}{\epsilon}\log\frac{k}{\delta\epsilon}$ as desired.
\end{proof}

\section{Applications to learning tasks}
\label{sec:apps}
In this section use our general approximation gaurantees from Theorems \ref{additiveErrorThm} and \ref{pcpTheorem} to prove that the kernel approximations given by RLS-Nystr{\"o}m sampling are sufficient for many downstream learning tasks. In other words, $\bv{\tilde K}$ can be used in place of $\bv{K}$ without sacrificing accuracy or statistical performance in the final computation.

\subsection{Kernel ridge regression}
\label{app:ridge_regression}
We begin with a standard formulation of the ubiquitous kernel ridge regression task \cite{scholkopf2002learning}. Given input data points $\bv{x}_1, \ldots, \bv{x}_n \in \R^d$ and labels $y_1, \ldots, y_n \in \R$ this problem asks us to solve:
\begin{align}
\label{eq:krr_min}
\bs{\alpha} \eqdef \argmin_{\bv{c} \in \R^n} \|\bv{K}\bv{c} - \bv{y}\|_2^2 + \lambda \bv{c}^T\bv{K} \bv{c},
\end{align}
which can be done in closed form by computing:
\begin{align*}
\bs{\alpha} = (\bv{K} + \lambda \bv{I})^{-1} \bv{y}.
\end{align*}
For prediction, when we're given a new input $\bv{x}$, we evaluate its label to be:
\begin{align}
\label{krr_predictor}
y = \sum_{i=1}^n \bv{\alpha}_i K(\bv{x}_i,\bv{x}).
\end{align}

\subsubsection{Approximate kernel ridge regression}
Naively, solving for $\bs{\alpha}$ exactly requires at least $O(n^2)$ time to compute $\bv{K}$, plus the cost of a direct or iterative matrix inversion algorithm. Prediction is also costly since it requires a kernel evaluation with all $n$ training points. These costs can be reduced significantly using Nystr{\"o}m approximation.

In particular, we first select landmark points and compute the kernel approximation $\bv{\tilde K} = \bv{K S} (\bv{S}^T\bv{KS})^+ \bv{S}^T\bv{ K}$. We can then compute an approximate set of coefficients:
\begin{align}
\bs{\tilde \alpha} \eqdef (\bv{\tilde K} + \lambda \bv{I})^{-1} \bv{y}.
\end{align}
With a direct matrix inversion, doing so only takes $O(ns^2)$ time when our sampling matrix $\bv{S}\in\R^{n\times s}$ selects $s$ landmark points. This is a significant improvement on the $O(n^3)$ time required to invert the full kernel. Additionally, the cost of multiplying by $\bv{\tilde K} + \lambda \bv{I}$, which determines the cost of most iterative regression solvers, is reduced, from $O(n^2)$ to $O(ns)$.

To predict a label for a new $\bv{x}$, we first compute its kernel product with all of our landmark points. Specifically, let $\bv{x}^{(1)}, \ldots, \bv{x}^{(s)}$ be the landmarks selected by $\bv{S}$'s columns. Define $\bv{w} \in \R^s$ as:
\begin{align*}
\bv{w}_i \eqdef K(\bv{x}^{(i)},\bv{x}).
\end{align*}
and let

\begin{align}
\label{approx_krr_predictor}
y = \bv{w}^T(\bv{S}^T\bv{KS})^+\bv{S}^T\bv{K}\bs{\tilde \alpha}.
\end{align}
Computationally, it makes sense to precompute $(\bv{S}^T\bv{KS})^+\bv{S}^T\bv{K}\bs{\tilde \alpha}$. Then the cost of prediction is just $s$ kernel evaluations to compute $\bv{w}$, plus $s$ additional operations to multiply $\bv{w}^T$ by $(\bv{S}^T\bv{KS})^+\bv{S}^T\bv{K}\bs{\tilde \alpha}$.

This approach is the standard way of applying Nystr{\" o}m approximation to the ridge regression problem and there are a number of ways to evaluate its performance. Beyond directly bounding minimization error for \eqref{eq:krr_min} (see e.g. \cite{relativeErrorRidge,yang2015randomized,divideandconquer}), one particularly natural approach is to consider how the statistical risk of the estimator output by our approximate ridge regression routine compares to that of the exactly computed estimator. 

\subsubsection{Relative error bound on statistical risk}
To evaluate statistical risk we consider a \emph{fixed design} setting which has been especially -popular \cite{bach2012sharp,alaoui2015fast, chengtaodpp,paul015column}. Note that more complex statistical models can be analyzed as well \cite{Hsu2014,rudi2015less}. In this setting, we assume that our observed labels $\bv{y} = [y_1, \ldots, y_n]$ represent underlying true labels $\bv{z} = [z_1, \ldots, z_n]$ perturbed with noise. For simplicity, we assume uniform Gaussian noise with variance $\sigma^2$, but more general noise models can be handled with essentially the same proof \cite{bach2012sharp}. In particular, our modeling assumption is that:
\begin{align*}
y_i = z_i + \eta_i
\end{align*}
where $\eta_i \sim N(0,\sigma^2)$.

Following \cite{bach2012sharp} and \cite{alaoui2015fast}, we want to bound the expected in sample risk of our estimator for $\bv{z}$, which is computed using the noisy measurements $\bv{y} = \bv{z} + \bs{\eta}$. For exact kernel ridge regression, we can check from \eqref{krr_predictor} that this estimator is equal to $\bv{K}\bs{\alpha}$. The risk $\mathcal{R}$ is:
\begin{align*}
\mathcal{R} &\eqdef \E_{\bs{\eta}} \|\bv{K}(\bv{K} + \lambda \bv{I})^{-1} (\bv{z} + \bs{\eta}) - \bv{z}\|_2^2
\\&=
\|\left(\bv{K}(\bv{K} + \lambda \bv{I})^{-1} - \bv{I}\right) \bv{z}\|_2^2 + \E_{\bs{\eta}} \|\bv{K}(\bv{K} + \lambda \bv{I})^{-1}\bs{\eta}\|_2^2
\\&=
\lambda^2 \bv{z}^T(\bv{K} + \lambda \bv{I})^{-2}\bv{z} + \sigma^2 \tr(\bv{K}^2(\bv{K} + \lambda\bv{I})^{-2}).
\end{align*}
The two terms that compose $\mathcal{R}$ are referred to as the bias and variance terms of the risk:
\begin{align*}
\text{bias}(\bv{K})^2 &\eqdef \lambda^2 \bv{z}^T(\bv{K} + \lambda \bv{I})^{-2}\bv{z} \\
\text{variance}(\bv{K}) &\eqdef \sigma^2 \tr(\bv{K}^2(\bv{K} + \lambda\bv{I})^{-2}).
\end{align*}

For approximate kernel ridge regression, it follows from \eqref{approx_krr_predictor} that our predictor for $\bv{z}$ is $\bv{\tilde K}\bs{\tilde \alpha}$. Accordingly, the risk of the approximate estimator, $\mathcal{\tilde R}$ is equal to:
\begin{align*}
\mathcal{\tilde R} = \text{bias}(\bv{\tilde K})^2 + \text{variance}(\bv{\tilde K})
\end{align*}
We're are ready to prove our main theorem on kernel ridge regression.

\begin{theorem}[Kernel Ridge Regression Risk Bound]
\label{thm:app_ridge}
Suppose $\bv{\tilde K}$ is computed using RLS-Nystr\"{o}m with approximation parameter $\epsilon\lambda$ and failure probability $\delta \in (0,1/8)$. Let $\bs{\tilde \alpha} = (\bv{\tilde K} + \lambda \bv{I})^{-1} \bv{y}$ and let $\bv{\tilde K}\bs{\tilde \alpha}$ be our estimator for $\bv{z}$ computed with the approximate kernel. With probability $1-\delta$:
\begin{align*}
\mathcal{\tilde R} \leq (1+3\epsilon) \mathcal{R}.
\end{align*}
By Theorem \ref{thm:main_algo_theorem}, Algorithm \ref{halvingFixed} can compute $\bv{\tilde K}$ with just $O(n s)$ kernel evaluations and $O(n s^2 )$ computation time, with $s = O\left(\frac{d_\text{eff}^\lambda}{\epsilon}\log\frac{d_{\text{eff}}^\lambda}{\delta\epsilon}\right)$.
\end{theorem}

In other words, replacing $\bv{K}$ with the approximation $\bv{\tilde K}$ is provably sufficient for obtaining a $1+\Theta(\epsilon)$ quality solution to the downstream task of ridge regression.

\begin{proof}
The proof follows that of Theorem 1 in \cite{alaoui2015fast}. First we show that:
\begin{align}
\label{eq:bias_bound}
 \text{bias}(\bv{\tilde K}) \leq (1+\epsilon) \text{bias}(\bv{K}).
\end{align}
At first glance this might appear trivial as Theorem \ref{additiveErrorThm} easily implies that 
\begin{align*}
(\bv{\tilde K} + \lambda\bv{I})^{-1} \preceq (1+\epsilon)(\bv{K} + \lambda\bv{I})^{-1}
\end{align*}
However, this statement \emph{does not imply} that
\begin{align*}
(\bv{\tilde K} + \lambda\bv{I})^{-2} \preceq (1+\epsilon)^2(\bv{K} + \lambda\bv{I})^{-2}
\end{align*}
since $(\bv{\tilde K} + \lambda\bv{I})^{-1}$ and $(\bv{K} + \lambda\bv{I})^{-1}$ do not necessarily commute. Instead we proceed: 
\begin{align}
\frac{1}{\lambda}\text{bias}(\bv{\tilde K}) &= \|(\bv{\tilde K} + \lambda \bv{I})^{-1}\bv{z}\|_2 \nonumber
\\&\leq
\|(\bv{K} + \lambda\bv{I})^{-1}\bv{z}\|_2 + \|(\bv{\tilde K} + \lambda \bv{I})^{-1}\bv{z} - (\bv{K} + \lambda \bv{I})^{-1}\bv{z}\|_2 &\text{(triangle inequality)} \nonumber
\\&= 
\|(\bv{K} + \lambda\bv{I})^{-1}\bv{z}\|_2 + \|(\bv{\tilde K} + \lambda \bv{I})^{-1} [(\bv{K}+\lambda \bv{I}) - (\bv{\tilde K} + \lambda \bv{I})](\bv{K} + \lambda \bv{I})^{-1}\bv{z}\|_2 \nonumber
\\&=
\|(\bv{K} + \lambda\bv{I})^{-1}\bv{z}\|_2 + \|(\bv{\tilde K} + \lambda \bv{I})^{-1} (\bv{K} - \bv{\tilde K})(\bv{K} + \lambda \bv{I})^{-1}\bv{z}\|_2 \nonumber
\\&\leq
\|(\bv{K} + \lambda\bv{I})^{-1}\bv{z}\|_2 + \|(\bv{\tilde K} + \lambda \bv{I})^{-1} (\bv{K} - \bv{\tilde K})\|_2\|(\bv{K} + \lambda \bv{I})^{-1}\bv{z}\|_2 &\text{(submultiplicativity)} \nonumber
\\&=
\frac{1}{\lambda}\text{bias}(\bv{K})\left(1 + \|(\bv{\tilde K} + \lambda \bv{I})^{-1} (\bv{K} - \bv{\tilde K})\|_2\right) \label{almost_bias_bound}.
\end{align}
So we just need to bound $\|(\bv{\tilde K} + \lambda \bv{I})^{-1} (\bv{K} - \bv{\tilde K})\|_2 \leq \epsilon$.
First note that, by Theorem \ref{additiveErrorThm}, Corollary \ref{main_cor},
\begin{align*}
\bv{K} - \bv{\tilde K} \preceq \bv \epsilon \lambda \bv{I}
\end{align*}
and since $(\bv{K} - \bv{\tilde K})$ and $\bv{I}$ \emph{commute}, it follows that
\begin{align}
\label{eq:sqr_bound}
(\bv{K} - \bv{\tilde K})^2 \preceq \bv \epsilon^2 \lambda^2 \bv{I}.
\end{align}
Accordingly,
\begin{align*}
\|(\bv{\tilde K} + \lambda \bv{I})^{-1} (\bv{K} - \bv{\tilde K})\|^2_2 &= \|(\bv{\tilde K} + \lambda \bv{I})^{-1} (\bv{K} - \bv{\tilde K})^2(\bv{\tilde K} + \lambda \bv{I})^{-1}\|_2 
\\&\leq 
\epsilon^2\lambda^2\|(\bv{\tilde K} + \lambda \bv{I})^{-2}\|_2 
\\&\leq \epsilon^2\lambda^2\frac{1}{\lambda^2} = \epsilon^2.
\end{align*}
So $\|(\bv{\tilde K} + \lambda \bv{I})^{-1} (\bv{K} - \bv{\tilde K})\|_2 \leq \epsilon$ as desired and plugging into \eqref{almost_bias_bound} we have shown \eqref{eq:bias_bound}, that $\text{bias}(\bv{\tilde K}) \leq (1+\epsilon) \text{bias}(\bv{K})$.
We next show that:
\begin{align}
\label{var_bound}
 \text{variance}(\bv{\tilde K}) \leq  \text{variance}(\bv{K}),
\end{align}
where $\text{variance}(\bv{K}) = \sigma^2 \tr(\bv{K}^2(\bv{K} + \lambda \bv{I})^{-2}) =  \sigma^2 \sum_{i=1}^n \left(\frac{\sigma_i(\bv{K})}{\sigma_i(\bv{K}) + \lambda}\right)^2$. Since $\bv{\tilde K} \preceq \bv{K}$ by Theorem \ref{additiveErrorThm}, $\sigma_i(\bv{\tilde K}) \leq \sigma_i(\bv{K})$ for all $i$. It follows that, for every $i$,
\begin{align*}
\frac{\sigma_i(\bv{\tilde K})}{\sigma_i(\bv{\tilde K}) + \lambda} \leq \frac{\sigma_i(\bv{K})}{\sigma_i(\bv{K}) + \lambda}.
\end{align*}
This in turn implies that 
\begin{align*}
\sum_{i=1}^n \left(\frac{\sigma_i(\bv{\tilde K})}{\sigma_i(\bv{\tilde K}) + \lambda}\right)^2  \leq \sum_{i=1}^n \left(\frac{\sigma_i(\bv{K})}{\sigma_i(\bv{K}) + \lambda}\right)^2,
\end{align*}
which gives \eqref{var_bound}.
Combining \eqref{var_bound} and \eqref{eq:bias_bound} we conclude that, for $\epsilon < 1$,
\begin{align*}
\mathcal{R}(\hat{f}_\bv{\tilde K}) \leq (1+\epsilon)^2 \mathcal{R}(\hat{f}_\bv{K}) \leq (1+3\epsilon)\mathcal{R}(\hat{f}_\bv{K}).
\end{align*}
\end{proof}

\subsection{Kernel $k$-means}
\label{sec:kmeans}
Kernel $k$-means clustering asks us to partition $\bv{x}_1,\ldots,\bv{x}_n$, into $k$ cluster sets, $\{C_1, \ldots, C_k\}$. Let $\bs{\mu}_i = \frac{1}{|C_i|} \sum_{\bv{x}_j \in C_i} \phi(\bv{x}_j)$ be the centroid of the vectors in $C_{i}$ after mapping to kernel space. 
The goal is to choose $\{C_1, \ldots, C_k\}$ which minimize the objective:
\begin{align}
\sum_{i=1}^k \sum_{\bv{x}_{j}\in C_i}\|\phi(\bv{x}_{j}) - \bs{\mu}_i\|_\mathcal{F}^2 
\end{align}
It is well known that this optimization problem can be rewritten as a \emph{constrained} low-rank approximation problem (see e.g. \cite{BoutsidisMD09} or \cite{cohen2015dimensionality}). In particular, for any clustering $C = \{C_1, \ldots, C_k\}$ we can define a rank $k$ orthonormal matrix $\bv{C} \in \mathbb{R}^{n \times k}$ called the cluster indicator matrix for $C$. $\bv{C}_{i,j} = 1/\sqrt{|C_j|}$ if $\bv{x}_i$ is assigned to $C_j$ and $\bv{C}_{i,j} =0$ otherwise. $\bv{C}^T\bv{C} = \bv{I}$, so $\bv{C}\bv{C}^T$ is a rank $k$ projection matrix.
Furthermore, it's not hard to check that:
\begin{align}
\label{eq:kmeans_trace}
\sum_{i=1}^k \sum_{\bv{x}_{j}\in C_i}\|\phi(\bv{x}_{j}) - \bs{\mu}_i\|_\mathcal{F}^2 = \tr\left(\bv{K} - \bv{C}\bv{C}^T\bv{K}\bv{C}\bv{C}^T\right).
\end{align}
Informally, if we work with the kernalized data matrix $\bv{\Phi}$, \eqref{eq:kmeans_trace} is equivalent to
\begin{align*}
\|\bv{\Phi} - \bv{C}\bv{C}^T\bv{\Phi}\|_F^2.
\end{align*}
Regardless, it's clear that solving kernel $k$-means is equivalent to solving:
\begin{align}
\min_{\bv{C} \in \mathcal{S}} \tr\left(\bv{K} - \bv{C}\bv{C}^T\bv{K}\bv{C}\bv{C}^T\right)
\end{align}
where $\mathcal{S}$ is the set of all rank $k$ cluster indicator matrices. From this formulation, we easily obtain:
\begin{theorem}[Kernel $k$-means Approximation Bound]
\label{thm:app_kmeans}
Let $\bv{\tilde K}$ be computed by RLS-Nystr\"{o}m with $\lambda  = \frac{\epsilon}{k}\sum_{i=k+1}^n \sigma_i(\bv{K})$ and $\delta \in (0,1/8)$.
Let $\bv{\tilde C}^*$ be the optimal cluster indicator matrix for $\bv{\tilde K}$ and let $\bv{\tilde C}$ be an approximately optimal cluster indicator matrix satisfying:
\begin{align*}
\tr\left(\bv{\tilde K} - \bv{\tilde C}\bv{\tilde C}^T\bv{\tilde K}\bv{\tilde C}\bv{\tilde C}^T\right) \leq (1+\gamma)\tr\left(\bv{\tilde K} - \bv{\tilde C}^*\bv{\tilde C}^{*T}\bv{\tilde K}\bv{\tilde C}^*\bv{\tilde C}^{*T}\right).
\end{align*}
Then, if $\bv{C}^*$ is the optimal cluster indicator matrix for $\bv{K}$:
\begin{align*}
\tr\left(\bv{K} - \bv{\tilde C}\bv{\tilde C}^T\bv{K}\bv{\tilde C}\bv{\tilde C}^T\right) \leq (1+\gamma)(1+\epsilon)\tr\left(\bv{K} - \bv{C}^*\bv{C}^{*T}\bv{K}\bv{C}^*\bv{C}^{*T}\right)
\end{align*}
By Theorem \ref{thm:main_algo_k}, Algorithm \ref{halvingK} can compute $\bv{\tilde K}$ with $O(n s)$ kernel evaluations and $O(n s^2 )$ computation time, with $s = O\left(\frac{k}{\epsilon}\log\frac{k}{\delta\epsilon}\right)$.
\end{theorem}

In other words, if we find an optimal set of clusters for our approximate kernel matrix, those clusters will provide a $(1+\epsilon)$ approximation to the original kernel $k$-means problem. Furthermore, if we only solve the kernel $k$-means problem approximately on $\bv{\tilde K}$, i.e. with some approximation factor $(1+\gamma)$, we will do nearly as well on the original problem. This flexibility allows for the use of $k$-means approximation algorithms (since the problem is NP-hard to solve exactly).
\begin{proof} The proof is almost immediate from our bounds on RLS-Nystr\"{o}m:
\begin{align*}
\tr\left(\bv{K} - \bv{\tilde C}\bv{\tilde C}^T\bv{K}\bv{\tilde C}\bv{\tilde C}^T\right) &\leq \tr\left(\bv{\tilde K} - \bv{\tilde C}\bv{\tilde C}^T\bv{\tilde K}\bv{\tilde C}\bv{\tilde C}^T\right) + c&\text{(Theorem \ref{pcpTheorem}})
\\&\leq 
(1+\gamma)\tr\left(\bv{\tilde K} - \bv{\tilde C}^*\bv{\tilde C}^{*T}\bv{\tilde K}\bv{\tilde C}^*\bv{\tilde C}^{*T}\right) + (1+\gamma)c  &\text{(by assumption)}
\\&\leq 
(1+\gamma)\tr\left(\bv{\tilde K} - \bv{C}^*\bv{C}^{*T}\bv{\tilde K}\bv{C}^*\bv{C}^{*T}\right) +  (1+\gamma)c&\text{(optimality of $\bv{\tilde C}^*$ )}
\\&\leq 
(1+\gamma)\tr\left(\bv{\tilde K} - \bv{C}^*\bv{C}^{*T}\bv{\tilde K}\bv{C}^*\bv{C}^{*T}\right) + c&\text{(since $c \geq 0$)}
\\&\leq 
(1+\gamma)(1+\epsilon)\tr\left(\bv{K} - \bv{\tilde C}^*\bv{C}^{*T}\bv{K}\bv{C}^*\bv{C}^{*T}\right). &\text{(Theorem \ref{pcpTheorem})}
\end{align*}
\end{proof}

\subsection{Kernel principal component analysis}
We consider the standard formulation of kernel principal component analysis (PCA) presented in \cite{Scholkopf:1999}.  The goal is to find principal components \emph{in the kernel space} $\mathcal{F}$ that capture as much variance in the kernelized data as possible. In particular, if we work informally with the kernelized data matrix $\bs{\Phi}$, we want to find a matrix $\bv{Z}_k$ containing $k$ orthonormal columns such that:
\begin{align*}
\bs{\Phi}\bs{\Phi}^T - (\bs{\Phi}\bv{Z}_k\bv{Z}_k^T)(\bs{\Phi}\bv{Z}_k\bv{Z}_k^T)^T
\end{align*} 
is as small as possible. In other words, if we project $\bs{\Phi}$'s rows to the $k$ dimensional subspace spanned by $\bv{V}_k$'s columns and then recompute our kernel, we want the approximate kernel to be close to the original. 

We focus in particular on minimizing PCA error according to the metric:
\begin{align}
\label{kernelpcaproblem}
\tr\left(\bs{\Phi}\bs{\Phi}^T - (\bs{\Phi}\bv{Z}_k\bv{Z}_k^T)(\bs{\Phi}\bv{Z}_k\bv{Z}_k^T)^T\right) = \|\bs{\Phi} - \bs{\Phi}\bv{Z}_k\bv{Z}_k^T\|_F^2,
\end{align}
which is standard in the literature \cite{TCS060, avron2014subspace}. As with $f$ in kernel ridge regression, to solve this problem we cannot write down $\bv{Z}_k$ explicitly for most kernel functions. However, the optimal $\bv{Z}_k$ always lies in the column span of $\bs{\Phi}^T$, so we can implicitly represent it by constructing a matrix $\bv{X} \in \R^{n\times k}$ such that $\bs{\Phi}^T\bv{X} = \bv{Z}_k$. It is then easy to compute the projection of any new data vector onto the span of $\bv{Z}_k$ (the typical objective of principal component analysis) since we can multiply by $\bs{\Phi}^T\bv{X}$ using the kernel function.

By the Eckart-Young theorem the optimal $\bv{Z}_k$ contains the top $k$ row principal components of $\bs{\Phi}$. Accordingly, if we write the singular value decomposition $\bs{\Phi} = \bv{U}\bs{\Sigma}\bv{V}^T$ we want to set $\bv{X} = \bv{U}_k\bs{\Sigma}_k^{-1}$, which can be computed from the SVD of $\bv{K} = \bv{U}\bs{\Sigma}^2\bv{U}^T$. $\bv{Z}_k$ will equal $\bv{V}_k$ and \eqref{kernelpcaproblem} reduces to:
\begin{align}
\tr(\bv{K} - \bs{\Phi}\bv{V}_k\bv{V}_k^T\bs{\Phi}) &= \tr(\bv{K} - \bv{V}_k\bv{V}_k^T\bv{K}) & \text{(cyclic property)}\nonumber\\ \label{kpcaopt}
&= \sum_{i=k+1}^n \sigma_i(\bv{K}).
\end{align}

\begin{theorem}[Kernel PCA Approximation Bound]
\label{thm:app_pca}
Let $\bv{\tilde K}$ be computed by RLS-Nystr\"{o}m with $\lambda  = \frac{\epsilon}{k}\sum_{i=k+1}^n \sigma_i(\bv{K})$ and $\delta \in (0,1/8)$. From $\bv{\tilde K}$ we can compute a matrix $\bv{X}\in \R^{s \times k}$ such that if we set $\bv{Z} = \bs{\Phi}^T\bv{S}\bv{X}$, with probability $1-\delta$:
\begin{align*}
\|\bs{\Phi} - \bs{\Phi}\bv{Z}\bv{Z}^T\|_F^2 \leq (1+2\epsilon) \|\bs{\Phi} - \bs{\Phi}\bv{V}_k\bv{V}_k^T\|_F^2 = (1+2\epsilon)\sum_{i=k+1}^n \sigma_i(\bv{K}) .
\end{align*}
By Theorem \ref{thm:main_algo_k}, Algorithm \ref{halvingK} can compute $\bv{\tilde K}$ with $O(n s)$ kernel evaluations and $O(n s^2 )$ computation time, with $s = O\left(\frac{k}{\epsilon}\log\frac{k}{\delta\epsilon}\right)$.
\end{theorem}
Note that $\bv{S}$ is the sampling matrix used to construct $\bv{\tilde K}$. $\bv{Z} = \bs{\Phi}^T\bv{S}\bv{X}$ can be applied to vectors (in order to project onto the approximate low-rank subspace) using only $s$ kernel evaluations.
\begin{proof}
Re-parameterizing $\bv{Z}_k = \bs{\Phi}^T\bv{Y}$, we see that minimizing \eqref{kernelpcaproblem} is equivalent to minimizing
\begin{align*}
\tr(\bv{K} - \bv{K}\bv{Y}\bv{Y}^T\bv{K})
\end{align*}
over $\bv{Y} \in \R^{n\times k}$ such that $(\bs{\Phi}^T\bv{Y})^T\bs{\Phi}^T\bv{Y} = \bv{Y}^T\bv{K}\bv{Y}  = \bv{I}$. Then we re-parameterize again by writing $\bv{Y} = \bv{K}^{-1/2}\bv{W}$ where $\bv{W}$ is an $n\times k$ matrix with orthonormal columns. Using linearity and cyclic property of the trace, we can write:
\begin{align*}
\tr(\bv{K} - \bv{K}\bv{Y}\bv{Y}^T\bv{K}) = \tr(\bv{K}) - \tr(\bv{Y}^T\bv{K}\bv{K}\bv{Y})
= \tr(\bv{K}) - \tr(\bv{W}^T\bv{K}\bv{W}) 
= \tr(\bv{K}) - \tr(\bv{W}\bv{W}^T\bv{K}\bv{W}\bv{W}^T).
\end{align*}
So, we have reduced our problem to a low-rank approximation problem that looks exactly like the $k$-means problem from Section \ref{sec:kmeans}, except without constraints.

Accordingly, following the same argument as Theorem \ref{thm:app_kmeans}, if we find $\bv{\tilde W}$ minimizing:
\begin{align*}
\tr(\bv{\tilde K}) - \tr(\bv{\tilde W}\bv{\tilde W}^T\bv{\tilde K}\bv{\tilde W}\bv{\tilde W}^T),
\end{align*}
then:
\begin{align*}
\tr(\bv{K}) - \tr(\bv{\tilde W}\bv{\tilde W}^T\bv{K}\bv{\tilde W}\bv{\tilde W}^T) \leq (1+\epsilon) \left[\min_{\bv{W}}\tr(\bv{K}) - \tr(\bv{W}\bv{W}^T\bv{K}\bv{W}\bv{W}^T)\right] = (1+\epsilon) \sum_{i=k+1}^n \sigma_i(\bv{K}) .
\end{align*}
$\bv{\tilde W}$ can be taken to equal the top $k$ eigenvectors of $\bv{\tilde K}$, which can be found in $O(n\cdot s^2)$ time. 

However, we are not quite done. Thanks to our re-parameterization this bound guarantees that $\bs{\Phi}^T\bv{K}^{-1/2}\bv{\tilde W}$ is a good set of approximate kernel principal components for $\bv{\Phi}$. Unfortunately, $\bs{\Phi}^T\bv{K}^{-1/2}\bv{\tilde W}$ cannot be represented efficiently (it requires computing $\bv{K}^{-1/2}$) and projecting new vectors to $\bs{\Phi}^T\bv{K}^{-1/2}\bv{\tilde W}$ would require $n$ kernel evaluations to multiply by $\bs{\Phi}^T$.

Instead, recalling the definition of $\bv{P}_\bv{S} = \bv{\Phi}^T\bv{S} (\bv{S}^T\bv{K}^T \bv{S})^+ \bv{S}^T\bv{\Phi}$ from Section \ref{nystromPrelim}, we suggest using the approximate principal components:
\begin{align*}
\bv{P}_\bv{S}\bs{\Phi}^T\bv{\tilde K}^{-1/2}\bv{\tilde W}.
\end{align*}
Clearly $\bv{P}_\bv{S}\bs{\Phi}^T\bv{\tilde K}^{-1/2}\bv{\tilde W}$ is orthonormal because: 
\begin{align*}
(\bv{P}_\bv{S}\bs{\Phi}^T\bv{\tilde K}^{-1/2}\bv{\tilde W})^T\bv{P}_\bv{S}\bs{\Phi}^T\bv{\tilde K}^{-1/2}\bv{\tilde W} &= \bv{\tilde W}^T\bv{\tilde K}^{-1/2}\bs{\Phi}^T\bv{P}_\bv{S}\bs{\Phi}\bv{\tilde K}^{-1/2}\bv{\tilde W} \\
&= \bv{\tilde W}^T\bv{I}\bv{\tilde W} = \bv{I}.
\end{align*}
We will argue that it is offers nearly as a good of a solution as $\bs{\Phi}^T\bv{K}^{-1/2}\bv{\tilde W}$. Specifically, substituting into \eqref{kernelpcaproblem} gives a value of:
\begin{align*}
\tr(\bv{K} - \bs{\Phi}\bv{P}_\bv{S}\bs{\Phi}^T\bv{\tilde K}^{-1/2}\bv{\tilde W}\bv{\tilde W}^T\bv{\tilde K}^{-1/2}\bs{\Phi}\bv{P}_\bv{S}\bs{\Phi}^T) &= \tr(\bv{K}) - \tr(\bv{\tilde W}\bv{\tilde W}^T\bv{\tilde K}^{-1/2}\bs{\Phi}\bv{P}_\bv{S}\bs{\Phi}^T\bs{\Phi}\bv{P}_\bv{S}\bs{\Phi}^T\bv{\tilde K}^{-1/2}) \\
&= \tr(\bv{K}) - \tr(\bv{\tilde W}\bv{\tilde W}^T\bv{\tilde K}^{-1/2}\bv{\tilde K}^2\bv{\tilde K}^{-1/2}) \\
&= \tr(\bv{K}) - \tr(\bv{\tilde W}\bv{\tilde W}^T\bv{\tilde K}).\\ 
\end{align*}
Compare this to the value obtained from $\bs{\Phi}^T\bv{K}^{-1/2}\bv{\tilde W}$:
\begin{align}
\label{extra_error_bound}
&\left[\tr(\bv{K}) - \tr(\bv{\tilde W}\bv{\tilde W}^T\bv{K}\bv{\tilde W}\bv{\tilde W}^T)\right] - \left[\tr(\bv{K}) - \tr(\bv{\tilde W}\bv{\tilde W}^T\bv{\tilde K}\bv{\tilde W}\bv{\tilde W}^T)\right] \nonumber \\
&= \tr\left(\bv{\tilde W}\bv{\tilde W}^T(\bv{K}-\bv{\tilde K})\right) 
= \tr\left(\bv{\tilde W}^T(\bv{K}-\bv{\tilde K})\bv{\tilde W}\right) 
= \sum_{i=1}^k \bv{\tilde w}_i^T(\bv{K}-\bv{\tilde K})\bv{\tilde w}_i
\leq k \frac{\epsilon}{k}\sum_{i=k+1}^n \sigma_i(\bv{K}).
\end{align}
The last step follows from Theorem \ref{additiveErrorThm} which guarantees that $(\bv{K}-\bv{\tilde K})\preceq \epsilon \lambda \bv{I}$. Recall that we set $\lambda = \frac{\epsilon}{k}\sum_{i=k+1}^n \sigma_i(\bv{K})$ and each column $\bv{\tilde w}_i$ of $\bv{\tilde W}$ has unit norm.

We conclude that the cost obtained by $\bv{P}_\bv{S}\bs{\Phi}^T\bv{\tilde K}^{-1/2}\bv{\tilde W}$ is bounded by:
\begin{align*}
\tr(\bv{K} - \bs{\Phi}\bv{P}_\bv{S}\bs{\Phi}^T\bv{\tilde K}^{-1/2}\bv{\tilde W}\bv{\tilde W}^T\bv{\tilde K}^{-1/2}\bs{\Phi}\bv{P}_\bv{S}\bs{\Phi}^T) &\leq \tr(\bv{K}) - \tr(\bv{\tilde W}\bv{\tilde W}^T\bv{K}\bv{\tilde W}\bv{\tilde W}^T) + \epsilon\sum_{i=k+1}^n \sigma_i(\bv{K}) \\
&\leq (1+2\epsilon)\sum_{i=k+1}^n \sigma_i(\bv{K}).
\end{align*}
This gives the result. Notice that $\bv{P}_\bv{S}\bs{\Phi}^T\bv{\tilde K}^{-1/2}\bv{\tilde W} = \bv{\Phi}^T\bv{S} (\bv{S}^T\bv{K}^T \bv{S})^+ \bv{S}^T\bv{\Phi}\bs{\Phi}^T\bv{\tilde K}^{-1/2}\bv{\tilde W}$ so, if we set:
\begin{align*}
\bv{X} = (\bv{S}^T\bv{K}^T \bv{S})^+ \bv{S}^T\bv{\tilde K}^{1/2}\bv{\tilde W},
\end{align*}
our solution can be represented as $\bv{Z} = \bv{\Phi}^T\bv{S}\bv{X}$ as desired.
\end{proof}

\subsection{Kernel canonical correlation analysis}
We briefly discuss a final application to canonical correlation analysis (CCA) that follows from applying our spectral  approximation guarantee of Theorem \ref{additiveErrorThm} to recent work in \cite{wang2016column}.

Consider $n$ pairs of input points $(\bv{x}_1,\bv{y}_1),...,(\bv{x}_n, \bv{y}_n) \in (\mathcal{X},\mathcal{Y})$ along with two positive semidefinite kernels, $K_x: \mathcal{X} \times \mathcal{X} \rightarrow \mathbb{R}$ and $K_y: \mathcal{Y} \times \mathcal{Y} \rightarrow \mathbb{R}$. Let $\mathcal{F}_x$ and $\mathcal{F}_y$ and $\phi_x: \mathcal{X} \rightarrow \mathcal{F}_x$ and $\phi_y: \mathcal{Y} \rightarrow \mathcal{F}_y$ be the Hilbert spaces and feature maps associated with these kernels. Let $\bv{\Phi}_x$ and $\bv{\Phi}_y$ denote the kernelized $\mathcal{X}$ and $\mathcal{Y}$ inputs respectively and $\bv{K}_x$ and $\bv{K}_y$ denote the associated kernel matrices. 

We consider standard regularized kernel CCA, following the presentation in \cite{wang2016column}. The goal is to compute coefficient vectors $\bs{\alpha}^x$ and $\bs{\alpha}^y$ such that $\bv{f}_x^* = \sum_{i=1}^n \bs{\alpha}^x_i \phi_x(\bv{x}_i)$ and $\bv{f}_y^* = \sum_{i=1}^n \bs{\alpha}^y_i \phi_y(\bv{y}_i)$ satisfy:
\begin{align*}
(\bv{f}_x^*,\bv{f}_y^*) &= \argmax_{\bv{f}_x \in \mathcal{F}_x, \bv{f}_y \in \mathcal{F}_y} \bv{f}_x^T\bv{ \Phi}_x^T \bv{ \Phi}_y \bv{f}_y^*\\
&\text{subject to}\\
\bv{f}_x^T\bv{ \Phi}^T_x \bv{ \Phi}_x \bv{f}_x &+ \lambda_x \norm{\bv{f}_x}_{\mathcal{F}_x}^2 = 1\\
\bv{f}_y^T\bv{ \Phi}_y^T \bv{ \Phi}_y \bv{f}_y &+ \lambda_y \norm{\bv{f}_y}_{\mathcal{F}_y}^2 = 1
\end{align*}
In \cite{wang2016column}, the kernelized points are centered to their means. For simplicity we ignore centering, but note that \cite{wang2016column} shows how bounds for the uncentered problem carry over to the centered one.

It can be shown that $\bs{\alpha}^x = (\bv{ K}_x + \lambda_x \bv{I})^{-1} \bs{\beta}^x$ and $\bs{\alpha}^y = (\bv {K}_y + \lambda_y \bv{I})^{-1} \bs{\beta}^y$ where $\bs{\beta}^x$ and $\bs{\beta}^y$ are the top left and right singular vectors respectively of 
\begin{align*}
\bv{T} = (\bv{ K}_x + \lambda_x \bv{I})^{-1} \bv{ K}_x \bv{ K}_y (\bv{ K}_y + \lambda_y \bv{I})^{-1}.
\end{align*}
The optimum value of the above program will be equal to $\sigma_1(\bv{T})$.

\cite{wang2016column} shows that if $\bv{\tilde K}_x$ and $\bv{\tilde K}_y$ satisfy:
\begin{align*}
\bv{\tilde K}_x \preceq \bv{K}_x \preceq \bv{\tilde K}_x + \epsilon\lambda_x\bv{I} \\
\bv{\tilde K}_y \preceq \bv{K}_y \preceq \bv{\tilde K}_y + \epsilon\lambda_x\bv{I}
\end{align*}
then if $\bs{\tilde \alpha}^x$ and $\bs{\tilde \alpha}^y$ are computed using these approximations, the achieved objective function value will be within  $\epsilon$ of optimal (see their Lemma 1 and Theorem 1). So we have:


\begin{theorem}[Kernel CCA Approximation Bound]
\label{thm:app_cca}
Suppose $\bv{\tilde K}_x$ and $\bv{\tilde K}_y$ are computed by RLS-Nystr\"{o}m with approximation parameters $\epsilon\lambda_x$ and $\epsilon\lambda_y$ and failure probability $\delta \in (0,1/8)$. If we solve for $ \bs{\tilde \alpha}^x$ and $\bs{\tilde \alpha}^y$, the approximate canonical correlation will be within an additive $\epsilon$ of the true canonical correlation $\sigma_1(\bv{T})$.

By Theorem \ref{thm:main_algo_theorem}, Algorithm \ref{halvingFixed} can compute $\bv{\tilde K}_x$ and $\bv{\tilde K}_y$ with $O(ns_x + ns_y)$ kernel evaluations and $O(ns_x^2 + ns_y^2)$ computation time, with $s_x = O\left(\frac{d_{\text{eff}}^{\lambda_x}}{\epsilon}\log\frac{d_{\text{eff}}^{\lambda_x}}{\delta\epsilon}\right)$ and $s_y = O\left(\frac{d_{\text{eff}}^{\lambda_y}}{\epsilon}\log\frac{d_{\text{eff}}^{\lambda_y}}{\delta\epsilon}\right)$.
\end{theorem} 

\section{Additional proofs}\label{sec:additional}
\subsection{Effective dimension bound}
\label{additional_proofs}
\begin{lemma}\label{decreasingScore2} For any $\bv{W} \in \mathbb{R}^{n \times p}$ with $\bv{WW}^T \preceq \bv{I}$, 
\begin{align*}
\sum_{i=1}^n l_i^\lambda (\bv{W}^T\bv{K}\bv{W}) \le \sum_{i=1}^n l_i^\lambda (\bv{K}),
\end{align*}
\vspace{-.5em}
or equivalently, by Fact \ref{sum_equal_deff},
\begin{align*}
d_\text{eff}^\lambda(\bv{W}^T\bv{K}\bv{W}) \leq d_\text{eff}^\lambda(\bv{K}).
\end{align*}

\end{lemma}
\begin{proof}
By Definition \ref{leverageDef}, $l_i^\lambda = \left(\bv{K}(\bv{K} + \lambda \bv{I})^{-1}\right)_{i,i}$ so 
\begin{align*}
\sum_{i=1}^n l_i^\lambda (\bv{K}) = \tr\left(\bv{K}(\bv{K} + \lambda \bv{I})^{-1}\right) = \sum_{i=1}^n \frac{\sigma_i(\bv{K})}{\sigma_i(\bv{K}) + \lambda}.
\end{align*}
Take any matrix $\bv{B} \in \R^{n\times n}$ such that $\bv{B}\bv{B}^T = \bv{K}$. Note that for any matrix $\bv{Y}$, $\sigma_i(\bv{Y}\bv{Y}^T) = \sigma_i(\bv{Y}^T\bv{Y})$ for any non-zero singular values. Accordingly,
\begin{align*}
\sigma_i(\bv{W}^T \bv{K} \bv{W}) = \sigma_i(\bv{W}^T \bv{B}\bv{B}^T  \bv{W}) = \sigma_i(\bv{B}^T  \bv{W}\bv{W}^T \bv{B}) \leq \sigma_i(\bv{B}^T\bv{B}) =\sigma_i(\bv{B}\bv{B}^T) = \sigma_i(\bv{K})  
\end{align*}
The $\leq$ step follows from $\bv{WW}^T \preceq \bv{I}$ so $\bv{B}^T  \bv{W}\bv{W}^T \bv{B} \preceq \bv{B}^T \bv{B}$. We thus have:
\begin{align*}
\sum_{i=1}^n l_i^\lambda (\bv{W}^T\bv{K}\bv{W}) = \sum_{i=1}^p \frac{\sigma_i(\bv{W}^T\bv{K}\bv{W})}{\sigma_i(\bv{W}^T\bv{K}\bv{W}) + \lambda} \le \sum_{i=1}^n \frac{\sigma_i(\bv{K})}{\sigma_i(\bv{K}) + \lambda} = \sum_{i=1}^n l_i^\lambda (\bv{K}),
\end{align*}
giving the lemma.
\end{proof}

\subsection{Proof of Theorem \ref{thm:main_algo_k}: fixed sample size guarantees}
We now prove Theorem \ref{thm:main_algo_k}, which gives the approximation and runtime guarantees for our fixed sample size algorithm, Algorithm \ref{halvingK}. The theorem follows from the recursive invariant:
\begin{theorem}\label{thm:halvingK} With probability $1-3\delta$, Algorithm \ref{halvingK} performs $O(n s)$ kernel evaluations, runs in $O(n s^2 )$ time, and for any integer $k$ with $s \ge c k \log(2k/\delta)$ returns $\bv{S}$ satisfying, for any $\bv{B}$ with $\bv{B}\bv{B}^T = \bv{K}$:
\begin{align}\label{fixedKBound}
\frac{1}{2}(\bv{B}^T\bv{B} + \lambda \bv{I}) \preceq (\bv{B}^T\bv{S}\bv{S}^T \bv{B} + \lambda \bv{I}) \preceq \frac{3}{2} (\bv{B}^T\bv{B} + \lambda \bv{I})
\end{align}
for $\lambda = \frac{1}{k}\sum_{i=k+1}^n \sigma_i(\bv{K})$.
\end{theorem}
\begin{proof}
Assume by induction that after forming $\bv{\bar S}$ via uniformly sampling, the recursive call to Algorithm \ref{halvingK} returns $\bv{\tilde S}$ such that $\bv{\hat S} = \bv{\bar S} \cdot \bv{\tilde S}$ satisfies:
\begin{align}\label{primeprimeBound}
\frac{1}{2}(\bv{B}^T\bv{\bar S}\bv{\bar S}^T\bv{B} + \lambda' \bv{I}) \preceq (\bv{B}^T\bv{\hat S}\bv{\hat S}^T \bv{B} + \lambda' \bv{I}) \preceq \frac{3}{2} (\bv{B}^T\bv{\bar S}\bv{\bar S}^T\bv{B} + \lambda' \bv{I}).
\end{align}
where $\lambda' =\frac{1}{k} \sum_{i=k+1}^n \sigma_i(\bv{\bar S}^T \bv{K} \bv{\bar S}).$ This implies that $\tilde \lambda = \frac{1}{k}\sum_{i=k+1}^n \sigma_i(\bv{\hat S}^T\bv{K}\bv{\hat S})$ satisfies:
\begin{align*}
\frac{1}{2k} \left ( \sum_{i=k+1}^n \sigma_i(\bv{\bar S}^T\bv{K}\bv{\bar S})+ k \lambda' \right) &\le \tilde \lambda \le \frac{3}{2k} \left ( \sum_{i=k+1}^n \sigma_i(\bv{\bar S}^T\bv{K}\bv{\bar S})+ k \lambda' \right)\\
 \lambda' &\le \tilde \lambda \le 3 \lambda'.
\end{align*}
Combining with \eqref{primeprimeBound} we have:
\begin{align*}
\frac{1}{2}(\bv{B}^T\bv{\bar S}\bv{\bar S}^T\bv{B} + \lambda' \bv{I}) \preceq (\bv{B}^T\bv{\hat S}\bv{\hat S}^T \bv{B} +  \tilde \lambda \bv{I}) \preceq \frac{9}{2}(\bv{B}^T\bv{\bar S}\bv{\bar S}^T\bv{B} + \lambda' \bv{I}).
\end{align*}

So, for all $i$, $\tilde l_i^{\lambda}$ (which is computed using $(\bv{B}^T\bv{\hat S}\bv{\hat S}^T \bv{B} + \tilde \lambda \bv{I})$ and oversampling factor $5$ in Step \ref{step5} of Algorithm \ref{halvingK}) is at least as large as the approximate leverage score computed using $\bv{\bar S}$ instead of $\bv{\hat S}$. If we sample by these scores, by Lemma \ref{uniformSampling} and Lemma \ref{bernstein} we will have with probability  $1-\delta$:
\begin{align*}
\frac{1}{2}(\bv{B}^T\bv{B} + \lambda' \bv{I}) \preceq (\bv{B}^T\bv{S}\bv{S}^T \bv{B} + \lambda' \bv{I}) \preceq \frac{3}{2} (\bv{B}^T\bv{B} + \lambda' \bv{I})
\end{align*}
which implies \eqref{fixedKBound} since $\lambda' \le \lambda$ since $\norm{\bv{\bar S}}_2 \le 1$ so $\sigma_i(\bv{\bar S}^T \bv{K} \bv{\bar S}) \le \sigma_i(\bv{K})$ for all $i$.

It just remains to show that we do not sample too many points. This can be shown using a similar reweighting argument to that used in the fixed $\lambda$ case in Lemma \ref{uniformSampling}. Full details appear in Lemma 13 of \cite{cohen2015ridge}. When forming the reweighting matrix $\bv{W}$, decreasing $\bv{W}_{i,i}$ will decrease $\sum_{i=k+1}^n \sigma_i(\bv{W}\bv{K}\bv{W})$ and hence will decrease $\lambda$. However, it is not hard to show that the $i^\text{th}$ ridge leverage score will still decrease. So we can find $\bv{W}$ giving a uniform ridge leverage score upper bound of $\alpha$. Let $\lambda ' = \sum_{i=k+1}^n \sigma_i(\bv{W}\bv{K}\bv{W})$. 

Using the same argument as Lemma \ref{uniformSampling}, we can bound the sum of estimated sampling probabilities by $64\log(\sum l_i^{\lambda'}(\bv{W}\bv{K}\bv{W}) / \delta) \cdot \sum l_i^{\lambda'}(\bv{W}\bv{K}\bv{W}) \le s/5$ by Fact \ref{ridgeScore2kBound} if we set $c$ large enough. The runtime and failure probability analysis is identical to that of Algorithm \ref{halvingFixed} (Theorem \ref{thm:halvingFixed}) -- the only extra step is computing $\tilde \lambda$ which can be done in $O(s^3)$ time via an SVD of $\bv{\hat S}^T \bv{K} \bv{\hat S}$.
\end{proof}

\begin{proof}[Proof of Theorem \ref{thm:main_algo_k}] The theorem follows immediately since Theorem \ref{thm:halvingK} guarantees that in the final level of recussion $\bv{K}$ is sampled by overestimates of its $\lambda$-ridge leverage scores.
The runtime bound follows from Theorem \ref{thm:halvingK} and the fact that it is possible to compute $\bv{KS}$ using $O(n  s)$ kernel evaluations and $(\bv{S}^T\bv{K}\bv{S})^+$ using $O(ns^2 + s^3) = O(n s^2)$ additional time.
\end{proof}

\end{document}